\RequirePackage{booktabs}
\documentclass[sn-mathphys,Numbered,dvipsnames]{sn-jnl}
\usepackage{etoolbox}

\makeatletter
\patchcmd{\@maketitle}{\removelastskip\vskip24pt}{\removelastskip\vskip5pt}{}{\errmessage{Patch 1 failed}}

\patchcmd{\@maketitle}{\removelastskip\vskip24pt}{\removelastskip\vskip2pt}{}{\errmessage{Patch 2 failed}}
\makeatother

\usepackage{multirow}%
\usepackage{amsmath,amssymb, amsfonts}%
\usepackage{booktabs}
\usepackage[title]{appendix}%
\usepackage{xcolor}%
\usepackage{colortbl}
\usepackage{color}
\usepackage{textcomp}%
\usepackage{manyfoot}%
\usepackage{algorithm}%
\usepackage{algorithmicx}%
\usepackage{algpseudocode}%
\usepackage{listings}%

\usepackage{caption}
\usepackage{mathtools} 

\usepackage{pifont}
\newcommand{\cmark}{\ding{51}}%
\newcommand{\xmark}{\ding{55}}%
\usepackage{dsfont}
\usepackage{amsthm}
\usepackage{bm}
\usepackage{eurosym}
\usepackage{mathabx}
\usepackage{aligned-overset}
\newtheorem{lemma}{Lemma}
\newtheorem{assumption}{Assumption}
\newtheorem{corollary}{Corollary}
\usepackage{mathabx}
\usepackage{subfig}

\usepackage{enumerate}
\usepackage{comment}

\usepackage{parskip}

\usepackage{graphicx}

\usepackage{xcolor}

\usepackage{color}
\usepackage{tikz}
\usepackage[mathscr]{eucal}

\DeclareMathSymbol{\shortminus}{\mathbin}{AMSa}{"39}

\newcommand{\argmin}{\operatorname{arg\,min}}

\newcommand{\norm}[1]{\left\lVert#1\right\rVert}

\newcommand{\explainup}[2]{\overset{\mathclap{\underset{\downarrow}{#2}}}{#1}}

\newcommand{\bbeta}{\bm{\beta}} 
\newcommand{\hbbeta}{\hat{\bm{\beta}}} 
\newcommand{\tbbeta}{\tilde{\bm{\beta}}} 
\newcommand{\bxi}{\bm{\xi}} 
\newcommand{\hbxi}{\hat{\bm{\xi}}} 
\newcommand{\tbxi}{\tilde{\bm{\xi}}} 
\newcommand{\bpsi}{\bm{\psi}} 
\newcommand{\hbpsi}{\hat{\bm{\psi}}} 
\newcommand{\tbpsi}{\tilde{\bm{\psi}}} 
\newcommand{\bu}{\bm{u}} 
\newcommand{\hbu}{\hat{\bm{u}}} 
\newcommand{\bv}{\bm{v}} 
\newcommand{\hbv}{\hat{\bm{v}}} 
\newcommand{\bnu}{\bm{\nu}} 
\newcommand{\hbnu}{\hat{\bm{\nu}}} 
\newcommand{\Rbeta}{\mathcal{R}_{\bm{\beta}}} 
\newcommand{\Rd}{\mathbb{R}^{d}} 
\newcommand{\Rdxi}{\mathbb{R}^{d_{\xi}}} 
\newcommand{\dxi}{{d_{\xi}}} 
\newcommand{\Rxi}{\mathcal{R}_{\bm{\xi}}} 
\ifdefined\P      
\renewcommand{\P}{\mathcal{P}} 
\else \newcommand{\P}{\mathcal{P}} \fi 
\ifdefined\L
\renewcommand{\L}{\mathcal{L}} 
\else \newcommand{\L}{\mathcal{L}} \fi 
\ifdefined\B     
\renewcommand{\B}{\mathcal{B}} 
\else \newcommand{\B}{\mathcal{B}} \fi 
\newcommand{\Q}{\mathcal{Q}} 
\newcommand{\K}{\mathcal{K}} 
\newcommand{\G}{\mathcal{G}} 
\newcommand{\J}{\mathcal{J}} 
\newcommand{\Gj}{{\mathcal{G}_j}} 

\allowdisplaybreaks

\newsavebox\dotbox
\sbox{\dotbox}{\(\displaystyle\bigodot\)}
\newlength{\dotheight}
\DeclareMathOperator{\odotg}{\odot_{\mathcal{G}}}

\newtheorem{theorem}{Theorem}
\newtheorem{proposition}[theorem]{Proposition}%
\newtheorem{example}{Example}%
\newtheorem{remark}{Remark}%

\newtheorem{definition}{Definition}%

\makeatletter

\def\ps@titlepage{%
  \def\@oddhead{%
    \vbox to 0pt{\vspace*{-38pt}%
      \hbox to \hsize{\hfill \hfill}}%
  }%
  \let\@evenhead\@oddhead%
  \def\@oddfoot{%
    \vbox to 24pt{%
      \vfill
      \vskip8pt
      \vbox{%
        \footnotesize
        \hbox to \textwidth{%
          This article has been accepted for publication at %
          \textit{Machine Learning}, after peer review, but is not the Version of Record.\hfill
        }%
        \hbox to \textwidth{%
          The Version of Record is available online at: %
          \url{https://doi.org/10.1007/s10994-026-06997-0}.\hfill
        }%
      }%
      \vspace{8pt} 
      \hbox to \textwidth{\hfil\thepage\hfil}
    }%
  }%
  \def\@evenfoot{}%
}

\g@addto@macro\@maketitle{\thispagestyle{titlepage}}

\makeatother

\begin{document}

\title[Smooth Optimization for Sparse Regularization using Hadamard Overparametrization]{\vspace{-4cm}\Large{\center{Smoothing the Edges:}\\ Smooth Optimization for Sparse Regularization using Hadamard Overparametrization}}


\author*[1,2]{\fnm{Chris} \sur{Kolb}}\email{chris.kolb@stat.uni-muenchen.de}

\author[1,2,3,4]{\fnm{Christian L.} \sur{M\"uller}}\email{christian.mueller@helmholtz-munich.de}

\author[1,2]{\fnm{Bernd} \sur{Bischl}}\email{bernd.bischl@stat.uni-muenchen.de}
\author[1,2]{\fnm{David} \sur{R\"ugamer}}\email{david.ruegamer@stat.uni-muenchen.de}

\affil*[1]{\orgdiv{\large{Department of Statistics}}, \orgname{\large{LMU Munich}}, \orgaddress{\street{\large{Ludwigstr. 33}}, \city{\large{Munich}}, \postcode{\large{80539}}, 
\country{\large{Germany}}}}

\affil[2]{\orgdiv{Munich Center for Machine Learning (MCML)}, 
 \orgaddress{\street{}, \city{Munich}, \postcode{80539}, 
 \country{Germany}}}

\affil[3]{\orgdiv{Institute of Computational Biology}, \orgname{Helmholtz Munich},\\ \orgaddress{\street{Ingolst\"adter Landstrasse 1}, \city{Neuherberg}, \postcode{85764}, 
\country{Germany}}}

\affil[4]{\orgdiv{Center for Computational Mathematics}, \orgname{Flatiron Institute},\\ \orgaddress{\street{162 5th Ave}, \city{New York}, \postcode{NY 10010}, 
\country{USA}}}


\abstract{In recent years, overparametrization has received considerable attention in various fields, shown to accelerate training and promote simplicity. However, few works study the induced sparse regularization of the original parameters that is caused by combining overparametrization with explicit smooth regularization. Here, we present a unifying framework for smooth optimization of explicitly regularized objectives for (structured) sparsity. These non-smooth and possibly non-convex problems typically rely on solvers tailored to specific models or regularizers and have not been widely adopted in deep learning.
In contrast, our method promises fully differentiable and approximation-free optimization for sparse regularizers and is thus compatible with the ubiquitous gradient descent paradigm. The proposed optimization transfer comprises overparameterization of selected parameters and a change of penalties. We prove that the surrogate objective is equivalent in the sense of identical global and local minima, thereby avoiding the introduction of spurious solutions.
We comprehensively review sparsity-inducing parametrizations across different
fields and combine them in our explicit surrogate regularization framework. We further extend their scope, point out improvements and present novel parametrizations. Numerical experiments further demonstrate the
correctness and effectiveness of our approach on several sparse learning problems from high-dimensional regression to sparse neural network training.
}




\maketitle

\section{Introduction and Background}
%
As a result of the recent proliferation of high-dimensional and unstructured data, methods for sparse or low-rank representations have become increasingly important in fields such as machine learning, statistics, and signal processing. Parsimonious models are commonly used to incorporate prior knowledge about the complexity of the underlying phenomenon, to obtain interpretable sparse approximations of non-sparse ground truths \citep{bach2012optimization}, or to regularize otherwise intractable inverse problems \citep{benning2018modern}. In deep learning (DL), the reduction in computational burden for large-scale optimization or inference is an important motivation for model sparsification, from both the perspective of efficiency and sustainability \citep{blalock2020state,hoefler2021sparsity}. 
Structured or group sparsity naturally generalizes the notion of unstructured sparsity to enable \textit{structural} prior information about parameter group complexity into the optimization problem \citep{huang2009learning, jenatton2011structured}.
\subsection{Convex and Non-Convex Sparse Regularization}
\begin{figure}[t!]
\centering
\vspace{-0.1cm}
\includegraphics[width=0.85\textwidth]{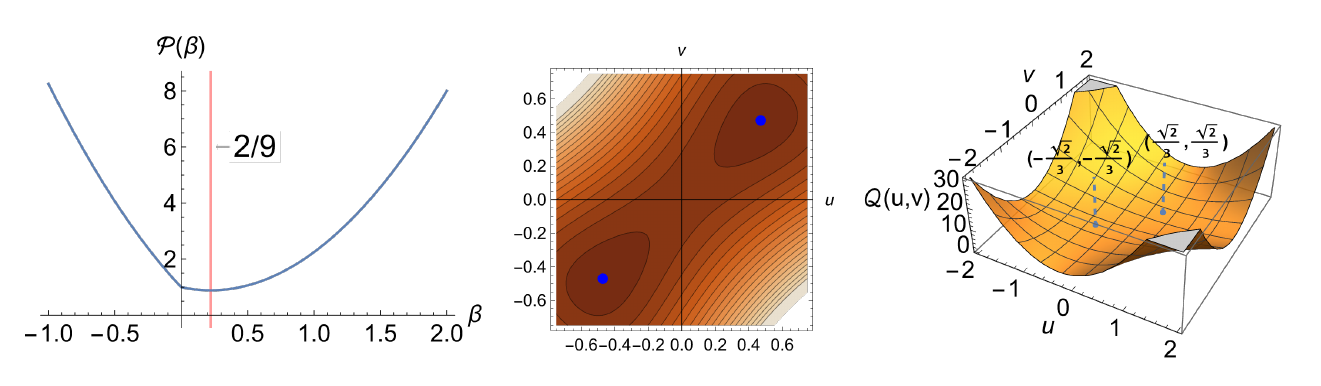}
\caption[Toy Loss Landscape for Optimization Transfer]{\small Illustration of smooth optimization transfer. \textbf{Left}: univariate lasso problem $\P(\beta)=(1-\frac{3}{2} \beta)^2+2|\beta|$ (red line indicates the global minimizer $\hat{\beta}$). \textbf{Middle}: contours of the equivalent smooth surrogate $\Q(u,v)=(1-\frac{3}{2} uv)^2+u^2+v^2$ using a Hadamard product parametrization (\ref{eq:hpp-def}) with $\K(u,v)=uv=\beta$. Both global minimizers (dots) map to $\K(\hat{u},\hat{v})=\hat{\beta}$. \textbf{Right}: non-convex surface of higher-dimensional $\Q(u,v)$.}
\label{fig:hpp-tri-1}
\vspace{-0.3cm}
\end{figure}
%
%
\textbf{$\bm{\ell_0}$ and $\bm{\ell_1}$ regularization} \, In sparse estimation problems of a parameter vector $\bm{\beta}\in\mathbb{R}^{d}$ given any objective function $\mathcal{L}:\mathbb{R}^{d} \to \mathbb{R}^{+}_{0}$, the classical optimization problem using explicit regularization is 
\begin{equation} \label{eq:standard-formulation}
    \min_{\bm{\beta} \in \mathbb{R}^{d}} \mathcal{L}(\bm{\beta}) + \lambda \mathcal{R}(\bm{\beta}) \,,
\end{equation}
with regularization or penalty function $\mathcal{R}:\mathbb{R}^{d} \to \mathbb{R}^{+}_{0}$ to $\L(\bbeta)$, whose strength is controlled by $\lambda \geq 0$. 
A natural choice for the regularizer is $\mathcal{R}(\bm{\beta}) =\Vert \bm{\beta} \Vert_0$, i.e., the cardinality of the support of $\bm{\beta}$ counting its non-zero entries. However, this best-subset approach is infeasible due to its non-convex, and non-continuous NP-hard nature \citep{natarajan1995sparse, chen2017strong}. To overcome these difficulties, convex relaxations of $\ell_0$ regularization have been proposed that enable optimization via, e.g., coordinate descent or projected gradient methods \citep{tropp2006just, schmidt2007fast}. The tightest convex relaxation of $\Vert \bm{\beta} \Vert_0$ is given by its convex envelope $\Vert \bm{\beta} \Vert_1$, resulting in
\begin{equation} \label{eq:l1-formulation}
    \min_{\bm{\beta} \in \mathbb{R}^{d}} \mathcal{L}(\bm{\beta}) + \lambda \Vert \bm{\beta} \Vert_{1} \,.
\end{equation}
This formulation is known as $\ell_{1}$ regularization today. In the context of linear models, it has been introduced as the lasso to the statistics community \citep{tibshirani1996regression} and as Basis Pursuit Denoising in signal processing \citep{chen1994basis,chen2001atomic}.
For convex $\L$, such as in linear regression, the well-developed machinery of convex optimization can be utilized to solve (\ref{eq:l1-formulation}). $\ell_1$ regularization has also been shown to have some favorable theoretical properties, such as consistent recovery of the true support of $\bm{\beta}$ under restricted conditions \citep{donoho2003optimally, zhao2006model, meinshausen2006high}. However, using $\ell_1$ regularization to achieve sparsity also comes with a disadvantage: whereas $\norm{\bbeta}_0$ is constant on the support of $\bm{\beta}$, the $\ell_1$ penalty increases linearly in the magnitude of its components. This leads to estimation bias for large parameters \citep{zhang2008sparsity} and inconsistent support recovery \citep{chartrand2007exact,xu2012l_}.
To mitigate the challenges posed by $\ell_1$ and $\ell_0$ regularization, the seminal work of \citet{fan2001variable} proposed smoothly clipped absolute deviations (SCAD), one of the earliest examples of non-convex regularizers. Another popular non-convex penalty that enables feature selection and nearly unbiased estimation is the minimax concave penalty (MCP) introduced by \citet{zhang2010nearly}.
\vspace{0.1cm}

\noindent \textbf{$\bm{\ell_{p,q}}$ regularization} \, In this work, however, we focus on a generalization of the $\ell_1$ penalty based on the $\ell_q$ quasi-norm, $\Vert \cdot \Vert_{q}$, for $0<q\leq1$. This approach was initially described by \citet{frank1993statistical} and subsequently popularized as the bridge penalty by \citet{fu1998penalized}. Non-convex bridge regularization is defined by a regularization term of the form $\mathcal{R}(\bbeta) = \Vert \bbeta \Vert_{q}^{q}$ for $0 < q < 1$. For the case of structured sparsity, $\ell_q$ regularization can be straightforwardly extended to mixed-norm $\ell_{p,q}$ regularization for $0<q<p\leq2$, studied, 
e.g., in \citet{hu2017group}. A number of important desirable theoretical results have been established for non-convex $\ell_q$ and $\ell_{p,q}$ regularization, such as requiring fewer linear measurements for support recovery and permitting sparser solutions compared to convex $\ell_1$ and $\ell_{2,1}$ (group-wise) regularization \citep{fu1998penalized, fu2000asymptotics,chartrand2007exact,  xu2012l_}. Moreover, the regularity conditions required for consistent recovery are weaker than typically required for $\ell_1$ \citep{chartrand2008restricted, loh2017support} or $\ell_{2,1}$ penalties \citep{hu2017group}. \\
\indent Optimization using $\ell_q$ regularization, however, poses a non-smooth and non-convex 
problem for $0<q<1$ and is thus difficult to solve efficiently. \citet{ge2011note} show that identification of the global minimum is strongly NP-hard. %
Still, computing local minima of the non-convex regularization problem usually performs better compared to convex regularization approaches \citep{xu2010l1, xu2012l_, lyu2013comparison, wen2018survey}. A variety of optimization techniques such as the local quadratic approximation or majorization-minimization algorithms \citep{lange2000optimization, hunter2005variable}, and  various flavors of coordinate or subgradient descent methods have been discussed in the literature %
\citep[see][for a survey of optimization with non-convex regularization]{wen2018survey}.

\vspace{0.08cm}
\noindent The need for specialized optimization routines for non-smooth and non-convex regularized optimization problems 
has arguably hindered the widespread use of $\ell_q$ and $\ell_{p,q}$ regularization, despite their favorable theoretical properties and the limitations of convex regularizers \citep[see, e.g.,][]{freijeiro2022critical}. 
In contrast, smooth first-order methods have become the go-to optimization tool for many researchers and practitioners, not limited to the field of DL anymore. This can be attributed to their applicability to a vast class of problems using automatic differentiation, their scalability to large data sets, and their surprising effectiveness despite using only cheaply computed gradient information.
While in practice, popular DL platforms offer implementations of $\ell_1$ regularization, this 
essentially reduces to applying stochastic gradient descent (SGD) to a non-differentiable problem. 
Unsurprisingly, this mismatch typically results in oscillating parameter updates, slow convergence, and a failure of parameter iterates to approach zero values (see Figure~\ref{fig:comparison-direct-gd-sparsity}).%

\subsection{Our Contributions}

To overcome the obstacles and complexities of using optimization routines tailored for specific non-smooth and potentially non-convex regularized problems, we apply a smooth variational form (SVF) that allows expressing the non-smooth regularizer as the constrained minimum of a smooth surrogate regularizer, where the constraint involves an overparametrization of model parameters. In our framework, we 
construct a general template for exact smooth surrogate optimization of non-smooth and potentially non-convex 
regularized problems. This optimization transfer is based on finding SVFs of the respective 
regularizers, which entail a smooth parametrization map together with a smooth surrogate regularizer. Combined, an equivalent smooth 
surrogate objective can be constructed. %
Specifically, we
\begin{itemize} \setlength\itemsep{0.5em}
\item provide a comprehensive 
review of the loosely connected works on Hadamard para\-metri\-zations, 
relating literature across DL, statistics, and optimization.

\item introduce a smooth surrogate optimization framework for non-smooth and non-convex regularization of arbitrary parameters, including a matching local minima property. 
While previous works often exploit properties particular to their setting, our main results (Thm.~\ref{theorem-general} and Lemma~\ref{lemma:p_to_pk}-\ref{lemma:q_to_p}) are stated broadly and hold for arbitrary losses, learning models, and regularizers given our assumptions.

\item apply our template method to a wide array of (group-)sparse $\ell_q$ and $\ell_{p,q}$ regularized problems, expanding the collection of sparsity-inducing parametrizations to Hadamard powers and shared parameters.

\item present different SVFs with variable amounts of overparametrization for the same induced regularizer, highlighting that it is not overparametrization \textit{per se} inducing sparsity, but rather its effect on the curvature of the loss landscape. 
%

\item identify parametrizations with specific neural network structures, generalizing previous findings on 
linear models 
to modular components within arbitrary networks. This enables the integration of sparse regularization into the prevalent SGD-based optimization paradigm in DL using sparse ``drop-in'' replacements. 

\item evaluate our smooth optimization transfer approach on various sparse learning applications and demonstrate its correctness and practical feasibility.
\end{itemize}

\vspace{0.18cm}
\noindent \textbf{Outline}\,
Section~\ref{sec:theory} 
establishes a set of theoretical results that prove the validity of our general framework and provide a construction template for various regularizers.
Sections~\ref{sec:l1-hpp-hdp} and~\ref{sec:had-group-lasso} apply our optimization transfer to construct equivalent smooth surrogates for convex $\ell_1$ and structured $\ell_{2,1}$ sparse regularization. 
Section~\ref{sec:hppk} discusses deeper factorizations 
involving more than two Hadamard factors, enabling smooth optimization of a restricted class of non-convex $\ell_q$ 
and $\ell_{p,q}$ 
regularized problems. Additionally, mitigation strategies to reduce the computational complexity of overparametrization, such as parameter sharing, 
are discussed.
Section~\ref{sec:hpowp} leverages the concept of Hadamard powers to broaden the expressivity of previous parametrizations, thereby lifting the aforementioned restrictions on the class of induced regularizers. 
Section~\ref{sec:optim} discusses specifics regarding the practical optimization of the constructed smooth surrogates.
Related work is discussed and compared with our approach in Section~\ref{sec:related-work}.
In Section~\ref{sec:experiments}, we showcase numerical experiments demonstrating the practical feasibility and competitiveness of our approach on a variety of model classes, ranging from sparse linear regression to (convolutional) neural network architectures. 
Section~\ref{sec:discussion} concludes by assessing the merits and limitations of our framework and identifying promising directions for future research.



\section{Set-Up for Transfer and Theoretical Results} \label{sec:theory}

\textbf{Notation}\, We represent vectors using bold lowercase letters and bold capital letters for matrices. We use 
$\bm{\beta} \in \mathbb{R}^{d}$ to denote the parameter vector which is subject to regularization, and $\bm{\psi} \in \mathbb{R}^{d_{\psi}}$ for the remaining parameters, so that all model parameters are collected in $(\bm{\psi},\bm{\beta})$. We make this notational distinction to emphasize that our approach can be applied to arbitrary subsets of parameters of an optimization problem, irrespective of the presence of other parameters or the structure of the main objective $\L$. Thus, sparse regularization using our optimization transfer framework can be applied to, e.g., specific layers of a neural network. 
Further, let $\Vert \bm{\beta} \Vert_{q} \triangleq ( \sum_{j=1}^{d} | \beta_j |^{q} )^{1/q}$
denote the $\ell_q$ norm $\forall \bm{\beta} \in \mathbb{R}^{d}, q\in(0,\infty)$. Note that for $0<q<1$, only a quasi-norm is defined as the subadditivity does not hold. For $q=0$, the $\ell_0$ ``norm'' penalty $\Vert \bm{\beta} \Vert_{0}$ 
counts the number of non-zero elements in $\bm{\beta}$. Given a partition $\mathcal{G}=\{\mathcal{G}_{1},\ldots,\mathcal{G}_{L}\}$ of $[d] \triangleq \{1,\ldots,d\}$, the $\ell_{p,q}$ group (quasi-)norm is defined as $\Vert \bm{\beta} \Vert_{p, q} \triangleq(\sum_{j=1}^{L}(\sum_{i \in \mathcal{G}_{j}} \left|\beta_{i}\right|^{p})^{q / p})^{1 / q} = (\sum_{j=1}^{L} \Vert \bm{\beta}_{j} \Vert_{p}^{q})^{1 / q} \, \forall \, \bm{\beta} \in \mathbb{R}^{d},\, p,q>0$, where $\bbeta_j$ 
contains the components corresponding to $\Gj$. The regularization term for an $\ell_{p,q}$ penalty is given by the $q$-th power of the $\ell_{p,q}$ mixed-norm,  $\norm{\bm{\beta}}_{p,q}^{q}=\sum_{j=1}^{L} \Vert \bm{\beta}_{j} \Vert_{p}^{q}$.
Further, we use various notations to define Hadamard product-like operations, introduced in the following. Let $\odot: \mathbb{R}^{d} \times \mathbb{R}^{d} \to \Rd$ denote the classical Hadamard product, defined as $(\bm{u},\bm{v}) \mapsto (u_1v_1,\ldots,u_dv_d)^{\top}$,
and $\bigodot_{l=1}^{k} \bm{u}_{l}$ the Hadamard product of $k$ vectors, for which we also use the shorthand notation $\bm{u}_{l}^{\odot k}$. For parameter vectors with more than one index, e.g., $\bm{u}_{jl}^{\odot k}$, the Hadamard product is always taken over the second index. The self-Hadamard product $\bm{u} \odot \bm{u}$ is simply written as $\bm{u}^{2}$. A generalization of the self-Hadamard product to non-integer exponents $k>0$, i.e., element-wise raising the entries of $\bm{u}$ to the $k$-th power, is denoted as $\bm{u}^{\circ k}$. Given a partition $\mathcal{G}$ of $[d]$ into $L \leq d$ subsets, we define the group Hadamard product $\odotg$ of two vectors $\bm{u} \in \mathbb{R}^{d}$ and $\bm{\nu} \in \mathbb{R}^{L}$ as $\bm{u} \odotg \bm{\nu} \triangleq (\bm{u}_{j} \nu_{j})_{j \in \mathcal{G}}$, or more explicitly as
{
\begin{equation}\label{eq:ghpp-definition}
\bm{u} \odotg \bm{\nu} \triangleq  \begin{pmatrix}
    \bm{u}_{1} \\ \vdots \\ \bm{u}_{L}
  \end{pmatrix}  \odot \begin{pmatrix} 
    \nu_1 \mathds{1}_{|\mathcal{G}_1|} \\ \vdots \\ \nu_{L} \mathds{1}_{|\mathcal{G}_L|}
  \end{pmatrix} \,, \nonumber
\end{equation}
}

\noindent where $\mathds{1}_{|\mathcal{G}_j|}$ denotes the $1$-vector of size $|\mathcal{G}_j|$. To make the distinction between vectors of size $d$ and $L$ more clear where necessary, we denote vectors in $\mathbb{R}^{d}$ as $\bm{v}$, and alternatively, use $\bm{\nu}$ for vectors in $\mathbb{R}^{L}$. In case $\bv=(\bv_1,\ldots,\bv_L)^{\top}$ is constant within groups $\Gj$, both are related as $\bm{v}_j = \nu_j \mathds{1}_{|\mathcal{G}_{j}|}$ for $j=1,\ldots,L$, and $\bu \odot \bv$ equals the group Hadamard product $\bu \odotg \bnu$. 
Further, $\mathcal{B}(\bbeta,\varepsilon)\subseteq \Rd$ is used to denote an open ball with radius $\varepsilon$ centered at $\bbeta \in \Rd$, for a Euclidean space endowed with the standard topology induced by the Euclidean metric. The non-negative reals are abbreviated as $\mathbb{R}_0^+$. Given a differentiable function $f:\mathbb{R}^{m} \to \mathbb{R},\, \bm{a} \mapsto f(\bm{a})$, the gradient $\nabla_{\bm{a}} f(\bm{a}) \in \mathbb{R}^{m}$  of $f$ at $\bm{a}$ contains partial derivatives $\partial f(\bm{a})/ \partial a_j$ for $j \in [m]$. The Hessian $\mathcal{H}_{f}(\bm{a})$ of $f$ at $\bm{a}$ is the $m \times m$ matrix containing second partial derivatives $(\mathcal{H}_{f}(\bm{a}))_{i j} \triangleq \partial^2 f / \partial a_i \partial a_j$. 
For vector-valued differentiable maps  $\bm{f}:\mathbb{R}^{m} \to \mathbb{R}^{n}, \bm{a} \mapsto \bm{f}(\bm{a})$, the Jacobian $\mathcal{J}_{\bm{f}}(\bm{a})$ of $\bm{f}$ at $\bm{a}$ is the $n \times m$ matrix containing partial derivatives $(\mathcal{J}_{\bm{f}}(\bm{a}))_{i j} \triangleq \partial \bm{f}_{i} / \partial a_j$. If we say a function is \textbf{smooth}, we require it merely to be $\mathcal{C}^r$-smooth, $r  \geq 1$, i.e., at least continuously differentiable. Local solutions to an optimization problem over $(\bpsi,\bbeta)$ are denoted by $(\hat{\bm{\psi}},\hat{\bbeta})$. Finally, the complete proofs are deferred to the appendix.
\vspace{0.15cm}
%

\noindent \textbf{Set-up}\,
Before discussing the applications of our proposed framework to specific sparse regularizers, we first provide a number of general results on the equivalence of (regularized) optimization problems under reparametrization and a change of penalties, which will be applied throughout the paper. Let 
\begin{equation} \label{eq:q-beginning}
\mathcal{P}:\mathbb{R}^{d_{\psi}}\times\mathbb{R}^d \to \mathbb{R}_{0}^{+},\,(\bm{\psi},\bm{\beta}) \mapsto \mathcal{L}(\bm{\psi}, \bm{\beta}) +  \lambda \cdot \mathcal{R}_{\bm{\beta}} (\bm{\beta}), 
\end{equation}
denote the regularized objective function in its base parametrization $(\bm{\psi}, \bm{\beta}) \in \mathbb{R}^{d_{\psi}} \times \mathbb{R}^{d}$, where $\bm{\beta}$ is an arbitrary subset of all model parameters, and $\bm{\psi}$ comprises the complementary components. 
In a typical empirical risk minimization setting, $\mathcal{L}$ can be written more explicitly as $\mathcal{L}(\bm{\psi}, \bm{\beta}) = \sum_{i=1}^{n} \mathscr{L}\left(\bm{y}_{i}, f({\bm{x}_i}|\bm{\psi},\bm{\beta}) \right)$, with independently sampled data $\mathcal{D}\triangleq\left\{\left(\boldsymbol{x}_{i}, \bm{y}_{i}\right)\right\}_{i=1}^{n}$, $(\bm{x}_i,\bm{y}_i) \in\mathcal{X} \times \mathcal{Y}$. Here,  
$\mathcal{X} \subseteq \mathbb{R}^{d_x}\,,\, \mathcal{Y}\subseteq \mathbb{R}^{d_y}$ denote generic feature and label spaces, $\mathscr{L}:\mathcal{Y}\times\mathbb{R}^{d_y} \to \mathbb{R}_{0}^{+}$ is an arbitrary loss contribution, and the (parametric) model $f: \mathcal{X} \to \mathbb{R}^{d_y}$ is 
parametrized by $(\bpsi,\bbeta)$.\\ 
%
The non-smooth and potentially non-convex regularizer $\Rbeta(\bbeta)$ is defined as $\mathcal{R}_{\bm{\beta}}: \mathbb{R}^{d} \to \mathbb{R}^{+}_{0},\, \bm{\beta} \mapsto \mathcal{R}_{\bm{\beta}}(\bm{\beta})$,
with $\lambda \geq 0$ controlling the amount of regularization. In this work, we consider classical norm-based sparsity-inducing regularizers.
\vspace{0.2cm}

\noindent \textbf{Optimization transfer}\, To transfer the optimization of $\mathcal{P}$ to a surrogate $\mathcal{Q}$, first consider a continuous and surjective parametrization of $\bm{\beta}$ defined by $\K:\Rdxi \to \Rd, \bxi \mapsto \K(\bxi)=\bbeta$. %
Moreover, we define a surrogate regularization function 
$\mathcal{R}_{\bm{\xi}}: \Rdxi \to \mathbb{R}^{+}_{0}, \bm{\xi} \mapsto \mathcal{R}_{\bm{\xi}}(\bm{\xi})$. %
Together, $(\Rbeta,\K,\Rxi)$ define our proposed two-step optimization transfer approach to construct an equivalent surrogate $\Q(\bpsi,\bxi)$ from the original objective $\P(\bpsi,\bbeta)$: 
\vspace{0.2cm}
\begin{definition}[Construction of surrogate \texorpdfstring{$\Q$}{Q}]\label{def:construction-smooth-Q}
Let the objective $\P(\bpsi,\bbeta)=\L(\bpsi,\bbeta)+\lambda \Rbeta(\bbeta)$ as in (\ref{eq:q-beginning}), $\mathcal{R}_{\bm{\beta}}(\bbeta)$ the non-smooth regularizer, $\K(\bxi)$ a parametrization of $\bbeta$, and $\Rxi(\bxi)$ a surrogate regularizer for $\bxi \in \Rdxi$. The variational surrogate $\mathcal{Q}$ can then be constructed from the tuple $(\Rbeta,\K,\Rxi)$ as follows. First, \textbf{i)} parametrize $\K(\bxi)=\bbeta$ to get a ``lifted'' $\P(\bpsi,\K(\bxi))$, and \textbf{ii)} substitute $\mathcal{R}_{\bm{\xi}}(\bm{\xi})$ for $\mathcal{R}_{\bm{\beta}}(\mathcal{K}(\bm{\xi}))$: 
\begin{equation} \label{eq:g-beginning}
\mathcal{Q}:\mathbb{R}^{d_{\psi}}\times \Rdxi \to \mathbb{R}_{0}^{+},\,(\bm{\psi},\bm{\xi}) \mapsto \mathcal{L}(\bm{\psi}, \mathcal{K}(\bm{\xi})) +  \lambda \mathcal{R}_{\bm{\xi}} (\bm{\xi}) \,.
\end{equation}

Further, if $\L$, $\K$, and $\Rxi$ are $\mathcal{C}^1$ functions, we call $\Q$ a smooth surrogate for $\P$.

\end{definition}

The next definition explicitly states our notion of equivalence between $\P$ and $\Q$:

\begin{definition}[Equivalence of optimization problems]\label{def:equivalence}
We say the two optimization problems $$\underset{\bm{\psi},\,\bm{\beta}}{\text{minimize}}\; \mathcal{P}(\bm{\psi},\bm{\beta})\quad\; \text{and}\;\quad \underset{\bm{\psi},\,\bm{\xi}}{\text{minimize}}\; \mathcal{Q}(\bm{\psi},\bm{\xi})\,,$$ are equivalent if the following conditions hold:
\begin{enumerate}[a)]
\item $\underset{\bm{\psi},\,\bm{\beta}}{\inf}\; \mathcal{P}(\bm{\psi},\bm{\beta}) = \underset{\bm{\psi},\,\bm{\xi}}{\inf}\, \mathcal{Q}(\bm{\psi},\bm{\xi})$, i.e., their globally optimal values coincide. 
\item If $(\hat{\bm{\psi}},\hat{\bm{\beta}})$ is a local minimizer of $\mathcal{P}(\bm{\psi},\bm{\beta})$, then there is a local minimizer $(\hat{\bm{\psi}},\hat{\bm{\xi}})$ of $\mathcal{Q}(\bm{\psi},\bm{\xi})$ with $\hat{\bm{\xi}}\in \mathcal{K}^{-1}(\hat{\bm{\beta}})$ and $\mathcal{Q}(\hat{\bm{\psi}}, \hat{\bm{\xi}}) = \mathcal{P}(\hat{\bm{\psi}}, \hat{\bm{\beta}})$. 
\item If $(\hat{\bm{\psi}},\hat{\bm{\xi}})$ is a local minimizer of $\mathcal{Q}(\bm{\psi},\bm{\xi})$, then $(\hat{\bm{\psi}},\hat{\bm{\beta}})$ with $\hat{\bm{\beta}}=\mathcal{K}(\hat{\bm{\xi}})$ is a local minimizer of $\mathcal{P}(\bm{\psi},\bm{\beta})$ and $\mathcal{Q}(\hat{\bm{\psi}}, \hat{\bm{\xi}}) = \mathcal{P}(\hat{\bm{\psi}}, \hat{\bm{\beta}})$.
\end{enumerate}
\end{definition}

Equivalence of local minima is particularly important for non-convex regularization, encompassed by the second and third conditions, as finding global minima in non-convex optimization is challenging and, for the most part, intractable. Moreover, in non-convex regularization, local minima have been observed to generalize similarly or even better than global minima on test data \citep{chartrand2008iteratively, chartrand2007exact, olsson2017non}. \\
This so-called matching of local minima \citep{levin2020towards} ensures that the transfer from the original objective $\P$ to a surrogate objective $\Q$ preserves all the properties of the local minima structure of $\P$. By focusing only on global minima, important information about the structure of the problem 
is neglected. Importantly, we do not 
introduce spurious local minima in $\mathcal{Q}$, which would 
artificially increase the difficulty of the optimization problem. 
Through the matching property, we can further use the surjection $\K(\bxi)$ to reconstruct all (local) minimizers of $\mathcal{P}$ from local minimizers of $\mathcal{Q}$ as $\K(\hbxi)=\hbbeta$. To guarantee this matching of local minima property for $\P(\bpsi,\bbeta)$ under the parametrization $\P(\bpsi,\K(\bxi))$, local openness of the parametrization mapping $\K(\bxi)$ at all local minimizers $\hat{\bxi}$ of $\P(\bpsi,\K(\bxi))$ is a crucial property \citep{nouiehed2022learning}. \citet{levin2024effect} show that it is both a necessary and sufficient condition for the preservation of local minima under the parametrization $\K(\bxi)=\bbeta$.\footnote{Local openness is closely related to, but distinct from the notion of continuity, which is defined as 
$\forall\, \varepsilon>0 \, \exists \, \delta>0:\,\K(\mathcal{B(\bxi)},\delta) \subseteq \mathcal{B}(\K(\bxi),\varepsilon)$ using the same notation. }
%
\begin{definition}[Local openness]\label{def:openness}
A mapping $\mathcal{K}: \Rdxi \to \mathbb{R}^{d},\, \bm{\xi} \mapsto \K(\bxi)$ is locally open at $\bxi$ if for every $\varepsilon>0$ we can find $\delta>0$ such that $\mathcal{B}(\K(\bxi), \delta) \subseteq \K(\mathcal{B}(\bxi, \varepsilon))$. Further, the map $\K$ is called globally open if it is locally open at all $\bxi \in \Rdxi$. 
\end{definition}
%
\textbf{General theoretical results}\,
Using the function characterizations encapsulated in Definition~\ref{def:construction-smooth-Q}, we can now prove the following results. Note that these hold for any potentially unregularized objective $\P(\bpsi,\bbeta)$ under reparametrization:
\vspace{0.2cm}
\begin{lemma}\label{lemma:p_to_pk} If $(\hat{\bpsi},\hat{\bbeta})$ is a local minimizer of $\P(\bpsi,\bbeta)$, and $\K(\bxi)$ is a continuous surjection, then all $(\hat{\bpsi},\hat{\bxi})$ such that $\hat{\bxi}\in\K^{-1}(\hat{\bbeta})$ are local minimizers of $\P(\bpsi,\K(\bxi))$ with $\P(\hat{\bpsi},\hat{\bbeta})=\P(\hat{\bpsi},\K(\hat{\bxi}))$.
\end{lemma}
\vspace{0.2cm}
\begin{lemma}\label{lemma:pk_to_p}
    If $(\hat{\bpsi},\hat{\bxi})$ is a local minimizer of $\P(\bpsi,\K(\bxi))$, and the continuous surjection $\K(\bxi)$ is locally open at $\hbxi$, then $(\hat{\bpsi},\K(\hat{\bxi}))=(\hbpsi,\hbbeta)$ is a local minimizer of $\P(\bpsi,\bbeta)$ with $\P(\hat{\bpsi},\hat{\bbeta})=\P(\hat{\bpsi},\K(\hat{\bxi}))$.
\end{lemma}
\vspace{0.15cm}
Their proofs are given in Appendices~\ref{app:p_to_pk} and \ref{app:pk_to_p}. Together, both results show that the set of local minima of $\P(\bpsi,\bbeta)$ and $\P(\bpsi,\K(\bxi))$ are equal 
if $\K(\bxi)$ is locally open at all local minimizers of $\P(\bpsi,\K(\bxi))$, and the local minimizers are related via $(\hbpsi,\K(\hbxi))=(\hbpsi,\hbbeta)$. Still, for non-smooth regularizers, smoothly parametrizing $\bbeta$ will not result in a smooth optimization problem. Nevertheless, the results using local openness of the map $\K(\bxi)$ can be applied to guarantee matching local minima under smooth and surjective parametrizations in general problems, e.g., for parametrizations used in the implicit regularization literature. 

\noindent In our optimization transfer approach, however, we further replace the para\-met\-rized regularizer $\Rbeta(\K(\bxi))$ by $\Rxi(\bxi)$ to obtain the surrogate objective $\Q(\bpsi,\bxi)$. The surrogate penalty $\Rxi$ and the non-smooth regularizer $\Rbeta$ are related as follows:
\vspace{0.2cm}
\begin{definition}[Smooth variational form] \label{def:smooth-varform}
A (smooth) variational form 
is an expression of a function $\Rbeta(\bbeta)$ as the minimum of a (smooth) surrogate $\Rxi(\bxi)$ over a feasible set given by the fiber $\K^{-1}(\bbeta)$ of a 
surjective parametrization $\K(\bxi)$ at $\bbeta$, 
i.e.,
\vspace{-0.0cm}
\begin{equation}\label{eq:svg}\Rbeta(\bbeta)=\min_{\bxi:\K(\bxi)=\bbeta}\Rxi(\bxi)\quad \forall \bbeta \in \Rd\,.
\end{equation}
\end{definition}
\vspace{-0.2cm}
By definition, $\Rxi(\bxi)$ majorizes $\Rbeta(\K(\bxi))$, i.e.,  $\Rxi(\bxi)\geq\Rbeta(\K(\bxi))\forall\bxi \in \Rdxi$. Importantly, finding the appropriate SVF is non-trivial and will be derived for each considered regularizer $\Rbeta$ in the respective section. To give a canonical example, the $\ell_1$ penalty $\Rbeta(\bbeta)=2\|\bbeta\|_1$ can be expressed as the minimum of $\Rxi(\bu,\bv)=\|\bu\|_2^2+\|\bv\|_2^2$, where $\bxi=(\bu,\bv)^{\top} \in \mathbb{R}^{2d}$, subject to the parametrization $\K(\bu,\bv)=\bu \odot \bv=\bbeta$ for any $\bbeta \in \Rd$ (cf.~Section~\ref{sec:hpp-vanilla}). It should be noted that the SVFs for an arbitrary regularizer $\Rbeta$ are not necessarily unique, as evidenced by several distinct SVFs, defined by pairs $(\K,\Rxi)$, yielding the same induced regularizer $\Rbeta$ in Table~\ref{tab:overview}. Further, determining conditions for the existence of an SVF, defined by a pair $(\K,\Rxi)$, for \textit{arbitrary} regularizers $\Rbeta$, is challenging, primarily as the induced regularizer is determined jointly by the parametrization $\K$ and the surrogate regularizer $\Rxi$. Hence, we leave this relevant question to future research.

To preserve local minima in our approach with replaced surrogate regularization $\Rxi$, we further require stability of the solutions $\hbxi(\bbeta)$ to the SVF with respect to the regularized parameter $\bbeta$, i.e., continuous dependence of the minimizers $\hbxi \in \argmin_{\bxi:\K(\bxi)=\bbeta}\Rxi(\bxi)$ on $\bbeta$, formalized through lower hemicontinuity of the set-valued solution mapping:
\vspace{0.1cm}
\begin{definition}[Lower hemicontinuity]\label{def:lhc}
A set-valued map $\hbxi:\Rd \rightrightarrows \Rdxi,\, \bbeta \mapsto \hbxi(\bbeta)$, is said to be lower hemicontinuous (l.h.c.) at $\bbeta \in \Rd$ if 
\vspace{-0.3cm}
$$
\forall\,\bxi \in \hbxi(\bbeta)\,\forall\,\varepsilon > 0\,\exists\,\delta > 0:\,\forall\,\tbbeta \in \B(\bbeta, \delta)\,\exists\,\tbxi \in \hbxi(\tbbeta)\,\cap \B(\bxi, \varepsilon).
$$
Hence, for every point $\bxi \in \hbxi(\bbeta)$ and every $\varepsilon > 0$, there exists $\delta > 0$ such that for all $\tbbeta \in \B(\bbeta, \delta)$, the set $\hbxi(\tbbeta)$ contains at least one point in the ball $\B(\bxi, \varepsilon)$. 
\end{definition}
%

%
Instead of requiring local openness, as previously for parametrizations without a change of regularizers, we use the lower hemicontinuity of the solution map $\hbxi(\bbeta)$, a property that is easily obtained as a by-product in the construction of our smooth variational forms. For details on set-valued analysis, we refer to \citet{aubin2009set}. In the following, we state a minimal but sufficient set of assumptions to establish matching of local minima, aiming to remain as general as possible.
\vspace{0.2cm}
\begin{assumption}[Minimal assumptions for optimization transfer]\label{ass:min}
Let $\P(\bpsi,\bbeta)=\L(\bpsi,\bbeta)+\lambda \Rbeta(\bbeta)$ be the original objective (\ref{eq:q-beginning}) and $\Q(\bpsi,\bxi)=\L(\bpsi,\K(\bxi))+\lambda \Rxi(\bxi)$ the surrogate (\ref{eq:g-beginning}). Let $\K(\bxi)=\bbeta$ be a continuous surjection, $\Rxi(\bxi)$ a surrogate regularizer such that $\Rbeta(\bbeta) =\min_{\bxi: \K(\bxi)=\bbeta}\Rxi(\bxi) \, \forall \bbeta$ and all constrained minima are global, and let the set-valued solution map $\bbeta \mapsto \argmin_{\K(\bxi)=\bbeta} \Rxi(\bxi)=\hbxi(\bbeta)$ be l.h.c.
\end{assumption}
\vspace{0.2cm}
\begin{lemma}\label{lemma:p_to_q}
    If $(\hbpsi,\hbbeta)$ is a local minimizer of $\P(\bpsi,\bbeta)$,
    then all $(\hbpsi,\hbxi)$ such that $\hbxi \in \argmin_{\bxi:\K(\bxi)=\hbbeta} \Rxi(\bxi)$ are local minimizers of $\Q(\bpsi,\bxi)$ with $\Q(\hbpsi,\hbxi)=\P(\hbpsi,\hbbeta)$ under Assumption~\ref{ass:min}.
\end{lemma}
\vspace{0.2cm}
\begin{lemma}\label{lemma:q_to_p}
    If $(\hbpsi,\hbxi)$ is a local minimizer of $\Q(\bpsi,\bxi)$, %
    then $(\hbpsi,\K(\hbxi))=(\hbpsi,\hbbeta)$ is a local minimizer of $\P(\bpsi,\bbeta)$ with $\Q(\hbpsi,\hbxi)=\P(\hbpsi,\hbbeta)$ under Assumption~\ref{ass:min}.
\end{lemma}
In the proof of Lemma~\ref{lemma:q_to_p}, it is also established that for all local minimizers $(\hbpsi,\hbxi)$ of $\Q(\bpsi,\bxi)$, $\hbxi$ must also minimize the SVF over the fiber $\K^{-1}(\K(\hbxi))$. Figure~\ref{fig:tikz-balls} provides some intuition behind the preceding results, assuming no additional parameters $\bpsi$.
%
%
It illustrates the relationship between a local minimizer $\hbxi$ of $\Q(\bxi)$ and the corresponding local minimizer $\K(\hbxi)=\hbbeta$ of $\P(\bbeta)$. Note that at points $\tbxi$ around $\hbxi$ that are in the image of $\hbxi(\bbeta)$ (dashed green), we have equality of $\Q(\tbxi)$ and $\P(\tbbeta)$. 
By Lemma~\ref{lemma:p_to_pk}, if $\hbbeta$ is a local minimizer of $\P(\bbeta)$, then any $\hbxi \in \K^{-1}(\hbbeta)$ (red curve) is a local minimizer of the non-smooth overparametrized $\P(\K(\bxi))$. But only those $\hbxi(\hbbeta)\in\argmin_{\bxi:\K(\bxi)=\hbbeta}\Rxi(\bxi)\subset \K^{-1}(\hbbeta)$ (red dot at vertex) 
are also local minimizers of $\Q(\bxi)$ due to the majorization property $\P(\K(\bxi))\leq\Q(\bxi)\forall \bxi \in \Rdxi$, combined with $\Q(\hbxi)=\P(\K(\hbxi))$. Conversely, if $\hbxi$ is a local minimizer of $\Q(\bxi)$, then by continuity of the solution map $\hbxi(\bbeta)$ at $\K(\hbxi)=\hbbeta$, if there existed $\tbbeta\in\mathcal{B}(\hbbeta,\delta)$ such that $\P(\tbbeta)<\P(\hbbeta)$, this would imply existence of $\tbxi\in\mathcal{B}(\hbxi,\varepsilon)$ with $\Q(\tbxi)<\Q(\hbxi)$, contradicting that $\hbxi$ is a local minimizer of $\Q(\bxi)$. Figure~\ref{fig:tikz-contours-hpp} shows a specific choice of functions for $\K$ and $\Rxi$.

Assumption~\ref{ass:min} requires only a continuous surjection $\K$ and, in principle, poses no restrictions on the functions $\L$ and $\Rxi$, as long as the variational expression $\Rbeta$ holds for any $\bbeta \in \Rd$ and $\hbxi(\bbeta)$ is l.h.c. However, without further smoothness assumptions, the surrogate $\Q$ might have the same differentiability issues as the base problem. So from now on, we only consider cases where $\L,\K,$ and $\Rxi$ are smooth, so that $\P(\bpsi,\bbeta)$ is non-smooth, but $\Q(\bpsi,\bxi)$ is smooth, and we can tackle the problem with (S)GD. The previous results let us now state our main result:
%
%
%
\begin{figure}[t!]
    \centering
    \subfloat[Schematic illustration of smooth optimization transfer.]{
        \resizebox{0.55\textwidth}{!}{%
            \begin{tikzpicture}\label{fig:tikz-balls}
                \definecolor{colordk}{rgb}{0, 0.5, 0}
                \draw[line width=1.2pt] (0,0) circle (3cm);
                \draw[line width=1.2pt] (8,0) circle (1.5cm);
                \fill[red] (0,0) circle (4pt);
                \fill[blue] (0.32,-0.33) circle (4pt);
                \fill[red] (-1.7,1) circle (4pt);
                \fill[blue] (-0.9,-1.2) circle (4pt);
                \fill[red] (8,0) circle (4pt);
                \fill[blue] (7.5,-0.65) circle (4pt);
                \node at (0,-3.7) {\huge$\mathcal{B}(\textcolor{red}{\hat{\bm{\xi}}}, \varepsilon)$};
                \node at (8,-3.7) {\huge$\mathcal{B}(\textcolor{red}{\hat{\bm{\beta}}}, \delta)$};
                \node[red] at (-1.1,-0.15) {\huge$\hat{\bm{\xi}}$};
                \node[blue] at (0,-0.91) {\huge$\tilde{\bm{\xi}}$};
                \node[red] at (8.6,0.5) {\huge$\hat{\bm{\beta}}$};
                \node[blue] at (8.2,-0.65) {\huge$\tilde{\bm{\beta}}$};
                \node[red] at (4.6,2.4) {\huge$\mathcal{K}(\hat{\bm{\xi}})$};
                \node[blue] at (3.4,0.1) {\huge$\mathcal{K}(\tilde{\bm{\xi}})$};
                \node[red] at (-4.1,1.3) {\huge$\mathcal{K}^{-1}(\hat{\bm{\beta}})$};
                \node[blue] at (-3.5,-2) {\huge$\mathcal{K}^{-1}(\tilde{\bm{\beta}})$};
                \node[colordk] at (4.1,-1) {\huge$\hat{\bm{\xi}}(\textcolor{blue}{\tilde{\bm{\beta}}})$};
                \node[colordk, rotate=0] at (-1.65,2.75) {\huge$\hat{\bm{\xi}}(\textcolor{black}{\bm{\beta}}) \in \arg\min_{\bm{\xi}: \mathcal{K}(\bm{\xi})=\bm{\beta}} \mathcal{R}_{\bm{\xi}}(\bm{\xi})$};
                \draw[red, line width=0.55mm] plot[domain=-1.95:1.95] ({-1*cosh(\x)+1}, {0.4*sinh(\x)});
                \draw[blue, line width=0.55mm] plot[domain=-1.93:1.93] ({-1.1*cosh(\x)+1.47}, {0.4*sinh(\x)-0.4});
                \draw[colordk, dashed, line width=0.6mm] (-2.1, 2.1) -- (2.1, -2.1);
                \draw[->, red, dashed, bend left=30, line width=0.85mm] (0,0) to (7.8,0);
                \draw[->, blue, dashed, bend left=30, line width=0.85mm] (0.42,-0.28) to (7.3,-0.5);
                \draw[->, red, dashed, bend left=30, line width=0.85mm] (-1.7,1) to (7.8,0.2);
                \draw[->, colordk, dashed, bend left=30, line width=0.85mm] (7.4,-0.65) to (0.6,-0.4);
                \draw[->, blue, dashed, bend right=30, line width=0.85mm] (-0.9,-1.2) to (7.5,-0.85);
            \end{tikzpicture}
        }
    }
    \hspace{1.5cm}
    \subfloat[Contours of HPP in (\ref{eq:hpp-def}) and surrogate $\ell_2$ penalty.]{
        \resizebox{0.27\textwidth}{!}{%
            \begin{tikzpicture}\label{fig:tikz-contours-hpp}
                \begin{scope}
                    \clip (-9,-9) rectangle (9,9);
                    \definecolor{colordk}{rgb}{0, 0.5, 0}
                    \definecolor{color0}{rgb}{1.0, 0.7, 0.6}
                    \definecolor{color1}{rgb}{1.0, 0.85, 0.7} 
                    \definecolor{color2}{rgb}{1.0, 0.9, 0.8}
                    \definecolor{color3}{rgb}{1.0, 0.93, 0.87} 
                    \definecolor{color4}{rgb}{1.0, 0.96, 0.92}
                    \definecolor{color5}{rgb}{1.0, 0.99, 0.98}
                    \definecolor{color6}{rgb}{1.0, 0, 0}
                    \draw[fill=color6, line width=0.44pt] (0,0) circle (12);
                    \draw[fill=color5, line width=0.44pt] (0,0) circle (13);
                    \draw[fill=color4, line width=0.44pt] (0,0) circle (9.5);
                    \draw[fill=color3, line width=0.44pt] (0,0) circle (7.9);
                    \draw[fill=color2, line width=0.44pt] (0,0) circle (6.2);
                    \draw[fill=color1, line width=0.44pt] (0,0) circle (4.5);
                    \draw[fill=color0, line width=0.44pt] (0,0) circle (3);
                    \draw[darkgray, line width=1.3mm] (-9,-9) rectangle (9,9);
                    \draw[thick,->,darkgray] (-9,0) -- (10,0) node[below]{\LARGE$u$};
                    \draw[thick,->,darkgray] (0,-9) -- (0,10) node[left]{\LARGE$v$};
                    \draw[red, line width=1.3mm, domain=-9:-0.2, samples=100, smooth] plot (\x,{4.5/\x});
                    \draw[red, line width=1.3mm, domain=0.2:9, samples=100, smooth] plot (\x,{4.5/\x});
                    \draw[blue, line width=1.3mm, domain=-9:-0.2, samples=100, smooth] plot (\x,{10.125/\x});
                    \draw[blue, line width=1.3mm, domain=0.2:9, samples=100, smooth] plot (\x,{10.125/\x});
                    \fill[red] (-2.1,-2.1) circle (9pt);
                    \fill[red] (2.1,2.1) circle (9pt);
                    \fill[blue] (-3.2,-3.2) circle (9pt);
                    \fill[blue] (3.2,3.2) circle (9pt);
                    \draw[colordk, dashed, line width=1.25mm, dash pattern=on 8pt off 8pt] (-9.1,-9.1) -- (9.1,9.1);
                    \draw[colordk, dashed, line width=1.5mm, dash pattern=on 8pt off 8pt] (-9.1,9.1) -- (9.1,-9.1);
                    \fill[black] (0,0) circle (6pt);
                    \node[red] at (-1.7,-0.65) {\scalebox{3}{\LARGE$\hat{\bm{\xi}}$}};
                    \node[blue] at (-3.5,-4.75) {\scalebox{3}{\LARGE$\tilde{\bm{\xi}}$}};
                    \node[colordk] at (-0.6,6.78) {\scalebox{2.2}{\LARGE$
                    \hat{\bm{\xi}}(\textcolor{black}{\bm{\beta}}) \in \underset{\bm{\xi}:\mathcal{K}(\bm{\xi})=\bm{\beta}}{\textbf{argmin}} \mathcal{R}_{\bm{\xi}}(\bm{\xi})
                    $}};
                \end{scope}
                \node at (0,-9.5) {\fontsize{100}{80}\selectfont \LARGE$\bm{u}$};
                \node at (-9.5,0) {\fontsize{100}{80}\selectfont \LARGE$\bm{v}$};
            \end{tikzpicture}
        }
    }
    \caption{{\small Relationship between local minimizer $\hbxi$ of $\Q$, the induced minimizer $\K(\hbxi)=\hbbeta$ of $\P$, and the cont. solution mapping $\hbxi(\bbeta)$ of the SVF. \textbf{Left}: solid curves show two fibers $\K^{-1}(\hbbeta)$ (red) and $\K^{-1}(\tbbeta)$ (blue). The solution map $\hbxi(\bbeta)$ (dashed green) maps to minimizers of the SVF for varying $\bbeta$. \textbf{Right}: concrete example showing scalar parametrization $\beta_j = \K(u_j,v_j)$ with surrogate $\ell_2$ penalty.}}
    \label{fig:proof-viz}
    \vspace{-0.3cm}
\end{figure}
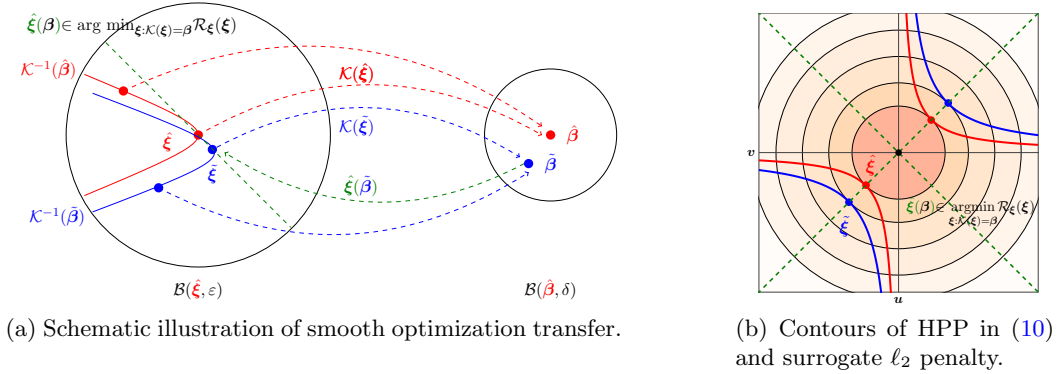

%
%
%
\vspace{-0.0cm}

\begin{theorem}[Smooth optimization transfer for sparse regularization] \label{theorem-general}\
Let the non-smooth objective $\mathcal{P}(\bpsi,\bbeta)$ and its smooth surrogate $\mathcal{Q}(\bpsi,\bxi)$ be defined as in Equations (\ref{eq:q-beginning}) and (\ref{eq:g-beginning}). Under Assumption~\ref{ass:min} the optimization problems
\begin{align}
\underset{\bm{\psi},\,\bm{\beta}}{\text{minimize}}\;\P(\bpsi,\bbeta) &\triangleq \L(\bpsi,\bbeta)+\lambda \Rbeta(\bbeta)\,, \\
\underset{\bm{\psi},\,\bxi}{\text{minimize}}\;\Q(\bpsi,\bxi) &\triangleq \L(\bpsi,\K(\bxi))+\lambda \Rxi(\bxi) \,,   
\end{align}
are equivalent by Definition~\ref{def:equivalence}.
\end{theorem}
\begin{proof}
For the first point of Definition~\ref{def:equivalence}, we show that the infima of both $\P(\bpsi,\bbeta)$ and $\Q(\bpsi,\bxi)$ coincide. Because $\L(\bpsi,\K(\bxi))$ is constant on the fiber of $\K$ at $\bbeta=\K(\bxi)$, we can pull in the infimum and re-state it in terms of $\bbeta$:
\begin{align}
\inf_{\bm{\psi}, \bm{\xi}} \mathcal{Q}(\bm{\psi}, \bm{\xi}) = \inf_{\bm{\psi}, \bm{\xi}} \big\{ \mathcal{L}(\bm{\psi}, \mathcal{K}(\bm{\xi})) +  \lambda \cdot \mathcal{R}_{\bm{\xi}} (\bm{\xi}) 
\big\}  
= \inf_{\bm{\psi},\bm{\beta}}  \big\{ \mathcal{L}(\bm{\psi}, \bm{\beta}) +  \lambda \inf_{\bm{\xi}: \mathcal{K}(\bm{\xi})=\bm{\beta}} \{\mathcal{R}_{\bm{\xi}} (\bm{\xi}) \} \big\} \nonumber
 \end{align}
By Assumption~\ref{ass:min}, we have $\inf_{\bxi:\K(\bxi)=\bbeta}\Rxi(\bxi)=\min_{\bxi:\K(\bxi)=\bbeta}\Rxi(\bxi)=\Rbeta(\bbeta)$, and thus
\begin{equation}
     \inf_{\bm{\psi}, \bm{\xi}} \mathcal{Q}(\bm{\psi}, \bm{\xi}) =  \inf_{\bm{\psi}, \bm{\beta}} \left\{ \mathcal{L}(\bm{\psi}, \bm{\beta}) +  \lambda \mathcal{R}_{\bm{\beta}}(\bm{\beta}) \right\} = \inf_{\bm{\psi}, \bm{\beta}} \mathcal{P}(\bm{\psi}, \bm{\beta})\,. \nonumber
\end{equation}
This shows the first point. For the second and third points of Definition~\ref{def:equivalence}, we can apply Lemma~\ref{lemma:p_to_q} together with Lemma~\ref{lemma:q_to_p} under Assumption~\ref{ass:min} to obtain the required matching of local minima with corresponding minimizers.
\end{proof}
\vspace{0.2cm}
From Theorem~\ref{theorem-general} it follows that there is a surjective mapping from the set of local minimizers of $\Q$ to the set of local minimizers of $\P$, obtained by restricting the domain of the parametrization $\mathcal{K}$ to the set of local minimizers of $\Q$. Table~\ref{tab:overview} shows an (incomplete) summary of the different parametrizations $\mathcal{K}(\bm{\xi})$ of $\bm{\beta}$ that can be represented in our framework, together with the sparse regularization terms $\mathcal{R}_{\bm{\beta}}$ that are induced by applying the smooth regularizers $\mathcal{R}_{\bm{\xi}}$ to the surrogate parameters $\bm{\xi}$.
{\small
\begin{table}[t!]
\resizebox{1.0\textwidth}{!}{
{\renewcommand{\arraystretch}{1.2}%
\begin{tabular}{lccccc}
\hline\hline
\multicolumn{1}{c}{Abbreviation} & \multicolumn{1}{c}{$\mathcal{K}(\cdot)=\bm{\beta}$} & \multicolumn{1}{c}{$\mathcal{R}_{\bm{\xi}}$} & \multicolumn{1}{c}{$\mathcal{R}_{\bm{\beta}}$} & \multicolumn{1}{c}{Type} & \multicolumn{1}{c}{Ref.} \\ \hline
HPP    & $\bm{u}\odot\bm{v}$  & $\Vert \bm{u} \Vert_{2}^{2} + \Vert \bm{v} \Vert_{2}^{2}$ & $2\Vert \bm{\beta} \Vert_1$ & $\ell_1$ & \cite{grandvalet1998least,hoff2017lasso} \\
HDP    & $\bm{\gamma}^2-\bm{\delta}^2$  & \textcolor{NavyBlue}{$\Vert \bm{\gamma} \Vert_{2}^{2} + \Vert \bm{\delta} \Vert_{2}^{2}$}  & \textcolor{NavyBlue}{$\Vert \bm{\beta} \Vert_1$}  & \textcolor{NavyBlue}{$\ell_1$} & \cite{vaskevicius2019implicit} \\
GHPP  & $\bm{u}\odotg\bm{\nu}$   & $\sum_{j=1}^{L} (\Vert \bm{u}_j \Vert_{2}^{2} + \nu_j^2)$  & $2\Vert\bm{\beta}\Vert_{2,1}$  & $\ell_{2,1}$ & \cite{tibs2021} \\
\textcolor{NavyBlue}{\textbf{Adj. GHPP}}  &  \textcolor{NavyBlue}{$\bm{u}\odotg\bm{\nu}$} & \textcolor{NavyBlue}{$\sum_{j=1}^{L} (\Vert \bm{u}_j \Vert_{2}^{2} +|\mathcal{G}_j| \nu_j^2)$} & \textcolor{NavyBlue}{$2\sum_{j=1}^{L} \sqrt{|\mathcal{G}_j|} \Vert \bm{\beta}_j \Vert_2$}  & \textcolor{NavyBlue}{$\ell_{2,1}$} & - \\ \cline{1-6}
\multicolumn{6}{l}{$k \in \mathbb{N},\; k_1 \in \mathbb{N},\; k_2\triangleq k-k_1 \in \mathbb{N}:$} \\ 
$\text{HPP}_{k}$  & $\bigodot_{l=1}^{k} \bm{u}_{l}$ & $\sum_{l=1}^{k} \Vert \bm{u}_{l} \Vert_{2}^{2}$  & $k \Vert \bm{\beta} \Vert_{2/k}^{2/k}$  & $\ell_{2/k}$ & \cite{hoff2017lasso} \\
$\text{GHPP}_{k}$ & $\bm{u} \odotg \bm{\nu}_{r}^{\odot(k-1)}$ & $\sum_{j=1}^{L} \Vert \bm{u}_j \Vert_{2}^{2} + \sum_{r=1}^{k-1} \nu_{jr}^2$  & $k \Vert \bm{\beta} \Vert_{2,2/k}^{2/k}$  & $\ell_{2,2/k}$ & \cite{tibs2021} \\
$\text{GHPP}_{k_1,k}$  & $\bm{\mu}_{t}^{\odot  k_1} \odotg \bm{\nu}_{r}^{\odot  k_2}$  & $\sum_{t=1}^{k_1} \Vert \bm{\mu}_{t}\Vert_{2}^{2}+\sum_{r=1}^{k_2} \Vert \bm{\nu}_{r} \Vert_{2}^{2}$  & $k \Vert \bm{\beta} \Vert_{2/k_1,2/k}^{2/k}$  & $\ell_{2/k_1,2/k}$ & \cite{dai2021representation} \\
$\text{HDP}_{k}$ & $\bm{u}_{l}^{\odot  k} - \bm{v}_{l}^{\odot  k}$  & \textcolor{NavyBlue}{$\sum_{l=1}^{k} \Vert \bm{u}_{l} \Vert_{2}^{2} + \Vert \bm{v}_{l} \Vert_{2}^{2}$}  & \textcolor{NavyBlue}{$k\Vert \bm{\beta} \Vert_{2/k}^{2/k}$}  & \textcolor{NavyBlue}{$\ell_{2/k}$}  & \cite{chou2023more} \\
\textcolor{NavyBlue}{\textbf{$\text{HPP}_{k}^{shared}$}} & \textcolor{NavyBlue}{$\bm{u} \odot \bm{v}^{k-1}$}  & \textcolor{NavyBlue}{$\Vert \bm{u} \Vert_{2}^{2} + (k-1) \Vert \bm{v} \Vert_{2}^{2}$}  & \textcolor{NavyBlue}{$k \Vert \bm{\beta} \Vert_{2/k}^{2/k}$}  & \textcolor{NavyBlue}{$\ell_{2/k}$} & - \\
$\text{HDP}_{k}^{shared}$ & $\bm{u}^k - \bm{v}^k$  & \textcolor{NavyBlue}{$\Vert \bm{u} \Vert_{2}^{2} + \Vert \bm{v} \Vert_{2}^{2}$}  & \textcolor{NavyBlue}{$\Vert \bm{\beta} \Vert_{2/k}^{2/k}$}  &  \textcolor{NavyBlue}{$\ell_{2/k}$} & \cite{chou2023more} \\ \cline{1-6}
\multicolumn{6}{l}{$k \in \mathbb{R}_{>2} ,\; k_1 \in \mathbb{R}_{>1},\; k_2\triangleq k-k_1 \in \mathbb{R}_{>1} ,\, \, (k \in \mathbb{R}_{>1}\,\text{ for Powerprop.}):$} \\ 
\textcolor{NavyBlue}{\textbf{$\text{HPowP}_{k}$}} & \textcolor{NavyBlue}{$\bm{u} \odot |\bm{v}|^{\circ (k-1)}$}  & \textcolor{NavyBlue}{$\Vert \bm{u} \Vert_{2}^{2} + (k-1) \Vert \bm{v} \Vert_{2}^{2}$}  & \textcolor{NavyBlue}{$k \Vert \bm{\beta} \Vert_{2/k}^{2/k}$} & \textcolor{NavyBlue}{$\ell_{2/k}$} & - \\
Powerprop. & $\bm{v} \odot |\bm{v}|^{\circ (k-1)}$ & \textcolor{NavyBlue}{$\Vert \bm{v} \Vert_{2}^{2}$}  & \textcolor{NavyBlue}{$\Vert \bm{\beta} \Vert_{2/k}^{2/k}$}  & \textcolor{NavyBlue}{$\ell_{2/k}$} & \cite{schwarz2021powerpropagation} \\
\textcolor{NavyBlue}{\textbf{$\text{GHPowP}_{k}$}} & \textcolor{NavyBlue}{$\bm{u} \odotg |\bm{\nu}|^{\circ (k-1)}$} & \textcolor{NavyBlue}{$\Vert \bm{u} \Vert_{2}^{2} + (k-1) \Vert \bm{\nu} \Vert_{2}^{2}$} & \textcolor{NavyBlue}{$k \Vert \bm{\beta} \Vert_{2,2/k}^{2/k}$} & \textcolor{NavyBlue}{$\ell_{2,2/k}$}  & - \\
\textcolor{NavyBlue}{\textbf{$\text{GHPowP}_{k_1,k}$}} & \textcolor{NavyBlue}{$(\bm{\mu} \odot |\bm{\mu}|^{\circ (k_1-1)}) \odotg |\bm{\nu}|^{\circ k_2}$} & \textcolor{NavyBlue}{$k_1 \Vert \bm{\mu} \Vert_{2}^{2} + k_2 \Vert \bm{\nu} \Vert_{2}^{2}$} & \textcolor{NavyBlue}{$\Vert \bm{\beta} \Vert_{2/k_1, 2/k}^{2/k}$} & \textcolor{NavyBlue}{$\ell_{2/k_1,2/k}$}  & - \\ \hline\hline \vspace{0.05cm}
\end{tabular}
}}
\caption{\small Overview of induced regularizers $\mathcal{R}_{\bm{\beta}}$ obtained by parametrizing $\bm{\beta}$ through $\mathcal{K}(\bm{\xi})$ and adding a smooth surrogate penalty $\mathcal{R}_{\bm{\xi}}$. 
The letter ``H" stands for ``Hadamard", the letter ``G" for ``Group", ``PP" for ``Product Parametrization", ``DP" for ``Difference Parametrization", and ``PowP" abbreviates ``Power Parametrization". Novel results in \textcolor{NavyBlue}{blue}.}
\label{tab:overview}
\vspace{-0.5cm}
\end{table}
}
It should be noted that although this work only considers convex and non-convex regularizers based on (quasi-)norm- and mixed-norm penalties $\Rbeta$, Theorem~\ref{theorem-general} provides a more general result that holds for any smooth optimization transfer fulfilling the assumptions.

To concretize the setting for the remainder of our work, given a partition $\mathcal{G}$ of parameter indices $[d]$ into $L \leq d$ groups, we consider sparsity-inducing regularizers of the form
%
${\small \textstyle \mathcal{R}_{\bm{\beta}}(\bm{\beta}) \in \Big\{ \mathcal{R}: \mathbb{R}^{d} \to \mathbb{R}_0^{+}\,, \bm{\beta} \mapsto \sum_{j=1}^{L} \omega_j \norm{\bm{\beta}_j}_p^q \;\big|\; 0<q \leq p\leq2 ,\; \omega_j > 0 \; \forall \,j 
\Big\}}$.
Note that setting $p=q$, $L=d$ and $\omega_j=1\,\forall\, j$ reduces the expression to the familiar $\ell_q$ regularizer $\Rbeta(\bbeta)=\norm{\bbeta}_q^q$. Merely setting $\omega_j=1$ results in the $\ell_{p,q}$ regularizer, whereas, e.g., {\scriptsize$\omega_j =  \sqrt{|\mathcal{G}_j|}$}$\in\mathbb{N}, p=2, q=1$ yields the $\ell_{2,1}$ group lasso \citep{yuan2006model}. %
Similarly, we take the smooth surrogate regularizers $\Rxi$ to be of the form
{\normalsize
${\textstyle \mathcal{R}_{\bm{\xi}}(\bm{\xi}) \in \big\{ \mathcal{R}: \Rdxi \to \mathbb{R}_0^{+},\; \bxi \mapsto \sum_{j=1}^{\dxi} \tilde{\omega}_j \; \xi_j^2 \;\big|\; \tilde{\omega}_j > 0 \; \forall \,j 
\big\}}$.
}

As the parametrizations considered in this work are based on Hadamard products and variations thereof, the following smoothness and separability assumptions on $\K$, as well as a specific monomial-like structure, describe the multiplicative parametrizations in the overview of parametrizations shown in Table~\ref{tab:overview}: 
\vspace{0.2cm}
\begin{assumption}[Power-Product Parametrizations $\K$]\label{ass-parametrization-map}
The parametrization map $\K:\Rdxi \to \Rd, \bxi \mapsto \bbeta$, is a $\mathcal{C}^{r}$-smooth surjection, $r \geq1$, with the following properties:
    \begin{enumerate}[a)]
        \item $\K$ is block-separable, i.e., for a partition of $[d]$ into $L\leq d$ groups of size $|\Gj|, j \in [L]$, the corresponding $\bbeta_j$ are parametrized by disjoint subsets $\bxi_j \in \mathbb{R}^{d_{\bm{\xi}_j}}$ of the entries of $\bxi \in \Rdxi$, where $d_{\bxi}=d_{\bm{\xi}_1}+\ldots+d_{\bm{\xi}_L}$. That is, $\K$ is the Cartesian function product $\K(\bxi)=(\K_1(\bxi_1),\ldots, \K_L(\bxi_L))$ of block-wise parametrizations $\K_j(\bxi_j)=\bbeta_j$.
        \item Each $\bxi_j$ can further be grouped into $k$ factors $\bxi_{jl} \in \mathbb{R}^{d_{jl}}, l \in [k]$, so that $\sum_{l=1}^k d_{jl}=d_{\bxi_j}$ and $d_{jl} \in \{1, |\Gj|\}$, i.e., each factor is either a scalar or a vector of the same dimension as $\bbeta_j$. 
        \item$\K_j(\bxi_{j1},\ldots,\bxi_{jk})$ has a power-product structure such that each coordinate $\K_{ji}(\bxi_j)=\beta_{ji}$, $i \in\Gj$, can be written as $\K_{ji}(\bxi_j) = \prod_{l=1}^k \operatorname{sign}(\xi_{jl}^{(i)}) \cdot |\xi_{jl}^{(i)}|^{\alpha_l}$, where $\xi_{jl}^{(i)} \in \bxi_{jl}$, and $\alpha_l \geq 1$ are the (entry-wise) exponents for factor $l \in [k]$. Some parametrizations omit the signs or absolute values, e.g., pure monomials. %
\end{enumerate}
\end{assumption}
%

%
\section[Smooth Sparse Regularization using Hadamard Products]{Smooth $\ell_1$ Regularization using Hadamard Products} \label{sec:l1-hpp-hdp}

In this section, we introduce two smooth surrogate approaches for sparsity-inducing $\ell_1$ regularization and provide some intuition on the underlying geometry.

\subsection{Hadamard Product Parametrization}\label{sec:hpp-vanilla}

We first present a canonical example of our optimization transfer framework based on the so-called Hadamard product parametrization \citep{hoff2017lasso}. This approach enables smooth optimization of $\ell_1$ regularized objectives by applying an overparametrization $\bbeta=\bu\odot\bv$ and imposing $\ell_2$ regularization on the surrogate parameters. %
As the prototype case of our framework, this connection between $\ell_1$ and $\ell_2$ regularization under reparametrization %
will be re-derived in the following for illustrative purposes. %
Assume a non-smooth $\ell_1$ regularized objective $\P$ with $\Rbeta(\bbeta)=2\norm{\bbeta}_1$ and consider the following overparametrized smooth surrogate $\Q$:
\begin{align}
&{\textstyle \mathcal{P}:\mathbb{R}^{d_{\psi}}\times\mathbb{R}^d \to \mathbb{R}_{0}^{+},\,(\bm{\psi},\bm{\beta}) \mapsto \mathcal{L}(\bm{\psi}, \bm{\beta}) + 2 \lambda \|\bm{\beta}\|_{1} = 
\mathcal{L}(\bm{\psi}, \bm{\beta})
+ 2\lambda \sum_{j=1}^{d}\left|\beta_{j}\right| } \,, \label{eq:q-abstract-intext}\\
&{\textstyle \mathcal{Q}:\mathbb{R}^{d_{\psi}}\times\mathbb{R}^d\times\mathbb{R}^d \to \mathbb{R}^{+}_{0},\,(\bm{\psi},\bm{u}, \bm{v}) \mapsto \mathcal{L}(\bm{\psi}, \bm{u}\odot\bm{v}) + \lambda\sum_{j=1}^{d}\big(u_j^2+v_j^2\big) } \label{eq:g-abstract-intext}\,. 
\end{align}
In (\ref{eq:g-abstract-intext}), the HPP map is defined as
\begin{equation} \label{eq:hpp-def}
\mathcal{K}:\mathbb{R}^{d}\times\mathbb{R}^{d}\to\mathbb{R}^{d}, (\bm{u}, \bm{v})\mapsto \bm{u} \odot \bm{v} = \bm{\beta} \,,
\end{equation}
while the surrogate penalty is the plain $\ell_2$ regularizer $\Rxi(\bu,\bv)=\norm{\bu}_2^2+\norm{\bv}_2^2$ with $\bxi=(\bu,\bv)^{\top}$. Our goal is to show that the minimization of (\ref{eq:q-abstract-intext}) and (\ref{eq:g-abstract-intext}) is equivalent according to Definition~\ref{def:equivalence}. In our smooth optimization transfer framework, the main assumption of Theorem~\ref{theorem-general} requires that the HPP $\bbeta=\bu\odot\bv$ and the surrogate regularization $\Rxi(\bu,\bv)=\Vert \bm{u} \Vert_{2}^{2} + \Vert \bm{v} \Vert_{2}^{2}$ together define an SVF for $\Rbeta(\bbeta)=2\norm{\bbeta}_1$ (cf.~Definition~\ref{def:smooth-varform}). The inequality of arithmetic and geometric means (AM-GM) provides a simple but powerful tool for the construction of SVFs using $\ell_2$ regularization as the surrogate penalty and is repeatedly applied throughout the paper. It states that, given a list of $n \in \mathbb{N}$ non-negative numbers $x_i,\,i=1,\ldots,n$, it holds that $\frac{x_1 + \ldots + x_n}{n} \geq \sqrt[n]{x_1 \cdots x_n}$ with equality if and only if $x_1 = \ldots = x_n$.

\noindent In the case of the HPP, it allows us to determine the minimum of the surrogate penalty $\Rxi$ under the constraint $\bm{u}\odot\bm{v}=\bbeta$ for any $\bbeta\in\mathbb{R}^{d}$.
\vspace{-0.0cm}
\begin{lemma} \label{lemma:S_hpp-def} Given the parametrization $\K(\bu,\bv)=\bu\odot\bv$, the minimum of  surrogate $\ell_2$ penalty $\Rxi(\bu,\bv)=\norm{\bu}_2^2+\norm{\bv}_2^2$ subject to $\bu\odot\bv=\bbeta$ constitutes an SVF for $\Rbeta(\bbeta) = 2 \norm{\bbeta}_1$ in (\ref{eq:q-abstract-intext}) and is given by
$\min_{\bm{u},\bm{v}: \bu\odot\bv=\bbeta} \quad \Vert \bm{u} \Vert_{2}^{2} + \Vert \bm{v} \Vert_{2}^{2} = 2\norm{\bbeta}_1\; \,\, \forall \bbeta \in \mathbb{R}^{d}$.
\end{lemma}
\begin{proof}
Because the HPP defines element-wise multiplication, we can minimize $u_j^2 + v_j^2$ such that $u_j v_j = \beta_j$ for some $\beta_j \in \mathbb{R}$ and $j=1,\ldots,d$. Using the AM-GM inequality for $n=2$ and the non-negative numbers $u_j^2$ and $v_j^2$, we obtain
$$
\frac{u_j^2 + v_j^2}{2} \geq \sqrt{u_j^2 v_j^2} = \sqrt{(u_j v_j)^2} = \sqrt{\beta_j^2} =  |\beta_j| \,,
$$
which reduces to equality if and only if $u_j^2 = v_j^2$, yielding a minimum value of $u_j^2 + v_j^2 = 2 |\beta_j|$. Repeating this procedure for all $j=1,\ldots,d$ shows that the constrained minimum of the surrogate penalty is indeed equal to $2 \Vert \bm{\beta} \Vert_1$ for all $\bbeta\in\Rd$.
\end{proof}
\vspace{-0.25cm}

\noindent The optimality conditions $u_j^2=v_j^2=|\beta_j|$ further ensure that we can derive continuous solutions $(\hat{u}_j,\hat{v}_j)$ as functions of $\beta_j=u_j v_j$. Analytically,  $(\hat{u}_j,\hat{v}_j)$ are of the form
{\scriptsize
\begin{equation}
   \underset{(u_j, v_j): u_j v_j = \beta_j}{\argmin} u_j^2+v_j^2 = 
    \begin{cases}
       \left(\sqrt{|\beta_j|},\sqrt{|\beta_j|}\right) \;\text{and}\;\left(-\sqrt{|\beta_j|},-\sqrt{|\beta_j|}\right) & \text{for $\beta_j > 0$} \\
       (0,0) & \text{for $\beta_j=0$} \\
       \left(\sqrt{|\beta_j|},-\sqrt{|\beta_j|}\right) \;\text{and}\; \left(-\sqrt{|\beta_j|},\sqrt{|\beta_j|}\right) & \text{for $\beta_j < 0$} \nonumber \, .
    \end{cases}
\end{equation}}

\noindent Further, we can determine the number of equivalent solutions in the surrogate problem, corresponding to a specific solution in the original problem, using the AM-GM inequality. Due to this duplicity for each $j=1,\ldots,d$, there are a total of $2^s$ equivalent local minimizers $(\hat{\bm{\psi}}, \hat{\bm{u}}, \hat{\bm{v}})$ of $\mathcal{Q}$ for each local minimizer $(\hat{\bm{\psi}}, \hat{\bm{\beta}})$ of $\mathcal{P}$, where $s=||\hat{\bbeta}||_0$. 
Moreover, we can establish the stability of a solution mapping in a more general setting for solutions that are characterized by necessary optimality conditions similar to the above, obtained from applying the AM-GM inequality to the squared surrogate parameters $u_j^2$ and $v_j^2$ for $j\in [d]$.
\begin{lemma}[Lower hemicontinuity of $\hbxi(\bbeta_j)$ under Ass.~\ref{ass-parametrization-map}] \label{lemma:lhc-solution-map}
Let $\hbxi(\bbeta):\Rd \rightrightarrows \Rdxi$ denote the solution mapping $\bbeta \mapsto \argmin_{\bxi \in \K^{-1}(\bbeta)} \mathcal{R}_{\bxi}(\bxi)$, where the parametrization $\K$ satisfies Assumption~\ref{ass-parametrization-map}, %
so that each coordinate $\beta_{ji}$ depends on the $l \in [k]$ factors only via the $i$th entries of the vectors $\bxi_{jl} \in \mathbb{R}^{|\Gj|}$ and the scalars $\xi_{jl}$ in a power-product structure with exponents $\alpha_l\geq1$ for $l \in [k], j \in [L]$. Let the surrogate penalty be
$\mathcal{R}_{\bxi_j}(\bxi_j)=\sum_{l=1}^k
\alpha_l \, \|\bxi_{jl}\|_2^2.$ Then $\hbxi(\bbeta)$ is lower hemicontinuous on $\Rd$.
\end{lemma}
%
Note that for the simple HPP, we have $L=d$ and $k=2$ scalars with exponents $\alpha_l=1$ for each $j \in [d]$. This ensures Assumption~\ref{ass:min} holds in this case, and we can combine Lemma~\ref{lemma:S_hpp-def} with Theorem~\ref{theorem-general} to obtain the final equivalence.
\begin{corollary}\label{cor:hpp}
The optimization of $\P$ in (\ref{eq:q-abstract-intext}) is equivalent to the optimization of the smooth surrogate $\Q$ in (\ref{eq:g-abstract-intext}) by Definition~\ref{def:equivalence}, and solutions to the base problem can be constructed as $(\hbpsi,\hat{\bbeta})=(\hbpsi,\hat{\bu}\odot\hat{\bv})$.
\end{corollary}
Inspired by the implicitly regularized elastic net \citep{zhao2022high}, we propose the following explicitly regularized differentiable variant:
\begin{remark}{(Smooth elastic net formulation via HPP)}
We can readily extend the HPP optimization transfer for $\ell_1$ regularization to the Elastic Net penalty $\mathcal{R}_{\bm{\beta}}(\bm{\beta})\triangleq (1-\alpha)\norm{\bm{\beta}}_1 + \alpha \norm{\bm{\beta}}_2^2$,  $\alpha\in(0,1)$, as introduced in \citet{zou2005regularization}. To do this, we merely redefine $\tilde{\L}(\bpsi,\bbeta) \triangleq \L(\bpsi,\bbeta)+\lambda\alpha \norm{\bbeta}_2^2$ and $\lambda\Tilde{\mathcal{R}}_{\bbeta}(\bbeta) \triangleq \lambda(1-\alpha)\norm{\bbeta}_1$. Applying the HPP, we minimize $$\tilde{\L}(\bpsi, \bm{u} \odot \bm{v})+\frac{\lambda}{2}\Tilde{\mathcal{R}}_{\bxi}(\bm{u},\bm{v}) = \L(\bpsi,\bm{u}\odot\bm{v})+\lambda\alpha \norm{\bm{u}\odot\bm{v}}_2^2 + \frac{\lambda (1-\alpha)}{2}(\norm{\bm{u}}_2^2 + \norm{\bm{v}}_2^2)$$ over $(\bpsi, \bm{u},\bm{v})$ instead of $(\bpsi,\bbeta)$.
Solutions to the Elastic Net-regularized problem can be reconstructed after optimization of the smooth surrogate as $(\hat{\bpsi},\hat{\bbeta})=(\hat{\bpsi},\hat{\bm{u}}\odot\hat{\bm{v}})$.
\end{remark}

\begin{remark}{(Smooth formulations of SCAD/MCP/TL1 via HPP)}
As in the previous remark, we can construct smooth surrogates for objectives with non-convex SCAD \citep{fan2001variable}, MCP \citep{zhang2010nearly}, or transformed $\ell_1$ \citep{zhang2018minimization} regularization via the HPP and a corresponding surrogate regularizer. The surrogate regularizer is constructed by replacing all terms involving $|\beta_j|$ with the left-hand side of the inequality $(u_j^2+v_j^2)/2 \geq |\beta_j|$. A detailed derivation can be found in Appendix~\ref{app:hpp-tl1-scad-mcp}.    
\end{remark}
Further, for general smoothly parametrized objectives $\P(\bpsi,\K(\bu,\bv))$ using the HPP, one requirement for equivalence to $\P(\bpsi,\bbeta)$ by Lemma~\ref{lemma:pk_to_p} is that the parametrization $\K$ is locally open at the local minimizers $(\hat{\bu},\hat{\bv})\in\Rd \times \Rd$ of $\P(\bpsi,\K(\bu,\bv))$. In fact, the Hadamard product of two $d$-dimensional real-valued vectors is a (uniformly) open map everywhere \citep{balcerzak2016certain}, meaning that under $\K$, the image of any open ball around $(\bu,\bv)$ in $\Rd \times \Rd$ contains an open ball around $\bu\odot\bv$ in $\Rd$ for all $(\bu,\bv)\in\Rd \times \Rd$ (cf.~Def.~\ref{def:openness}).
\subsection{The Hadamard Difference Parametrization}\label{sec:hdp-subsec}

An alternative smooth optimization transfer approach for $\ell_1$ regularization is based on the Hadamard difference parametrization (HDP), which is defined as 
\begin{equation}\label{eq:hdp-def}
\mathcal{K}:\mathbb{R}^{d}\times\mathbb{R}^{d}\to\mathbb{R}^{d},\, (\bm{\gamma}, \bm{\delta})\mapsto \bm{\gamma} \odot \bm{\gamma} - \bm{\delta} \odot \bm{\delta} = \bm{\beta} \,.
\end{equation}
This variant of the HPP is often employed in studying the implicit regularization effects of GD in linear neural networks, facilitating an easier theoretical analysis \citep{woodworth2020kernel, vaskevicius2019implicit, vivien2022label}.
In our framework, applying the HDP and imposing explicit surrogate $\ell_2$ regularization on $\bm{\gamma}$ and $\bm{\delta}$ corresponds to $\ell_1$ regularization of $\bbeta$, albeit without the scaling factor of $2$ present in the HPP. To establish a connection between the HDP and the HPP, let $\bm{\gamma} = \frac{\bm{u}+\bm{v}}{2}$ and $\bm{\delta} = \frac{\bm{v}-\bm{u}}{2}$, or equivalently, $\bm{u} = \bm{\gamma} - \bm{\delta}$ and $\bm{v} = \bm{\gamma} + \bm{\delta}$. Then it is easy to confirm that $\bm{\gamma} \odot \bm{\gamma} - \bm{\delta} \odot \bm{\delta} = \bm{u} \odot \bm{v}$. 

\noindent For an objective $\P$ with $\ell_1$ regularization of $\bbeta$, we can construct a smooth surrogate $\Q$ applying the HDP and surrogate $\ell_2$ regularization. Both objectives can be written as
{\small
\begin{align}
&{\textstyle \mathcal{P}(\bm{\psi}, \bm{\beta}) = \mathcal{L}(\bm{\psi}, \bm{\beta}) + \lambda \|\bm{\beta}\|_{1} =  
\mathcal{L}(\bm{\psi}, \bm{\beta})
+ \lambda \sum_{j=1}^{d}\left|\beta_{j}\right| }\,, \label{eq:optim_q_diff} \\
&{\textstyle \mathcal{Q}(\bm{\psi}, \bm{\gamma}, \bm{\delta}) =  \mathcal{L}(\bm{\psi},\bm{\gamma}^2 - \bm{\delta}^2) + \lambda (\Vert \bm{\gamma} \Vert_{2}^{2} + \Vert \bm{\delta} \Vert_{2}^{2}) =  
\mathcal{L}(\bm{\psi},\bm{\gamma}^2 - \bm{\delta}^2)
+\lambda \sum_{j=1}^{d}\big(\gamma_{j}^{2}+\delta_{j}^{2}\big) }.
\label{eq:optim_g_diff} 
\end{align}
}

To show equivalence of the smooth surrogate, we first establish that $\K$ and $\Rxi(\bm{\gamma},\bm{\delta})=\Vert \bm{\gamma} \Vert_{2}^{2} + \Vert \bm{\delta} \Vert_{2}^{2}$ together define an SVF for $\Rbeta(\bbeta)=\norm{\bbeta}_1$.
\vspace{0.1cm}
\begin{lemma} \label{lemma:S_hdp-def} Given the parametrization map $\K(\bm{\gamma},\bm{\delta})=\bm{\gamma}\odot\bm{\gamma}-\bm{\delta}\odot\bm{\delta}$ as defined in Equation (\ref{eq:hdp-def}), the minimum of the surrogate $\ell_2$ regularization $\Rxi(\bm{\gamma},\bm{\delta}))=\norm{\bm{\gamma}}_2^2+\norm{\bm{\delta}}_2^2$ subject to $\bm{\gamma}\odot\bm{\gamma}-\bm{\delta}\odot\bm{\delta}=\bbeta$ constitutes an SVF for $\Rbeta(\bbeta)=\norm{\bbeta}_1$ in (\ref{eq:optim_q_diff}):
 \begin{equation} \label{eq:S_hdp-def}
    \min_{\bm{\gamma}, \bm{\delta}: \bm{\gamma}^2 - \bm{\delta}^2 = \bm{\beta}} \Vert \bm{\gamma} \Vert_{2}^{2} + \Vert \bm{\delta} \Vert_{2}^{2} \; = \Vert \bm{\beta} \Vert_{1} \quad \forall \bbeta \in \Rd\,.
\end{equation}   
\end{lemma}
\vspace{-0.2cm}
For each $\beta_j$ in $\bbeta$, either $\gamma_j^2$ or $\delta_j^2$ must equal zero at the minimum, depending on the sign of $\beta_j$, with the square of the non-zero parameter being equal to $|\beta_j|$. The minimizers $(\hat{\gamma}_j,\hat{\delta}_j)$ hence form a continuous set-valued function of $\beta_j$. 
\vspace{0.1cm}
\begin{corollary}\label{cor:hdp}
Optimization of $\P$ (\ref{eq:optim_q_diff}) is equivalent to optimization of the smooth surrogate $\Q$ (\ref{eq:optim_g_diff}), and solutions to the $\P$ can be constructed as $(\hbpsi,\hat{\bm{\beta}})=(\hbpsi,\hat{\bm{\gamma}}^2 -\hat{\bm{\delta}}^2)$.
\end{corollary}
For the preservation of local minima in general, potentially unregularized objectives $\P(\bpsi,\bbeta)$ parametrized using the HDP, Lemma~\ref{lemma:pk_to_p} again requires local openness of the HDP at local minimizers $(\hat{\gamma},\hat{\delta})$ of $\P(\bpsi,\K(\bm{\gamma},\bm{\delta}))$. Recall that rotating a point $(u,v)\in\mathbb{R}^2$ by $45^{\circ}$ clockwise about the origin defines the transformation $(\gamma,\delta)\triangleq(\frac{u+v}{\sqrt{2}},\frac{v-u}{\sqrt{2}})$. Evaluating the HDP at the rotated point yields $\gamma^2-\delta^2=2uv$, showing that the HDP constitutes a rotation of the HPP scaled by a factor of $2$, with both actions preserving the openness. Details on the difference between HPP and HDP can be found in Appendix~\ref{app:hpp-vs-hdp}.
\begin{figure}[b!] 
    \centering
    \subfloat[HPP]{%
        \includegraphics[height=0.13\textheight]{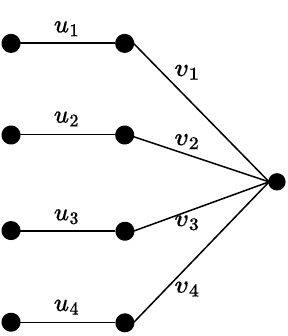}%
        \label{fig:hpp-depth1}%
        }%
    \hspace{0.12\textwidth}
    \subfloat[HDP]{%
        \includegraphics[height=0.13\textheight]{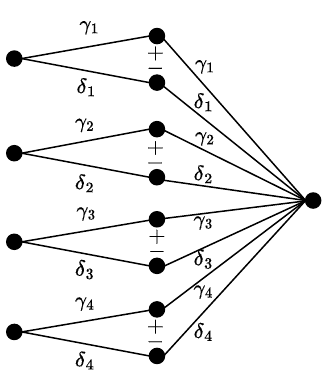}%
        \label{fig:hdp-depth1}%
        }%
    \hspace{0.12\textwidth}
    \subfloat[GHPP~(\ref{sec:group-lasso-vanilla})]{%
        \raisebox{0.15cm}{\includegraphics[height=0.12\textheight]{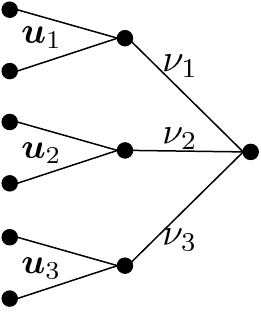}}%
        \label{fig:groupnetwork1}%
        }%
    \caption[Deep Diagonal Linear Network Structures]{\small Diagonal linear networks corresponding to different parametrizations of a linear predictor: \textbf{a)} HPP ($\ell_1$),\, \textbf{b)} HDP ($\ell_1$),\, \textbf{c)} Structure-inducing parametrization (GHPP for $\ell_{2,1}$, cf.~\ref{sec:group-lasso-vanilla}) with grouping layer. Left nodes are inputs, and the right-most node the output.}
    \label{fig:hpp-hdp-dep1}
    \vspace{-0.3cm}
\end{figure}
%
%
%
\subsection{Geometric Intuition and Induced Network Architectures}
\textbf{Correspondence to diagonal linear networks}\,
The HPP $\bm{\beta} = \bm{u} \odot \bm{v}$ and HDP $\bm{\beta} = \bm{\gamma} \odot \bm{\gamma} - \bm{\delta} \odot \bm{\delta}$ parametrizations reveal close connections to diagonal linear networks and linear regression \citep{woodworth2020kernel, tibs2021}. Assume a simple linear model $f(\bm{x}_i | \bm{\beta}) = \bm{x}_{i}^{\top}\bm{\beta}$ with no additional parameters $\bpsi$. This can be represented as a simplistic neural network with $d$ input neurons, a single output neuron, and $d$ edges for the weights $\beta_j$ with linear activations and no additional bias terms. Applying the respective parametrization, e.g., using the HPP, $f(\bm{x}_i | \bu,\bv)=\bm{x}_i^{\top}(\bu \odot \bv)$, which represents a diagonal linear network with linear activations, a single hidden layer, and no bias terms. By Corollary~\ref{cor:hpp}, this is equivalent to a linear regression with $\ell_1$ regularization of $\bbeta$ under $\ell_2$ regularization of the weights. Figure~\ref{fig:hpp-hdp-dep1} shows two such linear networks, with the diagonal network corresponding to the HPP on the left, and the diagonal network corresponding to the HDP in the middle. This correspondence, however, is not limited to overparametrized linear models. For example, we can ``stretch out" a network architecture by inserting additional diagonal layers at certain locations, promoting localized sparse representations. More generally, we can overparametrize any layer of a DNN by replacing its weights $\bbeta$ by $\K(\bxi)$. Imposing suitable surrogate regularization on the weights $\bxi$ then induces sparsity in the original layer in its parametrization. \vspace{0.0cm}
\begin{figure}[b!] 
    \centering
    \subfloat[Optim. Transfer (HPP)]{%
        \raisebox{0.27cm}{\includegraphics[height=0.16\textheight]{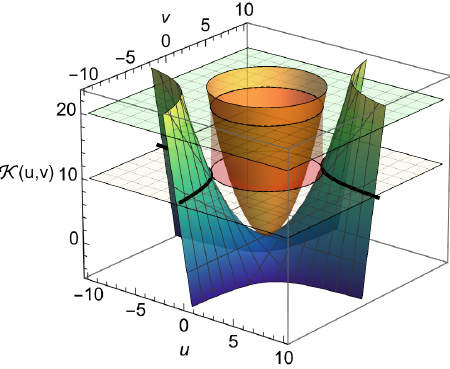}}%
        \label{fig:had-geom-intu-viewpoint}%
        }%
    \hspace{0.01\textwidth}
    \subfloat[\centering Majorization of $\ell_1$ via surrogate $\ell_2$ penalty (HPP)]{%
        \raisebox{0.0cm}{\includegraphics[height=0.2\textheight]{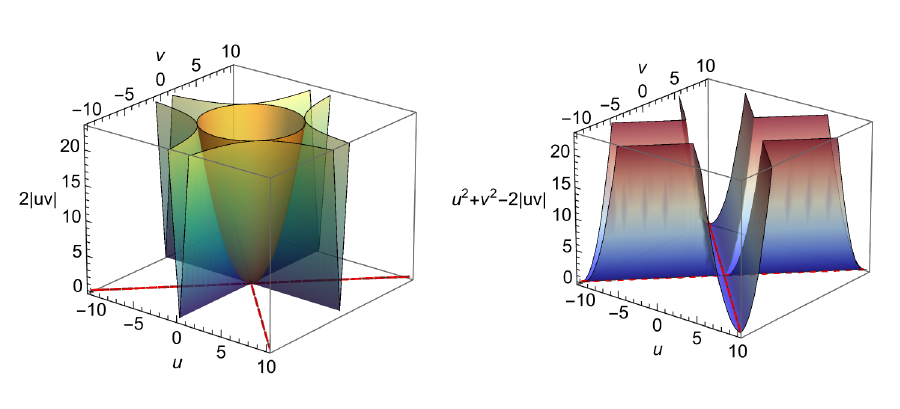}}%
        \label{fig:had-geom-majorant-math}%
        }%
    \caption{\small \textbf{a)}\, Illustration of $\ell_1$ optimization transfer using HPP and surrogate $\ell_2$ regularization on a scalar $\beta_j=10$ (lower plane). The hyperbolic paraboloid (blue/green) shows the parametrization $\K(u_j,v_j)=u_jv_j$ and the elliptic paraboloid (orange) the $\ell_2$ surrogate. 
    The fiber $\K^{-1}(10)$ defines a hyperbola (black), whose two vertices 
    achieve minimal a min. $\ell_2$ penalty of $2|10|=20$ (upper plane) over $\K^{-1}(10)$. \textbf{b)}\, Majorization of 
    overparametrized $\ell_1$ term $2|u_jv_j|$ (blue/green) through $\ell_2$ penalty. 
    The $\ell_2$ 
    (orange) is tightly ``hugged" by the $\ell_1$ term. The difference of both regularizers attains zero at the red perpendicular continuous lines intersecting at the origin, illustrating the l.h.c. of the SVF solution map. The lines are defined by $|u_j|=|v_j|$. It can be seen that for any $(u_j,v_j)$ with $|u_j|=|v_j|$, and any $\beta_j'=u_j' v_j'$ close to $\beta_j=u_jv_j$, there is $(u_j',v_j')$ near $(u_j,v_j)$ with $|u_j'|=|v_j'|$. This is because the solutions continuously increase in $|\beta|$ as we move away from $0$.
    }
    \label{fig:geom-intu-combined}
    \vspace{-0.3cm}
\end{figure}
%

\noindent \textbf{Geometric intuition}\, 
A graphical analysis of our optimization transfer approach for $\ell_1$ regularization using the HPP provides additional insights into the underlying geometry. Figure~\ref{fig:had-geom-intu-viewpoint} illustrates why the minimum of the surrogate $\ell_2$ penalty $\mathcal{R}_{\bm{\xi}}(\bm{u}, \bm{v})=\Vert \bm{u} \Vert_{2}^{2} + \Vert \bm{v} \Vert_{2}^{2}$ over $\{(\bm{u}, \bm{v}):\bm{u} \odot \bm{v} = \bm{\beta}\}$ equals $2\Vert\bm{\beta}\Vert_{1}$. 
The setting in Figure~\ref{fig:had-geom-intu-viewpoint} shows the HPP $\mathcal{K}(\bm{u}, \bm{v})=\bm{u}\odot\bm{v}$ (blue/green), the majorizing surrogate $\ell_2$ penalty (orange), as well as the feasible set defined by the fiber $\K^{-1}(\bbeta)$ (black hyperbola). In this example, we set $d=1$ and fix $\beta=10$. Alternatively, we can interpret the plot as an illustration for only a single entry $\beta_j=10$, $j\in [d]$. The shape of $\K(u_j,v_j)=u_jv_j$ is a hyperbolic paraboloid, and the fiber $\K^{-1}(\beta_j)\subset\mathbb{R}\times \mathbb{R}$ for $\beta_j=10$ is obtained by intersecting $\mathcal{K}$ with the horizontal plane $\beta_j=10$. The geometric shape of the resulting set is a rectangular hyperbola, forming an unbounded feasible set 
in the constrained minimization problem stated in Lemma~\ref{lemma:S_hpp-def}. 
Since the surrogate regularizer $\mathcal{R}_{\bm{\xi}}(\bm{u}, \bm{v})$ defines an elliptic paraboloid for each $j\in[d]$, the constrained minimization problem is solved by searching the entry-wise feasible sets with respect to the smallest surrogate penalty. %
For a hyperbola defined by $u_j(v_j)=\frac{\beta_j}{v_j}, \beta_j>0$, 
this is achieved at the vertices $(\sqrt{\beta_j}, \sqrt{\beta_j})$ and $(-\sqrt{\beta_j}, -\sqrt{\beta_j})$, with a minimal distance of $\sqrt{2\beta_j}$. Similarly, for 
$\beta_j<0$, minimal distance of $\sqrt{2|\beta_j|}$ is attained at $(-\sqrt{|\beta_j|}, \sqrt{|\beta_j|})$ and $(\sqrt{|\beta_j|}, -\sqrt{|\beta_j|})$. For $\beta_j=0$, the fiber $\K^{-1}(0)$ contains all points on the coordinate axes, with 0 minimal distance at $(0,0)$. The majorization property is visualized in Figure~\ref{fig:had-geom-majorant-math}. We further demonstrate how the proposed optimization transfer to an equivalent smooth surrogate transforms the loss landscape using a simple toy objective in Figure~\ref{fig:hpp-tri-1} and a more detailed visualization in Figure~\ref{fig:optim-landscapes-params}. 
%
%
%

\section{Hadamard Group Lasso for Structured Sparsity}\label{sec:had-group-lasso}

In many applications, we have additional \textit{a priori} structural information on the parameters, e.g., that certain gene pathways can only be jointly relevant or that a set of dummy-coded features representing a categorical variable should either be included in the model or fully selected out. To obtain structured sparsity, we make use of parametrization maps ``tying together" groups of parameters through shared factors, with the property that adding smooth $\ell_2$ regularization on the surrogate parameters induces an $\ell_{2,1}$ (group lasso) penalty $2 \sum_{j=1}^{L} \Vert \bm{\beta}_{j} \Vert_{2}$ in the base parametrization $\bbeta$. This is fundamentally different from the structured HPP approach discussed in \citet{hoff2017lasso}, where structure is induced through dependent Gaussian priors on $\bu$ and $\bv$ instead of mixed-norm regularization.
\vspace{-0.0cm}

\noindent \textbf{Set-up for structured sparsity regularization} \,
Let $[d]
$ denote the index set corresponding to the entries of $\bm{\beta}\in\mathbb{R}^{d}$, and define $\mathcal{G}_j \subseteq [d]$ to be the subsets of indices corresponding to groups $j=1,\ldots,L$. Let $\mathcal{G} = \{\mathcal{G}_1, \ldots, \mathcal{G}_{L}\}$ form a partition of $[d]$, i.e. $\cup_{j=1}^{L} \mathcal{G}_j= \{1,\ldots,d\}$ and $\mathcal{G}_i \cap \mathcal{G}_j = \emptyset$ for $i \neq j$, so that $|\mathcal{G}_1|+\ldots+|\mathcal{G}_L|=d$. The parameter vector $\bm{\beta}$ contains the group-wise vectors $\bbeta_j$, i.e., $\bm{\beta} = \left(\bm{\beta}_{1},\ldots, \bm{\beta}_{L}\right)^{\top}$, where $\bbeta_j=(\beta_{ji})_{i \in \Gj}\in \mathbb{R}^{|\mathcal{G}_j|}$ for $j \in [L]$.
\vspace{-0.2cm}
\subsection{Group Hadamard Product Parametrization}\label{sec:group-lasso-vanilla}

For the group Hadamard product parametrization (GHPP), we again use the parametrization structure $\bm{\beta} = \bm{u} \odot \bm{v}$, but now with the elements of $\bm{v}$ (and thus also $\bbeta$) constrained to reflect the group membership. Noting that $\Rd = \mathbb{R}^{|\mathcal{G}_1|+\ldots+|\mathcal{G}_L|}$, the Hadamard factors are
{\small
\begin{align}\label{eq:hadgroup-reparam-defs}
\bm{u} = \left( \bm{u}_{1},\ldots,\bm{u}_{L} \right)^{\top} \in \mathbb{R}^{d},\,\,\bm{v} = \left( \bm{v}_{1},\ldots,\bm{v}_{L} \right)^{\top} =
  \begin{pmatrix} 
    \nu_1 \mathds{1}_{|\mathcal{G}_1|} \\ \vdots \\ \nu_{L} \mathds{1}_{|\mathcal{G}_L|}
  \end{pmatrix}  \in \Rd\,\;. \nonumber
  \end{align}
Then we have $\bbeta_j=\bu_j \odot \bv_j = \nu_j \cdot (u_{j1},\ldots,u_{j |\mathcal{G}_j|})^{\top} \in \mathbb{R}^{|\Gj|}$ for $j \in [L]$.
}
Note that in this parametrization, 
the second Hadamard factor $\bm{v}$ is a $d$-dimensional vector containing values $\nu_{1}, \ldots, \nu_{L}$, where each $\nu_{j}$ is repeated $|\mathcal{G}_j|$ times in $\bv$. Comparing this to the Hadamard factor $\bm{v} = (v_1, \ldots, v_d)^{\top}$ in the HPP, the $d$ distinct entries of $\bm{v}$ are replaced by entries that are constant within groups $\mathcal{G}_{1}, \ldots, \mathcal{G}_L$, thereby ``tying'' together the parameters in each $\Gj$. The first Hadamard factor $\bu$ remains unconstrained as in the HPP, i.e., $\bm{u}=(u_1,\ldots, u_d)^{\top}\in\Rd$.
Letting $\bnu \in \mathbb{R}^{L}$ denote $(\nu_{1}, \ldots, \nu_{L})^{\top}$, the GHPP map is defined as:
{\small
\begin{equation}\label{eq:ghpp-reparam-mapping}
\mathcal{K}: \mathbb{R}^d \times \mathbb{R}^{L} \to \mathbb{R}^{d}, (\bm{u},\bm{\nu}) \mapsto  \begin{pmatrix} 
    \bm{u}_{1} \\ \vdots \\ \bm{u}_{L}
  \end{pmatrix}  \odot \begin{pmatrix} 
    \nu_1 \mathds{1}_{|\mathcal{G}_1|} \\ \vdots \\ \nu_{L} \mathds{1}_{|\mathcal{G}_L|}
  \end{pmatrix}  = \bm{u} \odotg \bm{\nu} = \begin{pmatrix}
  \bm{\beta}_{1} \\ \vdots \\ \bm{\beta}_{L}
\end{pmatrix}  = \bm{\beta} \,,
\end{equation}
}%
\noindent where we use the notation $\bm{u} \odotg \bm{\nu} \triangleq (\bm{u}_{j} \nu_{j})_{j \in \mathcal{G}}$. Given an objective $\P$ with non-smooth regularization $\Rbeta(\bbeta)=2\norm{\bbeta}_{2,1}$, defining the surrogate regularization as $\Rxi(\bu,\bnu)\triangleq\norm{\bu}_2^2+\norm{\bnu}_2^2$ provides a smooth optimization transfer $(\Rbeta,\K,\Rxi)$, from which we construct the smooth surrogate $\Q$:
{\small
\begin{align}
&{\textstyle \mathcal{P}(\bm{\psi}, \bm{\beta}) = \mathcal{L}(\bm{\psi}, \bm{\beta}) + 2 \lambda \Vert \bm{\beta}\Vert_{2,1} =  
\mathcal{L}(\bm{\psi}, \bm{\beta})
+ 2 \lambda \sum_{j=1}^{L} \Vert \bm{\beta}_{j} \Vert_{2} } \, , \label{eq:optim_q_group} \\
&{\textstyle \mathcal{Q}(\bm{\psi}, \bm{u}, \bm{v}) =  \mathcal{L}(\bm{\psi},\bm{u} \odotg \bm{v}) + \lambda (\Vert \bm{u} \Vert_{2}^{2} + \Vert \bm{\nu} \Vert_{2}^{2}) = \mathcal{L}(\bm{\psi},\bm{u} \odotg \bm{v})
+\lambda \sum_{j=1}^{L}\big(\Vert \bm{u}_{j} \Vert_{2}^{2} +\nu_{j}^{2}\big) }\,. \label{eq:optim_g_group}
\end{align}
}

\noindent The functions $\K$ and $\Rxi$ are chosen so that we obtain an SVF for $\Rbeta$:
\begin{lemma} \label{lemma:S_ghpp-def} Given the parametrization map $\K(\bu,\bnu)=\bu \odotg \bnu$, the minimum of the surrogate $\ell_2$ regularization $\Rxi(\bu,\bnu)=\norm{\bu}_2^2+\norm{\bnu}_2^2$ subject to $\K(\bu,\bnu)=\bbeta$ constitutes an SVF for $\Rbeta(\bbeta)$ in (\ref{eq:optim_q_group}) and is
 \begin{equation} \label{eq:S_ghpp-def}
    \min_{\bm{u}_{j}, \nu_j: \bm{\beta}_{j} = \nu_j \bm{u}_{j} } {\textstyle  \sum_{j=1}^{L} \Vert \bm{u}_{j} \Vert_{2}^{2} +\nu_{j}^{2}  = 2 \sum_{j=1}^{L} \Vert \bm{\beta}_{j} \Vert_{2} } \quad \forall \bbeta \in \mathbb{R}^{d}\,,
\end{equation}   
\end{lemma}
According to Lemma~\ref{lemma:lhc-solution-map}, the optimality conditions $\Vert \bm{u}_{j} \Vert_{2}^{2} = \nu_{j}^{2} = \Vert \bm{\beta}_{j} \Vert_{2}$ of the AM-GM inequality in the proof allow us to derive the minimizers $(\hat{\bu}_j,\hat{\nu}_j)$ as a lower hemicontinuous function of $\bbeta_j$ for all $\bbeta_j\in\mathbb{R}^{|\G_j|}$:
\vspace{-0.1cm}
\begin{equation}\label{eq:minimizers-ghpp}
\arg\hspace{-0.05cm}\min_{\hspace{-0.4cm}\substack{(\bu_j,\nu_j):\\ \bm{\beta}_{j} = \nu_j \bm{u}_{j}}}
   \Vert \bm{u}_{j} \Vert_{2}^{2} + \nu_{j}^{2} = 
    \begin{cases}
       \pm \left(\bbeta_j/\sqrt{\Vert \bm{\beta}_{j} \Vert_{2}}, \sqrt{\Vert \bm{\beta}_{j} \Vert_{2}}\right) 
       , & \hspace{-0.25cm}\text{$\norm{\bbeta_j}_2 > 0$} \\
       (\bm{0},0), & \hspace{-0.25cm}\text{$\norm{\bbeta_j}_2=0$} \\
    \end{cases}
\end{equation}
The tuple $(\Rbeta,\K,\Rxi)$ is thus a valid optimization transfer for $\ell_{2,1}$ group sparsity:
\begin{corollary}\label{cor:ghpp}
The optimization of $\P$ in (\ref{eq:optim_q_group}) is equivalent to the optimization of the smooth surrogate $\Q$ in (\ref{eq:optim_g_group}) by Definition~\ref{def:equivalence}, and solutions to the base problem can be obtained as $(\hbpsi,\hat{\bbeta})=(\hbpsi,\hat{\bu}\odotg\hat{\bnu})$. 
\end{corollary}

\noindent Note that there are only two equivalent minimizers for each $\bbeta_j$ given $\norm{\bbeta_j}_2>0$, as the sign of $\hat{\nu}_j$ uniquely determines the sign of all $\hat{u}_{ji}$ in $\hat{\bu}_j$ for $i \in \G_j$. Thus, for each minimizer $(\hat{\bpsi},\hat{\bbeta})$ of $\P$, there will be $2^{s}$ equivalent corresponding solutions $(\hat{\bpsi},\hat{\bu},\hat{\bnu})$ to $\Q$, where $s$ is the number of groups $\G_j$ with $\Vert\hat{\bbeta}_j\Vert_2>0$.

\noindent For linear predictors, structure-inducing overparametrization was also studied in \citet{tibs2021} and \citet{dai2021representation}, however, without proving the matching local minima property or going beyond linearity. %
Similar to the HPP approach to smooth $\ell_1$ regularization, the GHPP corresponds to a particular network structure with linear activations and a grouping layer when applied to a linear model, as shown in Figure~\ref{fig:groupnetwork1}. The $\ell_2$ regularized network then corresponds to a linear model with an $\ell_{2,1}$ penalty.

\noindent Considering the preservation of local minima in a general objective $\P(\bpsi,\bbeta)$ under smooth parametrization of $\bbeta$ using the GHPP, local openness of $\K$ at the local solutions to $\P(\bpsi,\K(\bxi))$ is a crucial requirement for Lemma~\ref{lemma:pk_to_p}. This assumption, however, is not straightforward for the GHPP. While the Banach open mapping theorem states that every continuous linear surjection between Banach spaces is globally open, it is known that this openness principle can not be extended to \textit{bilinear} continuous surjections \citep{horowitz1975elementary, balcerzak2013bilinear}. A widely used counterexample of a bilinear continuous surjection that is not open everywhere is given, e.g., in \citet[][Chapter 2, Exercise 11]{rudin1991functional}, corresponding to the GHPP for $L=1$ and $d=2$. Therefore, as opposed to the HPP, whose global openness is discussed at the end of Section~\ref{sec:hpp-vanilla}, the GHPP is not globally open in general. Whereas local minimality would be preserved under the $\text{HPP}$ for $\textit{any}$ function due to its global openness, this does not hold for the $\text{GHPP}$, depending on which points qualify as local minimizers. The points of openness for the $\text{GHPP}$ are as follows:
\vspace{-0.0cm}

\begin{lemma}[Local openness of the GHPP]\label{lemma:openness-ghpp}
The parametrization map defined by $\K:\Rd \times \mathbb{R}^{L},\,(\bu,\bnu)\mapsto \bu \odotg \bnu$ is locally open at $(\bu,\bnu)$, with $\bu=(\bu_1,\ldots,\bu_L)^{\top}$ and $\bnu=(\nu_1,\ldots,\nu_L)^{\top}$,  if the $(\bu_j,\nu_j)$ are such that $\nu_j=0$ implies $\norm{\bu_j}_2=0$ for all $j\in[L]$.
\end{lemma}
To establish matching local minima, we thus need to ensure that local solutions $(\hat{\bu},\hat{\bnu})$ to $\P(\bpsi,\K(\bu,\bnu))$ are indeed points of openness. Note that all minimizers of $\Q(\bu,\bnu)$ are of the form stated in (\ref{eq:minimizers-ghpp}), i.e., either $(\hat{\bu}_j,\hat{\nu}_j)=(\bm{0},0)$, or the $(\hat{\bu}_j,\hat{\nu}_j)$ are such that $\norm{\hat{\bu}_j}_2>0$ and $|\hat{\nu}_j|>0$ for all $j\in[L]$. Then, by Lemma~\ref{lemma:openness-ghpp}, $\K(\bu,\bnu)$ is locally open at all potential local minimizers of $\Q$. 

%
%
\vspace{-0.2cm}
\subsection{Adjusting the GHPP for Variable Group Sizes}
The well-known group lasso, initially proposed by \citet{yuan2006model}, 
does not employ plain $\ell_{2,1}$ regularization, but includes additional weights accounting for the variable group sizes $|\mathcal{G}_j|,j\in [L]$. With this modification, we can define the non-smooth penalty as $\mathcal{R}_{\bm{\beta}}(\bm{\beta})\triangleq\sum_{j=1}^{L} \sqrt{|\mathcal{G}_j|}\norm{\bm{\beta}_j}_2$. Interestingly, this regularizer can be obtained as a simple extension to the previous approach by introducing a scaling factor in the surrogate penalty. The derivation is deferred to Appendix~\ref{app:adj-ghpp}. This results in the following smooth objective $\Q$ and corresponding equivalent group lasso regularized objective $\P$:
\begin{align}
&{\textstyle \mathcal{P}(\bm{\psi}, \bm{\beta}) = \mathcal{L}(\bm{\psi}, \bm{\beta}) + 2 \lambda  \sum_{j=1}^{L} \sqrt{|\mathcal{G}_j|} \Vert \bm{\beta}_{j} \Vert_{2} }\,, \label{eq:p-adj-ghpp}\\
&{\textstyle \mathcal{Q}(\bm{\psi}, \bm{u}, \bm{\nu}) =  \mathcal{L}(\bm{\psi},\bm{u} \odotg \bm{\nu}) + \lambda \sum_{j=1}^{L}\big(\Vert \bm{u}_{j} \Vert_{2}^{2} + |\mathcal{G}_j| \nu_{j}^{2}\big) }. \label{eq:q-adj-ghpp} 
\end{align}
%
%
%
%
\section{Going Deeper: Non-Convex Regularization with Hadamard Product Parametrizations of Depth \textit{k}} \label{sec:hppk}

The Hadamard product parametrizations factorizing $\bbeta$ using two factors $\bu,\bv$ can be naturally extended to deeper factorizations of depth $k>2,\,k \in \mathbb{N}$. For a suitable surrogate penalty, these parametrizations induce (a restricted class) of non-convex $\ell_q$  and $\ell_{p,q}$ regularizers for $0<q<1$  and $0< q < p\leq 2$ in the base parametrization $\bbeta$. 


\subsection{Hadamard Product Parametrization of Depth \textit{k}} \label{sec:had-prod-chain}

First, consider a multilinear extension of the bilinear HPP termed the $\text{HPP}_k$,
\begin{equation}\label{eq:hppk}
   {\textstyle \mathcal{K}:\prod_{l=1}^{k} \mathbb{R}^{d} \to \mathbb{R}^{d}, (\bm{u}_1,\ldots,\bm{u}_{k}) \mapsto \bigodot_{l=1}^{k} \bm{u}_{l} = \bm{\beta} \,,}
\end{equation}
where $\prod_{l=1}^{k} \mathbb{R}^{d}$ denotes the $k$th Cartesian power of $\mathbb{R}^{d}$ and $k>2$. The depth two case recovers the simple HPP (\ref{eq:hpp-def}). Each $\beta_j,\, j\in [d]$, is parametrized as the product $\prod_{l=1}^{k} u_{jl}$, where each factor $u_{jl}$ is taken from a different $\bu_l$. Further, we define $\Rxi(\bu_1,\ldots,\bu_k)\triangleq\sum_{l=1}^{k}\norm{\bu_l}_2^2$. Then, minimizing $\Rxi(\bu_1,\ldots,\bu_k)$ subject to the constraint imposed by the parametrization map $\K$ yields an SVF for non-convex $\ell_{q}$ regularization with $q=2/k$:
\begin{lemma}\label{lemma:S_hppk-def} Given the parametrization map $\K(\bu_1,\ldots,\bu_k)=\bu_{l}^{\odot k}$, the minimum surrogate $\ell_2$ regularizer $\Rxi(\bu_1,\ldots,\bu_k)=\sum_{l=1}^{k}\norm{\bu_l}_2^2$ subject to $\K(\bu_1,\ldots,\bu_k)=\bbeta$ constitutes an SVF for $\Rbeta(\bbeta)\triangleq k \norm{\bbeta}_{2/k}^{2/k}$ and is given by
$\min_{\bm{u}_{l}: \bbeta=\bm{u}_{l}^{\odot  k}} {\textstyle \sum_{l=1}^{k}} \Vert \bm{u}_l \Vert_{2}^{2} = k \Vert \bm{\beta} \Vert_{2/k}^{2/k} \,\, \forall \bbeta \in \mathbb{R}^{d}$.
\end{lemma} 
A visualization of the $\text{HPP}_k$ for $k=3$ can be found in Appendix~\ref{app:geom-intuition-3d}, illustrating the shape of the fibers of $\K$ and the majorization of the non-smooth $\ell_{2/3}$ penalty by the smooth surrogate $\ell_2$ penalty. 
Given an objective $\P$ with smooth $\L$ and non-convex $\ell_{2/k}$ regularization $\Rbeta(\bbeta)$, applying the optimization transfer defined by $(\Rbeta,\K,\Rxi)$ yields the corresponding $\Q$:
\begin{align}
    &{\textstyle \mathcal{P}(\bm{\psi}, \bm{\beta}) = \mathcal{L}(\bm{\psi}, \bm{\beta}) + \lambda k \|\bm{\beta}\|_{2/k}^{2/k} =  
    \mathcal{L}(\bm{\psi}, \bm{\beta})
    + \lambda k \sum_{j=1}^{d}\left|\beta_{j}\right|^{2/k} }\,, \label{eq:optim_q_HPC} \\
%
&\hspace{-0.2cm}{\textstyle \mathcal{Q}(\bm{\psi}, \bm{u}_1,\ldots,\bm{u}_{k}) =  \mathcal{L}\big(\bm{\psi}, \bm{u}_{l}^{\odot  k}\big) + \lambda \sum_{l=1}^{k} \Vert \bm{u}_{l} \Vert_{2}^{2} = 
\mathcal{L}\big(\bm{\psi}, \bm{u}_{l}^{\odot  k}\big)
+ \lambda \sum_{j=1}^{d} \sum_{l=1}^{k} u_{jl}^{2} }\,. \label{eq:optim_g_HPC}
\end{align}
%
The optimality conditions of the AM-GM inequality ensure lower hemicontinuity of the solution map in Lemma~\ref{lemma:S_hppk-def} by Lemma~\ref{lemma:lhc-solution-map}, implying equivalence of $\P$ and $\Q$: 
\begin{corollary}\label{cor:hppk}
The optimization of $\P$ (\ref{eq:optim_q_HPC}) is equivalent to optimization of the smooth surrogate $\Q$ (\ref{eq:optim_g_HPC}) by Def.~\ref{def:equivalence}, and solutions to $\P$ can be constructed as $(\hbpsi,\hat{\bm{\beta}}) = (\hbpsi,\hat{\bm{u}}_{l}^{\odot  k})$. 
\end{corollary}
This result is also shown in \citet{hoff2017lasso}, but its application there is limited to simple linear models that can be optimized with alternating ridge regression.
Extending the HPP, the $\text{HPP}_k$ also corresponds to a horizontally ``stretched" diagonal network structure with increased depth, as shown in Figure~\ref{fig:diagnetwork}. The relation of parametrization and corresponding network structure for linear models was also studied in simpler settings and without proof of our general result \citep{tibs2021,dai2021representation}. Besides these works in explicit regularization, a strand of literature in DL uses diagonal linear networks to study the implicit regularization of GD \citep{gunasekar2018implicit,gissin2019implicit,woodworth2020kernel,moroshko2020implicit, li2021implicit}.
\begin{figure}[t!] 
    \centering
    \subfloat[$\text{HPP}_k$]{%
        \includegraphics[height=0.108\textheight]{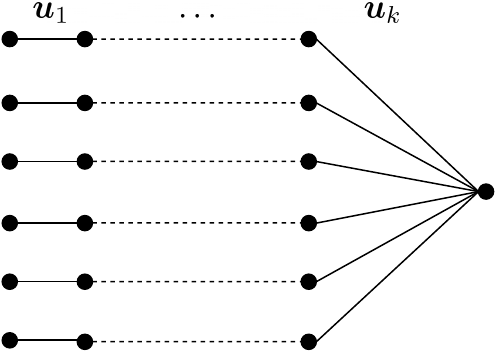}%
        \label{fig:diagnetwork}%
        }%
    \hspace{0.07\textwidth}
    \subfloat[$\text{GHPP}_k$]{%
        \includegraphics[height=0.108\textheight]{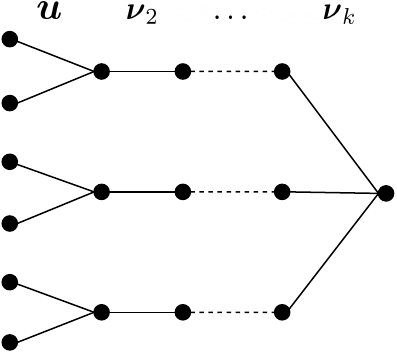}%
        \label{fig:deepgroup}%
        }%
    \hspace{0.07\textwidth}
    \subfloat[$\text{GHPP}_{k_1,k_1+k_2}$]{%
        \raisebox{0\textwidth}{\includegraphics[height=0.108\textheight]{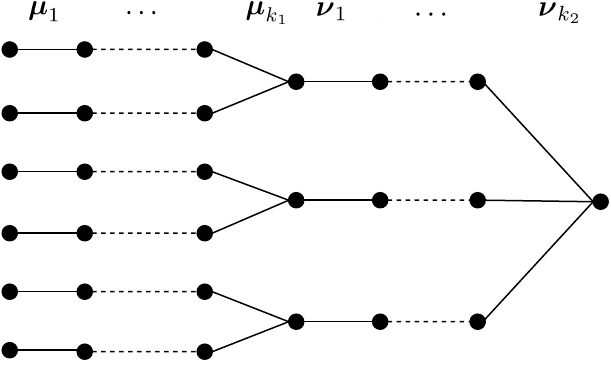}}%
        \label{fig:grouplpq}%
        }%
    \caption[Deep Diagonal Linear Network Structures]{\small Deep diagonal linear networks corresponding to parametrizations of a linear predictor. \textbf{a)} HPP (for $\ell_{2/k}$),\, \textbf{b)} $\text{GHPP}_k$ (for $\ell_{2,2/k}$),\, \textbf{c)} $\text{GHPP}_{k_1,k_1+k_2}$ (for $\ell_{2/k_1,2/(k_1+k_2)})$. The depth up to and including the grouping layer is $k_1$, followed by $k_2=k-k_1$ more diagonal layers. Nodes on the left represent input features and the single node on the right the output.}
    \label{fig:deeper-architectures-combined}
    \vspace{-0.35cm}
\end{figure}

\noindent Regarding applications of the $\text{HPP}_k$ to general objectives $\P(\bpsi,\bbeta)$ without surrogate regularization, we can establish the global openness of the $k$-linear surjection $\K$: 
\vspace{-0.0cm}
\begin{lemma}\label{lemma:proof-hppk}
The map $\mathcal{K}:\prod_{l=1}^{k} \mathbb{R}^{d} \to \mathbb{R}^{d}, (\bm{u}_1,\ldots,\bm{u}_{k}) \mapsto \bigodot_{l=1}^{k} \bm{u}_{l}$ is globally open.
\end{lemma}
\vspace{-0.0cm}
Consequently, applying Lemma~\ref{lemma:pk_to_p}, smoothly parametrizing any continuous objective using the $\text{HPP}_k$ 
preserves the local minima of $\P$.
%
%
%
%
\subsection{Group Hadamard Product Parametrizations of Depth \textit{k}}\label{sec:had-group-chains}

The smooth optimization transfer for $\ell_{2,1}$ regularized problems can be naturally extended to structured sparsity with non-convex $\ell_{2,2/k}$ regularization. We start with the same set-up as in Section~\ref{sec:had-group-lasso}, but now consider deeper factorizations of $\bm{\beta}$. Recall that the GHPP is defined as $\K(\bu,\bnu)=\bu \odotg \bnu=\bbeta$. 
Further factorizing the grouping parameter $\bnu$ into $k-1$ Hadamard factors, i.e., $\bnu=\bigodot_{r=1}^{k-1}\bnu_{r}$, defines the $\text{GHPP}_k$ map:
\begin{align} \label{eq:ghpp-k}
 \hspace{-0.2cm}\mathcal{K}:\,\mathbb{R}^{d} \times \prod_{r=1}^{k-1} \mathbb{R}^{L} \to \mathbb{R}^{d}, (\bm{u}, \bm{\nu}_1,\ldots,\bm{\nu}_{k-1}) \mapsto  \bm{u} \odotg  \bnu_{r}^{\odot(k-1)}= 
   {\footnotesize \begin{pmatrix} 
    \bm{u}_{1} \\ \vdots \\ \bm{u}_{L}
  \end{pmatrix}  \odot \begin{pmatrix} 
    \mathds{1}_{|\mathcal{G}_1|} \prod_{r=1}^{k-1} \nu_{1r} \\ \vdots \\ \mathds{1}_{|\mathcal{G}_L|} \prod_{r=1}^{k-1} \nu_{Lr} 
  \end{pmatrix}} \nonumber
\end{align}
Equivalently, the parametrization on the group level reads $\bbeta_j=\bu_j \prod_{r=1}^{k-1}\nu_{jr}$, where $\bbeta_j,\bu_j\in\mathbb{R}^{|\G_j|}$ and $\nu_{jr}\in\mathbb{R}$, for $j=1,\ldots,L$ and $r=1,\ldots,k-1$. Applying plain $\ell_2$ regularization under this parametrization, i.e., $\Rxi(\bu, \bnu_1, \ldots, \bnu_{k-1})\triangleq\norm{\bu}_2^2+\sum_{r=1}^{k-1}\norm{\bnu_r}_2^2$, induces the non-smooth and non-convex regularizer $\Rbeta(\bbeta)=k\norm{\bbeta}_{2,2/k}^{2/k}$ for structured sparsity in the base parametrization. To show this, we first prove that the minimum $\ell_2$ penalty under the parametrization map constraint equals $\Rbeta$:
\begin{lemma} \label{lemma:S_ghppk-def} Given the parametrization $\K(\bu,\bnu_1,\ldots,\bnu_{k-1})=\bu \odotg \bnu_{r}^{\odot (k-1)}$, the minimum of the surrogate $\ell_2$ penalty $\Rxi(\bu, \bnu_1, \ldots, \bnu_{k-1})\triangleq\norm{\bu}_2^2+\sum_{r=1}^{k-1}\norm{\bnu_r}_2^2$ subject to $\K(\bu,\bnu_1\ldots,\bnu_{k-1})=\bbeta$ constitutes the following SVF for $\Rbeta(\bbeta)\triangleq k \norm{\bbeta}_{2,2/k}^{2/k}$:
 \begin{equation} \label{eq:S_ghppk-def}
\min_{\substack{\bm{u}, \bm{\nu}_{1}, \ldots, \bm{\nu}_{k-1}:\\ \bm{\beta} = \bm{u}\odotg \bnu_{r}^{\odot (k-1)}}}
 \textstyle\sum_{j=1}^{L} \Big( \Vert \bm{u}_{{j}} \Vert_{2}^{2} + \textstyle\sum_{r=1}^{k-1}  \nu_{jr}^{2} \Big)  =  k \Vert \bm{\beta} \Vert_{2, 2/k}^{2/k}\quad\forall\bbeta\in\Rd \,.
\end{equation}   
\end{lemma}
For an objective $\P(\bpsi,\bbeta)$ with non-convex $\ell_{2,2/k}$ regularization, 
\begin{equation}
{\textstyle \mathcal{P}(\bm{\psi}, \bm{\beta}) = \mathcal{L}(\bm{\psi}, \bm{\beta}) + \lambda k \|\bm{\beta}\|_{2,2/k}^{2/k} = \L(\bpsi,\bbeta)+\lambda k \sum_{j=1}^{L}\norm{\bbeta_j}_2^{2/k} }\,, \label{eq:optim_q_HGPC}
\end{equation}
the smooth surrogate $\Q$ obtained from the tuple $(\Rbeta,\K,\Rxi)$ is given by
\begin{equation}
{\textstyle \mathcal{Q}(\bm{\psi}, \bm{u}, \bm{\nu}_{1},\ldots,\bm{\nu}_{k-1}) =  \mathcal{L}\big(\bm{\psi}, \bm{u} \odotg \bm{\nu}_{r}^{\odot (k-1)} \big) + \lambda \sum_{j=1}^{L} \Big(\Vert \bm{u}_{{j}} \Vert_{2}^{2} + \sum_{r=1}^{k-1}  \nu_{jr}^{2}\Big) }  \,. \label{eq:optim_g_HGPC} 
\end{equation}
By Lemma~\ref{lemma:lhc-solution-map}, the optimality conditions obtained in the proof of Lemma~\ref{lemma:S_ghppk-def} imply a lower hemicontinuous solution map as a function of $\bbeta$, so that we can state:
\begin{corollary}\label{cor:ghppk}
The optimization of $\P$ in (\ref{eq:optim_q_HGPC}) is equivalent to the optimization of the smooth surrogate $\Q$ in (\ref{eq:optim_g_HGPC}) by Definition~\ref{def:equivalence}, and solutions to $\P$ can be constructed as $(\hbpsi,\hat{\bm{\beta}}) = (\hbpsi,\hat{\bm{u}} \odotg  \hat{\bnu_{r}}^{\odot (k-1)})$. 
\end{corollary}
We can think of the parametrization $\K$ as a composition involving the GHPP and the $\text{HPP}_{k-1}$ for $\bnu$ to gain insights into the network architecture corresponding to a linear model overparametrized by $\K$. Compared to the depth-two network matching the GHPP in Figure~\ref{fig:groupnetwork1}, the network for the $\text{GHPP}_k$ shown in Figure~\ref{fig:deepgroup} adds $k-1$ diagonal layers after the initial layer, corresponding to the additional deeper factorization of $\bnu$ in the $\text{GHPP}_k$. In the previously mentioned less general setting, \citet{tibs2021} first discovered that optimizing a network as in Figure~\ref{fig:groupnetwork1} with weight decay induces an objective with the same global minimum as an $\ell_{2,2/k}$ regularized linear model.\\
\noindent Regarding the preservation of local minima when applying the $\text{GHPP}_k$ to a general objective $\P(\bpsi,\bbeta)$ without surrogate regularization, we can use the compositional nature of the $\text{GHPP}_k$ to obtain points of local openness, as required by Lemma~\ref{lemma:pk_to_p}:
\begin{corollary}[Points of local openness of the $\text{GHPP}_k$]\label{cor:openness-ghppk}
The parametrization mapping $\K(\bu,\bnu_1,\ldots,\bnu_{k-1}) = \bu \odotg \bnu_r^{\odot (k-1)}$ 
is locally open at $(\bm{u}, \bm{\nu}_1,\ldots,\bm{\nu}_{k-1})$ whenever the GHPP (\ref{eq:ghpp-reparam-mapping}) is locally open at $(\bu,\bnu_r^{\odot (k-1)})$.
\end{corollary}
Besides the proof in Appendix~\ref{app:cor-openness-ghppk}, conditions for the local openness of the GHPP are stated in Lemma~\ref{lemma:openness-ghpp}.
Note that the optimality conditions in the proof of Lemma~\ref{lemma:S_ghppk-def} thus also imply the local openness of $\K$ at all local minimizers of $\Q$.

%
%
%
%
\subsection[Generalizing the GHPP to Mixed Quasi-Norms]{Generalizing the GHPP to Mixed $\ell_{p,q}$ Quasi-Norms} \label{sec:had-lpq-group-chains}

We can extend the principle behind the construction of the $\text{GHPP}_k$, i.e., starting with the $\text{GHPP}$ and factorizing the $\bnu$ parameter, to deeper parametrizations factorizing both $\bu$ and $\bnu$ simultaneously into $k_1$ and $k_2$ Hadamard factors. In the following, we establish that smooth $\ell_2$ regularization of the resulting surrogate parameters induces non-convex $\ell_{p,q}$ mixed-norm regularization in the base parametrization, with $(p,q)\in\{(2/k_1,2/(k_1+k_2)):k_1,k_2 \in \mathbb{N}\}$. We start with the same structured parameter set-up as in Section~\ref{sec:had-group-lasso}, partitioning the components of $\bbeta$ into $L$ groups. Consider the GHPP map given by $\bm{\beta} = \bm{u} \odotg \bnu$, with $\bm{u}=\left( \bm{u}_{1},\ldots,\bm{u}_{L} \right)^{\top}$ and $\bnu=\left( \nu_{1},\ldots,\nu_{L} \right)^{\top}$, together comprising $L$ pairs of group-wise parameters $(\bm{u}_{{j}},\nu_{{j}})$. Factorizing each $\bm{u}_{{j}}$ into a product of $k_1$ Hadamard factors $\bm{\mu}_{jt}$, $t=1,\ldots,k_1$, and each $\nu_j$ into a product of $k_2$ scalar factors $\nu_{jr}$, $r=1,\ldots,k_2$, we can define the following surjective parametrization mapping $\mathcal{K}$ termed the $\text{GHPP}_{k_1,k_1+k_2}$:
{\small
\begin{align}\label{eq:ghppk1k-def}
\mathcal{K}\,:\, \textstyle\prod_{t=1}^{k_1} \mathbb{R}^{d} \times \textstyle\prod_{r=1}^{k_2} \mathbb{R}^{L} \to &\mathbb{R}^{d},\; (\bm{\mu}_{1}, \ldots, \bm{\mu}_{k_1}, \bm{\nu}_1,\ldots,\bm{\nu}_{k_2}) \mapsto  \bm{\mu}_{t}^{\odot  k_1} \odotg \bnu_{r}^{\odot  k_2}  \\
&=\begin{pmatrix} 
\bm{\mu}_{1t}^{\odot  k_1} \\ \vdots \\ \bm{\mu}_{Lt}^{\odot  k_1}
\end{pmatrix}  \odot \begin{pmatrix} 
\mathds{1}_{|\mathcal{G}_1|} \prod_{r=1}^{k_2} \nu_{1r} \\ \vdots \\ \mathds{1}_{|\mathcal{G}_L|} \prod_{r=1}^{k_2} \nu_{Lr}
\end{pmatrix} = \begin{pmatrix}
\bm{\beta}_{1} \\ \vdots \\ \bm{\beta}_{L}
\end{pmatrix} = \bm{\beta} \,, \nonumber
\end{align}
}%
\noindent where $\bm{\mu}_{t} \triangleq (\bm{\mu}_{1t}, \ldots, \bm{\mu}_{Lt})^{\top} \in \mathbb{R}^{d}$ and $\bm{\nu}_{r} \triangleq (\nu_{1r},\ldots, \nu_{Lr})^{\top} \in \mathbb{R}^{L}$. Note that each $\bm{\mu}_{jt}$ is the $t-$th factor of the $j$-th parameter group with entries $(\mu_{j1t}, \ldots, \mu_{j |\mathcal{G}_j| t})^{\top}\in\mathbb{R}^{|\G_j|}$. On the group level, the parametrization reads $\bbeta_j=\bu_j \nu_j = (\bigodot_{t=1}^{k_1} \bm{\mu}_{jt}) \prod_{r=1}^{k_2}\nu_{jr}=\bm{\mu}_{jt}^{\odot k_1}\prod_{r=1}^{k_2}\nu_{jr}$, for $j\in[L]$. Further, let $k \triangleq k_1+k_2$ denote the total factorization depth. To derive the non-convex group-sparse regularizer for $\bbeta$ induced through $\ell_2$ regularization of $\bm{\mu}_{jt},\nu_{jr}$ for  $j\in[L],t\in[k_1],r\in[k_2]$, a simple generalization of the AM-GM inequality is required.
Defining the surrogate penalty $\Rxi$ as plain $\ell_2$ regularization, we can show that $\Rxi$ and $\K$ induce an SVF for mixed-norm $\ell_{p,q}$ regularization.
\begin{lemma} \label{lemma:S_ghppk1k-def} Given a parametrization $\K(\bm{\mu}_{1},\ldots,\bm{\mu}_{k_1},\bnu_1,\ldots,\bnu_{k_2})=\bm{\mu}_{t}^{\odot  k_1} \odotg \bnu_{r}^{\odot  k_2}$, the minimum surrogate $\ell_2$ regularization $\Rxi(\bm{\mu}_{1}, \ldots, \bm{\mu}_{k_1} \bnu_1, \ldots, \bnu_{k_2} ) \triangleq \sum_{t=1}^{k_1}\norm{\bm{\mu}_{t}}_2^2+\sum_{r=1}^{k_2}\norm{\bnu_r}_2^2$ subject to $\K(\bm{\mu}_{1},\ldots,\bm{\mu}_{k_1},\bnu_1,\ldots,\bnu_{k_2})=\bbeta$ constitutes an SVF for $\Rbeta(\bbeta)\triangleq k \norm{\bbeta}_{2/k_1,2/k}^{2/k}$ and is given by
\vspace{-0.1cm}
\begin{equation} \label{eq:S_ghppk1k-def}
\min_{\substack{\bm{\mu}_{1}, \ldots, \bm{\mu}_{k_1}, \bnu_1, \ldots, \bnu_{k_2}:\\ 
\bbeta = \bm{\mu}_{t}^{\odot  k_1} \odotg  \bnu_{r}^{\odot  k_2}}} \,\, {\textstyle \sum_{j=1}^{L} \Big(\sum_{t=1}^{k_1} \Vert \bm{\mu}_{jt} \Vert_{2}^{2} + \sum_{r=1}^{k_2}  \nu_{jr}^{2} \Big)}  =  k \Vert \bm{\beta} \Vert_{2/k_1, 2/k}^{2/k} \quad\forall\bbeta\in\Rd \,.
\vspace{-0.15cm}
\end{equation}
\end{lemma}
Note that by Lemma~\ref{lemma:lhc-solution-map}, the optimality conditions in the proof above ensure lower hemicontinuity of the solution map to the SVF. Assuming an objective $\P(\bpsi,\bbeta)$ with non-convex $\ell_{2/{k_1},2/k}$ regularizer $\Rbeta(\bbeta)$, the optimization transfer $(\Rbeta,\K,\Rxi)$ defines the following equivalent smooth surrogate $\Q$: 
{\small
\begin{align}
&{\small \mathcal{P}(\bm{\psi}, \bm{\beta}) = \mathcal{L}(\bm{\psi}, \bm{\beta}) + \lambda k \|\bm{\beta}\|_{2/k_1,2/k}^{2/k} = \L(\bpsi,\bbeta) + \lambda k \sum_{j=1}^{L}\norm{\bbeta_j}_{2/k_1}^{2/k} } \,, \label{eq:optim_q_GHGPC}\\
&{\small \mathcal{Q}(\bm{\psi}, \bm{\mu}_{1},\ldots,\bm{\mu}_{k_1},\bnu_1,\ldots,\bnu_{k_2}) =  \mathcal{L}\big(\bm{\psi}, \bm{\mu}_{t}^{\odot  k_1} \odotg \bnu_{r}^{\odot  k_2} \big) + \lambda {\footnotesize \sum_{j=1}^{L} \Big( \sum_{t=1}^{k_1} \Vert \bm{\mu}_{jt} \Vert_{2}^{2} + \sum_{r=1}^{k_2}  \nu_{jr}^{2} \Big)}  \,.}\label{eq:optim_g_GHGPC} 
\end{align}
}

\begin{corollary}\label{cor:ghppk1k}
The objective $\P$ in (\ref{eq:optim_q_GHGPC}) is equivalent to the smooth surrogate $\Q$ in (\ref{eq:optim_g_GHGPC}) by Definition~\ref{def:equivalence}, and solutions to $\mathcal{P}$ can be constructed as $(\hbpsi,\hat{\bm{\beta}}) = (\hbpsi,\hat{\bm{\mu}}_{t}^{\odot  k_1} \odotg  \hat{\bnu}_{r}^{\odot  k_2}) = \big(\hbpsi,\big( \bigodot_{t=1}^{k_1} \hat{\bm{\mu}}_{t} \big) \odotg \big(\bigodot_{r=1}^{k_2} \hat{\bnu}_{r}\big)\big)$.
\end{corollary}
Figure~\ref{fig:grouplpq} shows an exemplary network architecture corresponding to the $\text{GHPP}_{k_1,k_1+k_2}$ applied to an LM \citep{dai2021representation}. The architecture also provides an intuitive visualization of mixed-norm regularization for structured sparsity as a whole. While the depth of the first block of diagonal layers, factorizing $\bu$ into $k_1$ Hadamard factors $\bm{\mu}_t$, determines the induced \textit{within-group} norm, the depth of the group-wise constant parameters in $\bnu$ into $k_2$ Hadamard factors determines the induced \textit{between-group} norm.
\vspace{-0.3cm}
%
%
\subsection{Parametrizations with Parameter Sharing}\label{sec:parameter-sharing}
Parameter or weight sharing enables interesting modifications of the previously presented parametrizations, as the parameter redundancy caused by overparametrization 
can be greatly reduced by allowing for shared parameters between the Hadamard factors. Parameter sharing can be defined as identifying two or more parameters of an objective function as a single parameter, i.e., interpreting them as identical. For example, the group structure-inducing GHPP, $\K(\bu,\bnu) = \bu \odotg \bnu$, is essentially the HPP $\K(\bu,\bv)=\bu \odot \bv$, but with shared parameters $\bv_j = \nu_j \mathds{1}_{|\G_j|}$ within groups $j \in [L]$, collapsed into the scalar $\nu_j$. 
Despite requiring many fewer additional parameters, these parametrizations still define a valid SVF $\Rbeta$ like their fully overparametrized counterparts.
\vspace{0.15cm}

\noindent \textbf{Deep HPP with shared parameters}\,
Consider the parametrization map for the $\text{HPP}_k$, defined as $\K(\bu_1,\ldots,\bu_k)=\bigodot_{l=1}^{k}\bu_l$. By introducing parameter sharing between $(k-1)$ Hadamard factors, i.e., replacing the Hadamard product of $k-1$ separate factors with a self-Hadamard product, we retain enough freedom to ensure surjectivity of the parametrization. 
We use $\bu\in\Rd$ to denote the first Hadamard factor, and $\bv^{k-1}\in\Rd$ for the other factors that are collapsed into a single shared vector $\bv \in \Rd$. The following defines the $\text{HPP}_{k}^{shared}$ 
\begin{equation}\label{eq:hppk-shared}
   \hspace{-0.1cm}\mathcal{K}:\mathbb{R}^{d} \times \mathbb{R}^{d} \to \mathbb{R}^{d}, (\bm{u}, \bm{v}) \mapsto \bm{u} \odot \left( \bm{v} \odot \cdots \odot \bm{v} \right) = \bm{u} \odot {\textstyle (\bigodot_{l=1}^{k-1}} \bm{v})  = \bm{u} \odot \bm{v}^{k-1} = \bm{\beta} \,.
\end{equation}
The suitable surrogate penalty $\Rxi$ to obtain an SVF is a re-weighted $\ell_2$ penalty accounting for the increased contribution of the shared parameter to the parametrization. 
More precisely, the shared parameter $\bv$ is counted $(k-1)$ times, thereby ensuring the appropriate re-weighting for $\Rxi$ to define an SVF for non-convex $\ell_{2/k}$ regularization:
\begin{lemma} \label{lemma:S_hppk-shared-def} Given the parametrization $\K(\bu,\bv)=\bm{u} \odot \bm{v}^{k-1}$, the minimum surrogate $\ell_2$ penalty $\Rxi(\bu,\bv) \triangleq \norm{\bu}_2^2+(k-1)\norm{\bv}_2^2$ subject to $\K(\bu,\bv)=\bbeta$ constitutes an SVF for $\Rbeta(\bbeta)\triangleq k \norm{\bbeta}_{2/k}^{2/k}$, i.e.,
  $\min_{\bm{u}, \bm{v}: \bm{u} \odot \bm{v}^{k-1} = \bm{\beta}}  \Vert \bm{u} \Vert_{2}^{2} + (k-1) \Vert \bm{v}  \Vert_{2}^{2} = k \Vert \bm{\beta} \Vert_{2/k}^{2/k}  \,\,\forall\bbeta\in\Rd\,$. 
\end{lemma}
%
%
However, despite constituting a valid SVF with less overparametrization, parameter sharing breaks the balance and symmetry in the parametrization, with unclear consequences for the optimization. 
Yet, we can relate the GD optimization dynamics for the $\text{HPP}_{k}^{shared}$ to its fully overparametrized counterpart $\text{HPP}_k$ under identical initialization of the to-be-shared parameters. Using a rescaled learning rate for the shared factors, we derive identical updates for both variants, as detailed in Appendix~\ref{app:ident-init}.\\
\noindent Moreover, initializing \textit{all} $k$ Hadamard factors of the $\text{HPP}_k$ identically prohibits them from changing their sign over the iterations for sufficiently small step sizes, since the gradient updates vanish as the reconstructed coefficients $\beta_j$ approach zero. This is because under identical initialization, the gradient of the entry-wise parametrization $\K_j((u_{jl})_{l=1}^k)=\prod_{l=1}^k u_{jl}$ is given by a vector of identical entries $\prod_{l' \neq l} u_{jl'}$ for $l \in [k]$. Hence, since all $u_{jl}^0$ are identical, they also receive identical updates and thus stay identical, $u_{jl}^{t+1}=u_{jl'}^{t+1} \, \forall l \in [k]$. In this case, the only way for the product to change signs is by passing through the origin, which is prohibited by the vanishing products in the parametrization gradient. This can be exploited to solve non-negative least squares \citep{gissin2019implicit, chou2022non}. %
\vspace{0.2cm}
%
%

\noindent \textbf{HDP of depth \textit{k} without and with shared weights}\, 
Similar to how the HPP can be generalized to the deeper parametrization $\text{HPP}_k$, the HDP from \ref{sec:hdp-subsec} can be generalized to deeper variants inducing  $\ell_{2/k}$ regularization in the base parametrization under $\ell_2$ regularization of the surrogate parameters. \citet{chou2023more} mention this fully-overparametrized generalization of the HDP, here named $\text{HDP}_k$:
In their analysis of gradient dynamics they restrict themselves to the case of identical initialization, effectively giving rise to the following parametrization termed the $\text{HDP}_{k}^{shared}$, incorporating parameter sharing between the $\bm{u}_l$ and the $\bm{v}_l$ for $l \in [k]$, respectively: 
%
%
$\mathcal{K}: \mathbb{R}^{d} \times \mathbb{R}^{d} \to \mathbb{R}^{d}, (\bm{u},\bm{v}) \mapsto 
\bm{u}^{k} - \bm{v}^{k} = \bm{\beta}$. 
In DL, these parametrizations are widely applied in the implicit regularization literature to obtain simple-to-analyze depth-$k$ networks that exhibit rich optimization and implicit regularization dynamics \citep[see, e.g.,][]{woodworth2020kernel, li2021implicit}. 
%
\section[Hadamard Powers: 
Non-Integer Factorization Depths for Unrestricted Lq and Lpq Regularization]{Hadamard Powers:  Non-Integer Factorization Depths for Unrestricted $\ell_q$ and $\ell_{p,q}$ Regularization}\label{sec:hpowp}

The parametrizations based on (group) Hadamard products can induce $\ell_q$ and $\ell_{p,q}$ regularization under surrogate $\ell_2$ regularization for the restricted class $q\in\{2/k|k\in\mathbb{N}\}$ and $(p,q)\in\{(2/k_1,2/(k_1+k_2)|k_1,k_2\in\mathbb{N}\}$. Extending Hadamard product-based parametrizations to Hadamard powers
permits a more flexible choice of the induced regularizer, allowing selection of the previously restricted $p$ and $q$ arbitrarily from $q \in (0,1]$ and $0<q < p \leq 2$. Thus, smooth optimization for non-convex sparse regularization can be achieved using our framework for any feasible real-valued choices of $q$ and $p$, extending previous results to non-integer factorization depths.
\vspace{-0.3cm}

\subsection{Hadamard Power Parametrization}\label{sec:hpowp-subsect}

To construct a parametrization that induces $\ell_{q}$ regularization of $\bm{\beta}$ under (slightly modified) $\ell_2$ regularization of the surrogate parameters for any $q\in (0,1]$, we extend the notion of self-Hadamard products to Hadamard powers. For powers $v_j^k$ with positive, real-valued exponents $k$ to be well-defined, we require positivity of the base $v_j$, e.g., by designing parametrizations of the form $\beta_j = u_j |v_j|^{k-1}$. 
The resulting $\text{HPowP}_k$ map is
\begin{equation} \label{eq:hpowp}
\mathcal{K}: \mathbb{R}^{d} \times \mathbb{R}^{d} \to \mathbb{R}^{d}, (\bm{u}, \bm{v}) \mapsto  \bm{u} \odot |\bm{v}|^{\circ (k-1)} = \bm{\beta} \,,
\end{equation}
where $|\bm{v}|^{\circ(k-1)}$ denotes element-wise raising the $|v_j|$ to the $(k-1)$-th power, with $k>2$ to ensure $\K \in \mathcal{C}^1$. This generalizes the self-Hadamard product $\bigodot_{l=1}^{k-1} \bm{v} = \bm{v}^{k-1}$, defined for $k\in\mathbb{N}$, to real-valued positive exponents, with $\circ(k-1)$ denoting non-integer exponents.
\begin{lemma} \label{lemma:S_hpowp} Given the parametrization $\K(\bu,\bv)=\bm{u} \odot |\bm{v}|^{\circ(k-1)}$, the minimum surrogate $\ell_2$ regularization $\Rxi(\bu,\bv) \triangleq \norm{\bu}_2^2+(k-1)\norm{\bv}_2^2$ subject to $\K(\bu,\bv)=\bbeta$ constitutes an SVF for $\Rbeta(\bbeta)\triangleq k \norm{\bbeta}_{2/k}^{2/k}$, i.e.,
  $\min_{\bm{u}, \bm{v}: \bm{u} \odot |\bm{v}|^{\circ(k-1)} = \bm{\beta}} \Vert \bm{u} \Vert_{2}^{2} + (k-1) \Vert \bm{v}  \Vert_{2}^{2} = k \Vert \bm{\beta} \Vert_{2/k}^{2/k}  \,\forall\bbeta\in\Rd$.
\end{lemma}
Note that the sign of the constrained minimizer $\hat{u}_j$ is uniquely determined by the sign of $\beta_j$ due to the non-negativity of $|\hat{v}_j|^{k-1}$. By the optimality conditions, the squared coefficients $u_j^2$ and $|v_j|^2$ must equal $|\beta_j|^{2/k}$ at the minimum, so that by Lemma~\ref{lemma:lhc-solution-map}, the set-valued solution map is lower hemicontinuous and Assumption~\ref{ass:min} is satisfied. Thus, for any $k>2$, given an $\ell_{2/k}$ regularized base objective $\P(\bpsi,\bbeta)$,
\begin{equation}
\mathcal{P}(\bm{\psi}, \bm{\beta}) = \mathcal{L}(\bm{\psi}, \bm{\beta}) + \lambda k \|\bm{\beta}\|_{2/k}^{2/k} %
\label{eq:optim_q_WAMGM} \,,
\end{equation}
we can construct an equivalent differentiable $\Q(\bpsi,\bu,\bv)$ from the tuple $(\Rbeta,\K,\Rxi)$:
\begin{align}\label{eq:optim_g_WAMGM}
\mathcal{Q}(\bm{\psi}, \bm{u},\bm{v}) &=  \mathcal{L}\big(\bm{\psi}, \bm{u} \odot |\bm{v}|^{\circ(k-1)}\big) + \lambda \left( \Vert \bm{u} \Vert_{2}^{2} + (k-1) \Vert \bm{v} \Vert_{2}^{2} \right)\,. %
\end{align}
%
%
\begin{corollary}\label{cor:hpowp}
The optimization of $\P$ in (\ref{eq:optim_q_WAMGM}) is equivalent to the optimization of the smooth surrogate $\Q$ in (\ref{eq:optim_g_WAMGM}) by Definition~\ref{def:equivalence}, and solutions to $\mathcal{P}$ can be constructed as $(\hbpsi,\hat{\bm{\beta}}) = (\hbpsi,\hbu \odot |\hbv|^{\circ(k-1)})$.
\end{corollary}
Note that similar to Lemma~\ref{lemma:S_hppk-shared-def}, we modify the usual $\ell_2$ regularization by multiplying each of the $|v_j|^2$ by $(k-1)$ to reflect the imbalance of $\bm{u}$ and $\bm{v}$ in the parametrization $\bm{\beta} = \bm{u} \odot |\bm{v}|^{\circ(k-1)}$. %

%
%
\subsection{Invertible Reparametrization with Hadamard Powers}\label{subset:powerprop}
In \citet{schwarz2021powerpropagation}, a differentiable sparsity-promoting parametrization termed Powerpropagation was introduced, aligning with discussions of related approaches in mathematical optimization \citep{ramlau2012minimization}. The underlying motivation is to artificially increase the curvature of the loss landscape, which induces optimization- and initialization-dependent ``rich get richer" dynamics for sparse training of DNNs: the key idea is that applying a power parametrization makes the gradient with respect to the surrogate parameters critically depend on their current values 
(cf.~Figure~\ref{fig:landscape-powerprop}). \\
Intuitively, this promotes the accumulation of weights either close to or far away from zero, however, \citet{schwarz2021powerpropagation} did not realize the induced sparse regularization in the base parametrization under explicit $\ell_2$ regularization. Being bijective, Powerpropagation is not an over- but rather a reparametrization given by
\begin{equation} \label{eq:powerprop-param}
\mathcal{K}: \mathbb{R}^{d} \to \mathbb{R}^{d}, \bm{v} \mapsto  \bm{v} \odot |\bm{v}|^{\circ (k-1)} = \bm{\beta} \quad , \, k>1.
\end{equation}
Note that with a single parameter $\bv$, it suffices to require $k>1$ to ensure $\K \in \mathcal{C}^1$ contrasting the previous $\text{HPowP}_k$ (\ref{eq:hpowp}). To see the ``rich get richer'' effect, consider a generic objective $\P(\bbeta)$, whose gradient under Powerporpagation is given by $\nabla_{\bv}\P(\K(\bv))=\nabla_{\bbeta}\P(\bbeta) \cdot \text{diag}(k |\bv|^{\circ(k-1)})$. This additional factor causes amplification of gradients for $v_j$ with large magnitudes and attenuation for small magnitudes.
Considering $\ell_2$ regularization for this parametrization, the feasible set of the problem $\min_{\bm{v}: \bm{v} \odot |\bm{v}|^{\circ(k-1)} = \bm{\beta}}\Vert \bm{v} \Vert_{2}^{2}$ is a singleton containing $\hat{\bm{v}}$ such that $\hat{v}_j = \sqrt[k]{|\beta_j|}$ for $\beta_j \geq 0$ and $\hat{v}_j = -\sqrt[k]{|\beta_j|}$ for $\beta_j<0$, $j \in [d]$. Hence, $\Vert \hbv \Vert_{2}^{2}$ contains $d$ summands $\hat{v}_j^2 = |\beta_j|^{2/k}$, and we conclude $\min_{\bm{v}: \bm{v} \odot |\bm{v}|^{\circ(k-1)} = \bm{\beta}}\Vert \bm{v} \Vert_{2}^{2} = \Vert \bm{\beta} \Vert_{2/k}^{2/k}$.
Since the solution map is continuous in $\bbeta$, Assumption~\ref{ass:min} holds. Thus, for an $\ell_{2/k}$ regularized objective $\P(\bpsi,\bbeta)$ with real-valued $k>1$, we can construct an equivalent smooth $\Q(\bpsi,\bv)$ as follows:
\begin{align}
\mathcal{P}(\bm{\psi}, \bm{\beta}) &= \mathcal{L}(\bm{\psi}, \bm{\beta}) + \lambda  \|\bm{\beta}\|_{2/k}^{2/k} 
, \label{eq:optim_q_bijective} \\
\mathcal{Q}(\bm{\psi},\bm{v}) &=  \mathcal{L}\big(\bm{\psi}, \bm{v} \odot |\bm{v}|^{\circ(k-1)}\big) + \lambda \Vert \bm{v} \Vert_{2}^{2}
\,.  \label{eq:optim_g_bijective}
\end{align}
\begin{corollary}\label{cor:powerprop}
The optimization of $\P$ in (\ref{eq:optim_q_bijective}) is equivalent to the optimization of the smooth surrogate $\Q$ in (\ref{eq:optim_g_bijective}) by Definition~\ref{def:equivalence}, and solutions to $\mathcal{P}$ can be constructed as $(\hbpsi,\hat{\bm{\beta}}) = (\hbpsi,\hbv \odot |\hbv|^{\circ(k-1)})$.
\end{corollary}
This result shows that it is the functional shape of the parametrization and its warping effect on the loss surface that induces sparsity, not overparametrization \textit{per se}. %
%
%
%
%
\subsection{Hadamard Group Powers (GHPowP)} \label{sec:had-group-powers}
We can naturally extend the Hadamard power parametrization presented in \ref{sec:hpowp-subsect} to structured sparsity, thereby obtaining a more flexible choice of the hyperparameters $p$ and $q$ in $\ell_{p,q}$ regularization. The following two subsections are structured analogously to their Hadamard product-based counterparts discussed in Section~\ref{sec:had-group-lasso}. As before, we consider the parameter vector with group structure $\bm{\beta} = \left(\bm{\beta}_{1},\ldots,\bm{\beta}_{L}\right)^{\top}$. Consider the following parametrization mapping, named the $\text{GHPowP}_k$,
\begin{equation} \label{eq:reparam-had-group-powers}
{\scriptsize
\mathcal{K}: \mathbb{R}^{d} \times \mathbb{R}^{L} \to \mathbb{R}^{d}, (\bm{u}, \bm{\nu}) \mapsto  \bm{u} \odotg |\bnu|^{\circ (k-1)} = \begin{pmatrix} 
    \bm{u}_{1} \\ \vdots \\ \bm{u}_{L}
  \end{pmatrix}  \odot \begin{pmatrix} 
    |\nu_1|^{k-1} \mathds{1}_{|\mathcal{G}_1|} \\ \vdots \\ |\nu_{L}|^{k-1} \mathds{1}_{|\mathcal{G}_L|}
  \end{pmatrix}  = \bm{\beta} \,.}
\end{equation}
On the group level, we have $\bm{\beta}_{{j}} = |\nu_{j}|^{k-1} (u_{j1},\ldots,u_{j |\G_j|})^{\top}\, \forall j\in [L]$, where $\bm{u} = \left(\bm{u}_{{1}}, \ldots, \bm{u}_{{L}} \right)^{\top}$, $\bm{\nu} = (\nu_{1},\ldots, \nu_{L})^{\top}$, and $k>2$. Now, define $\mathcal{R}_{\bm{\beta}}(\bm{\beta}) \triangleq k \Vert \bm{\beta} \Vert_{2,2/k}^{2/k}$ and $\mathcal{R}_{\bm{\xi}}(\bu,\bnu) \triangleq \Vert \bm{u} \Vert_{2}^{2} + (k-1) \Vert \bm{\nu} \Vert_{2}^{2}$. 
\begin{lemma} \label{lemma:S_ghpowp} Given the parametrization map $\K(\bu,\bnu)=\bm{u} \odotg |\bnu|^{\circ(k-1)}$, the minimum of the surrogate $\ell_2$ regularizer $\Rxi(\bu,\bnu) \triangleq \norm{\bu}_2^2+(k-1)\norm{\bnu}_2^2$ subject to $\K(\bu,\bnu)=\bbeta$ constitutes an SVF for $\Rbeta(\bbeta)\triangleq k \norm{\bbeta}_{2,2/k}^{2/k}$ and is given by
\begin{equation} \label{eq:S_ghpowp}
      \min_{\bm{u}, \bm{\nu}: \bm{u} \odotg |\bnu|^{\circ(k-1)} = \bm{\beta}} \Vert \bm{u} \Vert_{2}^{2} + (k-1) \Vert \bm{\nu} \Vert_{2}^{2} = k \Vert \bm{\beta} \Vert_{2,2/k}^{2/k} \quad\forall\bbeta\in\Rd \,.
\end{equation}
\end{lemma}
Using  Lemma~\ref{lemma:lhc-solution-map}, the optimality conditions provided in the proof ensure lower hemicontinuity of the solution map. Therefore, Assumption~\ref{ass:min} holds and we can construct an equivalent smooth surrogate $\Q$ to an $\ell_{2,2/k}$ regularized base objective $\P(\bpsi,\bbeta)$ for any $k>2$: 
\begin{align}
&\mathcal{P}(\bm{\psi}, \bm{\beta}) = \mathcal{L}(\bm{\psi}, \bm{\beta}) + \lambda k \|\bm{\beta}\|_{2,2/k}^{2/k} 
, \label{eq:optim_q_had-group-power} \\
&\mathcal{Q}(\bm{\psi}, \bm{u},\bm{\nu}) =  \mathcal{L}\big(\bm{\psi}, \bm{u} \odotg |\bnu|^{\circ(k-1)}\big) + \lambda \big( \Vert \bm{u} \Vert_{2}^{2} + (k-1) \Vert \bm{\nu} \Vert_{2}^{2} \big). \label{eq:optim_g_had-group-power}  
\end{align}

\begin{corollary}\label{cor:ghpowp}
The optimization of $\P$ in  (\ref{eq:optim_q_had-group-power}) is equivalent to the optimization of $\Q$ in (\ref{eq:optim_g_had-group-power}) by Def.~\ref{def:equivalence}, and solutions to $\mathcal{P}$ can be constructed as $(\hbpsi,\hat{\bm{\beta}}) = (\hbpsi,\hbu \odotg |\hbnu|^{\circ(k-1)})$.
\end{corollary}
%
%
%
\subsection[Mixed Norm Regularization with Hadamard Group Powers]{Mixed Norm Regularization with Hadamard Group Powers} \label{sec:had-lpq-group-powers}

Analogous to the previous subsection, we can further apply Hadamard powers to induce $\ell_{p,q}$ mixed-norm regularization for arbitrary feasible values $0<q < p \leq 2$. As a starting point, we again revisit the structured group set-up $\bm{\beta}=\left(\bm{\beta}_{1},\ldots, \bm{\beta}_{L}\right)^{\top}$. However, to allow for non-integer factorization depths, a more complex nested power parametrization 
is required and constructed in the following. %
Consider the parametrization $\bm{\beta} = \bm{u} \odotg |\bnu|^{\circ k_2}$, $k_2>1$, with $\bm{u} = \left( \bm{u}_{1},\ldots,\bm{u}_{L} \right)^{\top} \in \Rd$ and $\bm{\nu}=(\nu_{1},\ldots,\nu_{L})^{\top} \in \mathbb{R}^{L}$, corresponding to parametrization (\ref{eq:reparam-had-group-powers}).
Additionally, the auxiliary parameter $\bu$ is parametrized using an invertible pre-composition, i.e., $\bu=\bm{\mu} \odot |\bm{\mu}|^{\circ(k_1-1)}$ with $k_1>1$ and surrogate parameters $\bm{\mu} = (\bm{\mu}_{1}, \ldots, \bm{\mu}_{L})^{\top} \in \mathbb{R}^{|\mathcal{G}_1|+\ldots+|\mathcal{G}_L|} = \mathbb{R}^{d}$. We can then define the $\text{GHPowP}_{k_1,k_1+k_2}$ as
{\small
\begin{align} \label{eq:reparam-had-lpq-powers}
\mathcal{K}: \mathbb{R}^{d} \times \mathbb{R}^{L} \to \mathbb{R}^{d},\,&(\bm{\mu}, \bm{\nu}) \mapsto \bm{\mu} \odot |\bm{\mu}|^{\circ(k_1-1)} \odotg |\bnu|^{\circ k_2} 
= {\footnotesize \begin{pmatrix} 
\bm{\mu}_{1} \odot |\bm{\mu}_{1}|^{\circ (k_1-1)} \\ \vdots \\ \bm{\mu}_{L} \odot |\bm{\mu}_{L}|^{\circ (k_1-1)}
\end{pmatrix}  \odot \begin{pmatrix} 
|\nu_1|^{k_2}\mathds{1}_{|\mathcal{G}_1|} \\ \vdots \\ |\nu_{L}|^{k_2}\mathds{1}_{|\mathcal{G}_L|} \\
\end{pmatrix}}, \nonumber
\end{align}}

\noindent or equivalently on the group level, $\bm{\beta}_{{j}} = \bu_j |\nu_j|^{k_2} = \bm{\mu}_{j} \odot |\bm{\mu}_{j}|^{\circ (k_1-1)} \cdot |\nu_{j}|^{k_2}$ for groups $j \in [L]$. The parametrization of $\bm{u}_{{j}}$ via $\bm{\mu}_{j}$ is bijective, so that for each $u_{ji}, i \in \mathcal{G}_{j}$ in $\bm{u}_{{j}}$, it holds $\mu_{ji} = \text{sign}(u_{ji})\cdot |u_{ji}|^{1/k_1}$. Thus, we can express the squared Euclidean norm of $\bm{\mu}_{j}$ as 
\begin{equation}
    \Vert \bm{\mu}_{j} \Vert_{2}^{2} = \sum_{i \in \mathcal{G}_{j}} \mu_{ji}^{2} = \sum_{i \in \mathcal{G}_{j}} |u_{ji}|^{2/k_1} = \Vert \bm{u}_{{j}} \Vert_{2/k_1}^{2/k_1} \,. \nonumber
\end{equation}
Letting $k\triangleq k_1+k_2>2$, we define the non-convex base regularizer as $\mathcal{R}_{\bm{\beta}}(\bm{\beta}) \triangleq k \Vert \bm{\beta} \Vert_{2/k_1,2/k}^{2/k}$ and the surrogate as $\mathcal{R}_{\bm{\xi}}(\bm{\mu},\bnu) \triangleq k_1 \Vert \bm{\mu} \Vert_{2}^{2} + k_2 \Vert \bm{\nu} \Vert_{2}^{2}$. Together, $\K$ and $\Rxi$ form an SVF for $\Rbeta$: 
\begin{lemma} \label{lemma:S_ghpowpk1} For a parametrization $\K(\bu,\bnu)=\bm{\mu} \odot |\bm{\mu}|^{\circ(k_1-1)} \odotg |\bnu|^{\circ k_2}$, the minimum of the surrogate $\ell_2$ regularizer $\Rxi(\bu,\bnu) \triangleq k_1 \Vert \bm{\mu} \Vert_{2}^{2} + k_2 \Vert \bm{\nu} \Vert_{2}^{2}$ subject to $\K(\bu,\bnu)=\bbeta$ constitutes an SVF for $\Rbeta(\bbeta)\triangleq k \norm{\bbeta}_{2/k_1,2/k}^{2/k}$ and is given by
\begin{equation} \label{eq:S_ghpowpk1}
          \min_{\substack{\bm{\mu}, \bm{\nu}: \bm{\mu}\odot|\bm{\mu}|^{\circ(k_1-1)} \odotg |\bnu|^{\circ k_2} = \bm{\beta}}}  k_1 \Vert \bm{\mu} \Vert_{2}^{2} + k_2 \Vert \bm{\nu} \Vert_{2}^{2}  = k \Vert \bm{\beta} \Vert_{2/k_1,2/k}^{2/k} \quad\forall\bbeta\in\Rd \,,
\end{equation}
\end{lemma}
The optimality conditions obtained in the proof of this result further ensure Assumption~\ref{ass:min} holds by establishing lower hemicontinuity of the set-valued solution map of the SVF according to Lemma~\ref{lemma:lhc-solution-map}. Given an $\ell_{2/k_1,2/k}$ regularized objective $\P(\bpsi,\bbeta)$, we can construct a surrogate $\Q(\bpsi,\bm{\mu},\bnu)$ from the tuple $(\Rbeta,\K,\Rxi)$: 
\begin{align}
&\mathcal{P}(\bm{\psi}, \bm{\beta}) = \mathcal{L}(\bm{\psi}, \bm{\beta}) + \lambda k \|\bm{\beta}\|_{2/k_1,2/k}^{2/k}, \label{eq:optim_q_had-lpq-power} 
\\
&\mathcal{Q}(\bm{\psi}, \bm{\mu},\bm{\nu}) =  \mathcal{L}\big(\bm{\psi}, (\bm{\mu} \odot |\bm{\mu}|^{\circ (k_1-1)}) \odotg |\bnu|^{\circ k_2}\big) + \lambda \left( k_1 \Vert \bm{\mu} \Vert_{2}^{2} + k_2 \Vert \bm{\nu} \Vert_{2}^{2} \right). \label{eq:optim_g_had-lpq-power} 
\end{align}

\begin{corollary}\label{cor:ghpowpk1}
The optimization of $\P(\bpsi,\bbeta)$ in (\ref{eq:optim_q_had-lpq-power}) is equivalent to the optimization of the smooth surrogate $\Q(\bpsi,\bu, \bnu)$ in (\ref{eq:optim_g_had-lpq-power}) for any $k_1,k_2>1$ according to Definition~\ref{def:equivalence}, and solutions to $\mathcal{P}$ can be constructed from solutions to $\mathcal{Q}$ as $(\hbpsi,\hat{\bm{\beta}}) = (\hbpsi,\bm{\mu} \odot |\bm{\mu}|^{\circ (k_1-1)} \odotg |\hbnu|^{\circ k_2})$.
\end{corollary}
%

%

\section{Optimization Details}\label{sec:optim}

In this section, we discuss optimization details of our smooth optimization transfer approach and provide some guidance regarding practical implementations.

\vspace{0.15cm}
\noindent \textbf{Iterative optimization using (S)GD}\, 
A considerable body of literature has established desirable convergence properties of (S)GD that hold in overparametrized non-convex settings, such as provably almost always escaping (strict) saddle points under random initialization and mild regularity conditions \citep{lee2016gradient, lee2019first}. For full-batch GD, however, \citet{du2017gradient} show that it might take exponentially long to escape saddle points. This can be reduced to polynomial time in the presence of sufficient perturbation in the gradient updates \citep{ge2015escaping, jin2017escape}, emphasizing the benefit of SGD in efficiently optimizing non-convex problems.\\ 
The effect of applying a power-product parametrization on the optimization landscape is to transfer the problem to a more curved space, which impacts the optimization geometry of (S)GD in a way that has been termed the ``rich get richer" effect in the Powerpropagation literature \cite{schwarz2021powerpropagation}, but can also take other forms depending on the parametrization. The effect hinges on the multiplicative structure of the parametrizations $\K$, leading to additional multiplicative dependence of the gradient updates of one factor on the current values of a subset of all factors.
For the bijective Powerpropagation (\ref{eq:powerprop-param}), there is only a single surrogate parameter. Hence, the multiplicative dependence on current values becomes self-reinforcing: As discussed in Section~\ref{subset:powerprop}, loss gradients for large-magnitude parameters are magnified, while the gradients for small-magnitude parameters are shrunken, promoting a heavy-tailed weight distribution. The same holds for the HDP (\ref{eq:hdp-def}), but separately for each surrogate parameter: its (entry-wise) gradient is $\nabla \K_j(\gamma_j,\delta_j)=(2\gamma_j,2\delta_j)$, resulting in self-reinforcing ``rich get richer'' dynamics for $\gamma_j$ and $\delta_j$, respectively. Interestingly, for the HPP (\ref{eq:hpp-def}), the multiplicative dependence implies dynamics that are more akin to a ``Robin Hood'' effect that balances the effective learning rates for both parameters, because its entry-wise gradient is $(v_j,u_j)$, i.e., the partial derivatives contain the \textit{other} factor. For an unbalanced factorization $|u_j|\gg|v_j|$, $u_j$ moves slowly because the respective gradient is attenuated by $v_j$, and $v_j$ moves fast because the gradient is magnified by $u_j$. If $|u_j|=|v_j|$, then both parameters have the same effective learning rate. 
These adaptive dynamics also show in the gradient of the smooth surrogate (\ref{eq:gradient-intext}) where the product-structured Jacobian $\mathcal{J}_{\K(\bxi)}(\bxi)$ essentially acts as a parameter-dependent preconditioner leading to adaptive step sizes and momentum \citep{arora2019implicit}.\\ %
Besides overparametrization, our approach also imposes differentiable surrogate regularization, which fundamentally differentiates our approach from the mere (unpenalized) overparametrization in implicit regularization. In these methods, the optimization dynamics of the overparametrized problem are tweaked, e.g., via impractically small initialization scales \citep{woodworth2020kernel, zhao2022high, vaskevicius2019implicit}, so that the optimizer converges to a specific solution on the global minima manifold, such as the minimum $\ell_1$-norm solution. In contrast, our optimization transfer approach transforms the sparsity-regularized problem while exactly preserving the solution structure, and hence, in theory, is agnostic to how this solution is reached. This means our smooth surrogate can, in principle, be solved with any optimizer and does not rely on changing the gradient flow dynamics in a specific way to induce implicit regularization behavior.

%
%

\vspace{0.15cm}
\noindent \textbf{Critical points}\, 
Due to the results obtained in Lemma~\ref{lemma:p_to_q} and Lemma~\ref{lemma:q_to_p}, any local minimum of the surrogate optimization problem corresponds to a local minimum in the base parametrization. As a result, if the base optimization problem $\P(\bpsi,\bbeta)$
is convex, e.g., for a convex $\L(\bpsi,\bbeta)$ with $\ell_1$ or $\ell_{2,1}$ regularization, every local minimum of the surrogate problem $\Q(\bpsi,\bxi)$ is necessarily global. For non-convex base problems, our approach ensures no spurious minima are created in the optimization transfer.

\noindent However, such a matching property does not necessarily hold for critical points of the surrogate $\Q$, owed to the zero-product property of the parametrizations $\K$. Without loss of generality, consider a non-smooth regularized objective $\P(\bbeta)=\L(\bbeta)+\lambda \Rbeta(\bbeta)$ with smooth loss $\L(\bbeta)$ and no additional unregularized parameters $\bpsi$. Applying the proposed smooth optimization transfer, we construct the surrogate $\Q(\bxi)=\L(\K(\bxi))+\lambda \Rxi(\bxi)$ using a smooth parametrization $\K(\bxi)$ and further imposing surrogate $\ell_2$ regularization on $\bxi$. The gradient of $\Q$ with respect to $\bm{\xi}$ is then given by
\begin{equation} \label{eq:gradient-intext}
\nabla_{\bxi} \Q(\bm{\xi}) = \mathcal{J}_{\mathcal{K}(\bm{\xi})}^{\top}(\bxi) \nabla_{\K} \mathcal{L}(\mathcal{K}(\bm{\xi})) + \lambda \nabla_{\bxi} \mathcal{R}_{\bm{\xi}}(\bxi)\,,
\end{equation}
where $\mathcal{J}_{\mathcal{K}(\bm{\xi})}(\bxi)$ is the $d \times \dxi$-dimensional Jacobian of $\mathcal{K}$ at $\bm{\xi}$, and the gradients $\nabla_{\K} \mathcal{L}(\mathcal{K}(\bm{\xi}))$ and $\nabla_{\bxi} \mathcal{R}_{\bm{\xi}}(\bxi)$ are $d$- and $\dxi$-dimensional vectors, respectively. For the parametrizations we consider, the Jacobian $\mathcal{J}_{\K(\bxi)}(\bm{0})$ at $\bxi=\bm{0}$ is the null matrix. As $\Rxi(\bxi)$ is a type of $\ell_2$ penalty, we have $\nabla_{\bxi} \Rxi(\bm{0})=\bm{0}$, and it follows $\nabla_{\bxi} \Q(\bm{0})=\bm{0}$. Therefore, $\bxi=\bm{0}$ is a critical point of $\Q$, irrespective of the gradient $\nabla_{\K} \L(\K(\bxi))$ of $\L$ in the base objective. A derivation of the Hessian of $\Q$ is given in Appendix~\ref{app:derivations-optim}.

\noindent Regarding the nature of potentially spurious critical points, it is known that parametrizations of depth $k=2$, such as the HPP or HDP, only induce strict saddle points at $\bxi=\bm{0}$, since their Hessian evaluated at the origin $\mathcal{H}_{\K}(\bm{0})$ contains parameter-independent non-zero constants that ensure a strictly negative eigenvalue \citep{zhao2022high}.\footnote{A strict or ridable saddle point is a saddle at which the Hessian has at least one strictly negative eigenvalue, i.e., there is a direction of descent.} %
Through construction of a counterexample, \citet[][Corollary 2.4]{kawaguchi2016deep} shows that the strict saddle property does not necessarily hold for deep factorizations with depth $k>2$. In our framework, this corresponds to those parametrizations $\K$ that induce non-convex $\ell_q$ or $\ell_{p,q}$ regularization in the base objective under surrogate $\ell_2$ regularization. For this class of non-convex regularizers with unbounded derivatives approaching the origin, $\hbbeta=\bm{0}$ is always a local minimizer in the base problem $\P(\bbeta)$, 
regardless of $\L(\bbeta)$ \citep{loh2015regularized}. In the constructed smooth surrogate $\Q(\bxi)$, this is reflected in the Hessian $\mathcal{H}_{\Q(\bxi)}(\bxi)$. For $k>2$ and $\lambda=0$, the Hessian at $\bxi=\bm{0}$ degenerates to a null matrix, $\mathcal{H}_{\Q(\bxi)}(\bm{0})=\bm{0}$, inducing a higher-order saddle point. 
For the regularized problems we are interested in, 
the strong convexity of $\Rxi(\bxi)$ guarantees that $\mathcal{H}_{\Q(\bxi)}(\bm{0})$ has only positive eigenvalues. Thus, $\hbxi=\bm{0}$ is a local minimizer of $\Q$, corresponding to the local minimizer $\hbbeta=\bm{0}$ in $\P(\bbeta)$ that is induced by the non-convex regularizer $\Rbeta(\bbeta)$. Hence, the additional $\ell_2$ regularization in our smooth surrogate avoids the problematic spurious non-strict saddle point at $\bxi=\bm{0}$ induced by $\K$, even for non-convex regularization with $k>2$. Importantly,
$\hbxi=\bm{0}$ being a local minimizer of the surrogate $\Q(\bxi)$ for non-convex $\ell_q$ and $\ell_{p,q}$ regularization in the base problem $\P(\bbeta)$ is not a property of our proposed method, but of the non-convex regularizer $\Rbeta(\bbeta)$.
%
%
\vspace{0.15cm}

\noindent \textbf{Initialization}\, Another relevant question concerns finding effective and well-founded initializations for the surrogate parameters, and how they relate to an appropriate initialization of the base parameter $\bbeta^{t}$ at $t=0$. 
A natural approach would be to initialize the surrogate parameters functionally equivalent to a standard initialization scheme for the base parameter $\bbeta^{0}$. However, in the case of overparametrization, there are many such options, and it is \textit{a priori} unclear how to optimally select among feasible initializations of $\bbeta^{0}$. It seems natural to initialize the surrogate parameters $\bxi$ according to the optimality conditions provided by the implemented SVF, i.e., $\bxi^{0} = \hbxi(\bbeta^{0})$, where $\hbxi(\bbeta)$ is the set-valued solution mapping of the SVF, and $\bbeta^{0}$ is obtained from a standard initialization scheme for the base parameters. This ensures that the optimization is initialized at a minimizer of the surrogate penalty $\Rxi(\bxi)$ over $\{\bxi: \K(\bxi)=\bbeta^{0}\}$. 

\noindent To provide two examples, consider a parametrization of $\bbeta$ using $\text{HPP}_k$, i.e., $\bbeta=\bu_l^{\odot k}$ with surrogate $\ell_2$ regularization. One approach then entails initializing the surrogate factors $\bu_l$ identically as $\bu_l^{0}=\sqrt[k]{|\bbeta^{0}|}$, and subsequently multiplying one (arbitrary) factor $\bu_l^{0}$ by the respective signs of $\bbeta^0$.\footnote{Applying any sign pattern to the $\bu_l$ that respects the signs of $\bbeta^{0}$ under the parametrization $\K$ is valid.} For structured sparsity using the GHPP, $\bbeta = \bu \odotg \bnu$, the surrogate parameters are initialized as {\footnotesize $\nu_j^{0} = \sqrt{\Vert \bbeta_j^{0} \Vert_2}$} and {\footnotesize $\bu_j^{0}=\bbeta_j^{0}/\sqrt{\Vert \bbeta_j^{0} \Vert_2}$}, again equivalently for sign patterns $\pm (\bu_j^0,\nu_j^0)$. Another option would be to randomly initialize all factors, but increase the initialization scale so that the product is initialized at a desired scale, a variant of which is proposed in \cite{kolb2025deep}.

%
\vspace{0.15cm}
\noindent \textbf{Effects on Optimization Landscape}\, 
The parametrizations $\K(\bxi)=\bbeta$ considered in this work (cf.~Assumption~\ref{ass-parametrization-map}) are based on Hadamard products and powers. This has a notable effect on the loss landscape, primarily due to a modification of curvature induced by the multiplicative nature of the parametrizations. For the bijective Powerpropagation (\ref{eq:powerprop-param}), the warping effect of the reparametrization takes place in the same space as the base parameters and can thus be disentangled from overparametrization. Appendix~\ref{app:effects-optim-landscape} contains more details.


\vspace{0.15cm}
\noindent \textbf{Practical implementation}\, 
It is important to consider the case when the surrogate parameters $\bxi_j$ corresponding to some base parameter $\beta_j$ are initialized in an orthant that maps to an incorrect sign under the multiplicative parametrization $\K$ compared to the solution $\hat{\beta}_j$. In these cases, it is crucial to use large learning rates during early iterations, as previously suggested by, e.g., \citet{li2023implicit}. Otherwise, the respective parameter iterates will gradually approach zero from the side of the initial orthant. This occurs due to the ``rich get richer'' effect, resulting in diminishing gradient magnitudes as the parameter approaches zero, making it difficult to ``step over" the zero boundary. Although for most DNNs the sign pattern is not identified, large step sizes in DL have been found to drive SGD toward simpler structures \citep{andriushchenko2023sgd, chen2024stochastic} and balanced Hadamard factors \citep{ziyin2024symmetry}, thus facilitating sparse optimization.
%
Besides initially large learning rates, we further emphasize the importance of the commonplace recommendation of using either small batch sizes in SGD or perturbing the gradient updates via additional noise injection for faster convergence and the improved ability to escape saddles and local minima \citep{jin2017escape}.
Further, note that using (S)GD to optimize the differentiable surrogate does not have an inherent proximal step. Consequently, the iterates can only asymptotically approach theoretically zero values in a finite number of steps. However, this issue is benign, as with standard training hyperparameters, numerically zero or negligibly small floating point representations can be attained. For additionally accelerated and more adaptive optimization, a dynamic thresholding schedule can be implemented, in which during training the surrogate parameters $\bxi_j=\bm{0}$ parametrizing a scalar parameter are thresholded if the reconstructed parameter $|\beta_j|<\varepsilon_{tiny}$. In this way, the weights are automatically removed from future updates early on, allowing the optimization more time exploring sparser subnetworks without having to implement a hard mask or change the architecture. This is not required due to the multiplicative structure of the gradients of $\K$, which necessarily all become zero for all future iterations once $\bxi_j=\bm{0}$.
%
%

\section{Related Work} \label{sec:related-work}
In this section, we relate our optimization transfer framework and the presented parametrizations to prior art. In recent years, parametrizations based on Hadamard products 
have attracted considerable interest in several fields, including DL, statistics, signal processing, and optimization. The context in which they are applied varies significantly, so we focus on works that use these parametrizations for sparsity-inducing regularization with explicit surrogate regularization, i.e., not relying on manipulating the optimization dynamics as in implicit regularization approaches.

\noindent \textbf{Comparison to prior work using explicit regularization}\, In our work, we extend the literature on approximation-free, differentiable optimization for sparse regularization using a combination of overparametrization and surrogate regularization. An early connection between $\ell_1$ and an adaptive variant of $\ell_2$ regularization using the HPP was first observed by \citet{grandvalet1998least}. In statistics, the basic idea was re-discovered for a restricted problem class by \citet{hoff2017lasso} using the $\text{HPP}_k$, however, without noting its compatibility with SGD or its applicability to non-linear models. Unlike our work, their suggested method is limited to linear models, for which the overparametrized objective can be optimized using alternating ridge regressions due to the multi-convex nature of the optimization problem. This multi-convexity is lost in more complex models like neural networks, requiring less restrictive optimizers like SGD.
%
Besides these, \citet{tibs2021} additionally studies the $\text{GHPP}_k$ in linear models and finds that they have identical global minima to certain weight-decayed network architectures (cf.~Figure~\ref{fig:hpp-hdp-dep1}). Notably, a simple weight-decayed diagonal linear network with one hidden layer has the same global minimum as the lasso, which, in turn, is equivalent to applying the Hadamard product parametrization and $\ell_2$ regularization. This observation can also be implicitly inferred from the representation cost analysis of the same architecture presented in \citet{dai2021representation}. Building on the results of \citet{hoff2017lasso}, the work of \citet{ziyin2022spred} constitutes a first endeavor to adapt differentiable sparse regularization to the dominant SGD optimization paradigm for DNNs, however, using only a two-parameter factorization. %

\indent Previous works, however, exhibit several limitations. First, their scope is limited to single or a small subset of known sparsity-inducing parametrizations (cf.~Table~\ref{tab:overview}), while we are the first to provide a comprehensive account. Among the works incorporating $\ell_2$ regularization to induce sparsity, the proposed methods are either confined to linear models by scope \citep{tibs2021, dai2021representation} or restrictions of their optimization procedure \citep{hoff2017lasso}. The only work applying overparametrization with $\ell_2$ regularization to broader objectives is \citet{ziyin2022spred}, whose approach is limited to a simple overparametrization for induced convex $\ell_1$-type regularization. 
Further, we place particular emphasis on ensuring matching local minima as a crucial property to preserve structure in the overparametrized problem, which aligns with the work of \citet{levin2020towards} and \citet{nouiehed2022learning} and was previously only discussed by \cite{hoff2017lasso,ziyin2022spred}.\\
Moreover, although the implicit regularization literature also studies several overparametrizations \cite{woodworth2020kernel, gunasekar2018characterizing, nacson2022implicit, vaskevicius2019implicit,li2021implicit, gissin2019implicit, moroshko2020implicit, chou2023more, vivien2022label}, these works do not consider the induced regularizer under explicit surrogate regularization and base their implicit regularization on the manipulation of optimization dynamics, mainly through vanishing initialization scales. This renders implicit regularization approaches impractical and fundamentally different from our approach, besides their restriction to convex $\ell_1$ regularization in important settings \citep{nacson2022implicit}. Due to these limitations, the goal of implicit regularization works is mainly to understand optimization dynamics, and few works are specifically geared toward practical applications \cite{zhao2022high, chou2022non}. In contrast, our method does not rely on manipulation of the optimization dynamics to reach a specific minimizer, but has the same solution structure as the sparsity-regularized problem, a property that is independent of optimization dynamics.\\
Table~\ref{tab:comparison} compares the most closely related prior works. In the following, we further describe the related literature on Hadamard parametrizations in different subfields.

\begin{table}[t]
\centering
\resizebox{1.0\textwidth}{!}{
{
\renewcommand{\arraystretch}{1.5}
\begin{tabular}{l|c|c|c|c|c}  
\hline\hline
\parbox{1.9cm}{\centering Reference} & \parbox{2.8cm}{\centering Regularization} & \parbox{2.5cm}{\centering Induced sparse \\regularizers} & \parbox{1.9cm}{\centering Matching Local Min.} & \parbox{2.3cm}{\centering Corresponding \\ NN struct.} & \parbox{2.7cm}{\centering Application to\\ arbitrary model \\ subcomponents}\\[0.1cm]
\hline 
\citet{grandvalet1998least} & Explicit \normalsize{$\ell_2$} (\scriptsize{adaptive}) & $\ell_1$ & \textcolor{RedOrange}{\xmark} & \textcolor{RedOrange}{\xmark} & \hspace{0.8cm}\textcolor{RedOrange}{\xmark} (LM) \\ 
\citet{hoff2017lasso} & Explicit $\ell_2$ & $\ell_{q}$ \scriptsize{(restricted)} & \textcolor{ForestGreen}{\cmark} & \textcolor{RedOrange}{\xmark} & \hspace{0.8cm}\textcolor{RedOrange}{\xmark} {(LM)}\\ 
\citet{ziyin2022spred} & Explicit $\ell_2$ & $\ell_{1},\ell_{2,1}$  & \hspace{0.7cm}\textcolor{ForestGreen}{\cmark} \,\scriptsize{($\ell_1$)} & \textcolor{RedOrange}{\xmark} & \textcolor{ForestGreen}{\cmark}\\ 
\citet{tibs2021} & Explicit $\ell_2$ & $\ell_{q},\ell_{2,q}$ \scriptsize{(restricted)} & \textcolor{RedOrange}{\xmark} & \textcolor{ForestGreen}{\cmark} & \hspace{0.8cm}\textcolor{RedOrange}{\xmark} (LM)\\ 
\citet{zhao2022high} & Implicit (GD) & $\text{min}$-$\ell_1$-solution & \textcolor{ForestGreen}{\cmark} & \textcolor{RedOrange}{\xmark} & \hspace{0.8cm}\textcolor{RedOrange}{\xmark} (LM)\\ 
Our framework: & Explicit \normalsize{$\ell_2$} (\scriptsize{weighted}) &  $\ell_q,\ell_{p,q}$ & \textcolor{ForestGreen}{\cmark} & \textcolor{ForestGreen}{\cmark} & \textcolor{ForestGreen}{\cmark} \\ 
\hline\hline
\end{tabular}}}
\caption{{\small Overview of related works using Hadamard parametrizations for explicit and implicit sparse regularization. GD stands for gradient descent, and LM for linear model. In the third column, the addition (restricted) refers to a choice of $q=2/k,\,k \in \mathbb{N}$. 
} 
}
\label{tab:comparison}
\vspace{-0.3cm}
\end{table}

\textbf{DL literature}\,\, In the theoretical DL community, the surge in activity can be ascribed to the correspondence of Hadamard product-based parametrizations of linear models and simple, easy-to-analyze network architectures with linear activations \citep{tibs2021,dai2021representation}, predominantly studied under the name of diagonal linear networks \citep{gunasekar2018implicit, gissin2019implicit,pesme2021implicit, li2021implicit, even2023s, wang2023implicit}, as well as similar stylized architecture for structured sparsity \citep{li2023implicit}. These networks are primarily analyzed in the context of implicit regularization effects and the representation cost of neural networks. The first phenomenon studies initialization and trajectory-based regularization effects of (S)GD without any explicit regularization term \citep{vaskevicius2019implicit,woodworth2020kernel}, whereas the latter is concerned with measuring the cost that is required for a DNN to represent particular functions in terms of norms of network weights \citep{dai2021representation,jacotfeature}. Implicit regularization through Hadamard product-based overparametrization of linear models was further extended to robust and sparse linear regression in \citet{ma2022blessing} using subgradient descent. In the absence of explicit regularization, \citet{chou2023more} study the implicit regularization of two variants of Hadamard parametrizations on gradient flow under vanishing initialization, obtaining improved sample complexity for compressed sensing problems. The ``rich'' gradient dynamics \citep{woodworth2020kernel} caused by identical small initialization are further developed for overparametrized non-negative least squares problems in \citet{chou2022non}. Contrasting our explicit surrogate regularization, an important shortcoming of implicit regularization approaches is that it is limited to convex $\ell_q$ norms for $q \geq 1$ for common losses such as the square loss, ruling out non-convex $\ell_q$ regularization for $q<1$ \citep{woodworth2020kernel,nacson2022implicit}. \cite{li2024improving} improve the adaptivity of sequence models using multiplicative overparametrization, whereas \cite{kolb2025dgating,kolb2025differentiable} apply structured overparametrization to induce, e.g., attention sparsity in transformer models.

\noindent \textbf{Statistics literature}\,\, The implicit $\ell_1$ regularization effect of applying a simple Hadamard product parametrization to the parameters of a linear model under vanishing initialization and GD was studied by \citet{zhao2022high}. Under the name ``neuronized priors'', \citet{shin2022neuronized} studies similar parameter factorizations in a Bayesian modeling framework, whereas \cite{cheltsov2024hadamard} proposes Hadamard Langevin dynamics for sampling $\ell_1$ priors. \citep{kaushik2024precise} derive asymptotics for a family of reweighted least-squares approaches for linear models with Hadamard product parametrization.

\noindent \textbf{Signal processing literature}\,\,  \citet{li2023tail} recently applied the Hadamard Product Parametrization (HPP) to solve the tail-$\ell_1$ problem in compressed sensing. In a more general setting, \citet{yang2022better}, and subsequently \citet{parhi2023deep}, analyze the sparse functional representations learned by $\ell_2$ regularized neural networks with homogeneous activation functions from a signal processing perspective, employing a similar line of reasoning to our work to show the equivalence of group sparse $\ell_{2,1}$ and surrogate $\ell_2$ regularization using their Neural Balance Theorem. Similarly, starting with \citet{neyshabur2015norm, neyshabur2015search}, several works advanced the understanding of $\ell_2$ regularized networks and the inductive biases in the learned representations \cite[see, e.g.][]{pilanci2020neural,ergen2021path, ergen2021revealing, jagadeesan2022inductive}.

\noindent \textbf{Optimization literature}\,\, \citet{micchelli2013regularizers} study the optimization of convex regularizers such as the $\ell_1$ penalty by smoothly approximating the absolute value using a quadratic variational formulation of the regularizer involving an additional surrogate parameter $\bm{\eta}$. %
This concept is similarly discussed in \citet{bach2012optimization} under the umbrella term sub-quadratic norms. Formal connections between the so-called $\bm{\eta}$-trick and the Hadamard product (over)parametrizations studied in \citet{hoff2017lasso} are established for convex and lower semicontinuous proper loss functions in \citet{poon2021smooth}, who subsequently leverage Hadamard parametrizations to smooth bilevel programming \citep{poon2023smooth}. Recently, \citet{ouyang2024kurdyka} studied smooth $\ell_1$ regularization using Hadamard parametrizations and derive the surrogate Kurdyka-Lojasiewicz exponent at second-order stationary points from that of the original objective. An analysis of the optimization dynamics of the $\text{HPP}_k$ applied to a linear model under gradient flow is presented in \cite{labarriere2024optimization} and connections to a corresponding mirror flow are made by, e.g., \cite{poon2023smooth, labarriere2024optimization, jacobs2025mask, jacobs2025mirror}. 
Another branch of literature in optimization that is related to our approach  
is the perspective functions framework, a versatile tool for constructing proximal methods \citep{combettes2018perspective1b, combettes2020perspective}. 
%

%
\section{Numerical Experiments}\label{sec:experiments}
In this section, we present experimental findings supporting our theoretical results and demonstrating the generality of our method by applying it to various learning problems ranging from non-convex regularized linear regression to enhanced DNN pruning and filter-sparse convolutional neural networks (CNNs). The main goal of these experiments is not to establish the superiority of our method over other approaches but rather to demonstrate the practical feasibility and competitiveness of using SGD to solve non-smooth regularization.\footnote{We stress that 
the proposed method offers a differentiable formulation of sparse regularizers, thus inherently tying its performance to that of the induced regularizer.} Details on optimization settings and architectures can be found in Appendix~\ref{app:experiments}.

\subsection{Failure of (Sub)GD to Solve Sparse Regularization} \label{subsec:failure-direct-optim}

First, we illustrate the failure of directly applying GD to solve both unstructured and structured sparsity regularization, even in the case of a convex (group) lasso objective with linear predictor and independent features. In DL libraries, the gradient at non-differentiable points is typically assigned zero in the GD update, effectively constituting subgradient descent. To this end, we draw $\bm{X} \in \mathbb{R}^{1000 \times 100}$, $\bbeta \in \mathbb{R}^{100}$, and $\bm{\varepsilon} \in \mathbb{R}^{1000}$ from independent Gaussians and compose the noisy outcome as $\bm{Y}=\bm{X}\bbeta + \bm{\varepsilon}$. For the group lasso, the parameters are partitioned into $L=20$ groups. The objectives in the base parametrization  for both regularizers are
$\P_{\ell_1}(\bbeta)=\frac{1}{n} \Vert \bm{Y}-\bm{X}\bbeta \Vert_2^2 + \lambda \Vert \bbeta \Vert_1$ and $ \P_{\ell_{2,1}}(\bbeta)=\frac{1}{n} \Vert \bm{Y}-\bm{X}\bbeta \Vert_2^2 + \lambda \sum_{j=1}^{L} \Vert \bbeta_j \Vert_2$, and we compare three optimization approaches: directly applying GD to the non-smooth objective, GD under smooth optimization transfer using the (G)HPP, and a highly efficient specialized combination of non-smooth methods, implemented in \texttt{glmnet} \citep{friedman2010regularization} and \texttt{SGL} \citep{simon2013sparse}. The equivalent differentiable objectives of the second approach are defined as $\Q_{\ell_1}(\bu,\bv)= \frac{1}{n} \Vert \bm{Y}-\bm{X} (\bu \odot \bv) \Vert_2^2 + \frac{\lambda}{2} (\Vert \bu \Vert_2^2 + \Vert \bv \Vert_2^2)$ and $\Q_{\ell_{2,1}}(\bu,\bnu)= \frac{1}{n} \Vert \bm{Y}-\bm{X} (\bu \odotg \bnu) \Vert_2^2 + \frac{\lambda}{2} (\Vert \bu \Vert_2^2 + \Vert \bnu \Vert_2^2)$ for $\bu,\bv \in \mathbb{R}^{100}$ and $\bnu \in \mathbb{R}^{20}$.\\ 
Figure~\ref{fig:comparison-direct-gd-sparsity} shows the failure of direct GD to achieve parameter (group) sparsity. In contrast, applying GD to the equivalent smooth objective $\Q$ matches the regularization paths of the specialized optimizers, providing numerical evidence that by optimizing the equivalent surrogate, the non-smooth base problem can be solved exactly using fully differentiable standard GD. Figure~\ref{fig:comparison-direct-gd-norms} further plots the parameter norms as a function of $\lambda$, complementing previous findings. For direct GD, the weight norm even starts to increase for large values of $\lambda$, raising serious concerns about the actual effect achieved by direct GD optimization for $\ell_1$ regularized DNNs \citep[e.g.,][]{han2015learning,wen2016structuredsparsity,liu2017learning}.
\begin{figure}[t!]
  \centering
    \subfloat[HPP vs direct GD for lasso objective]{%
        \raisebox{0.0cm}{\includegraphics[width=0.42\textwidth]{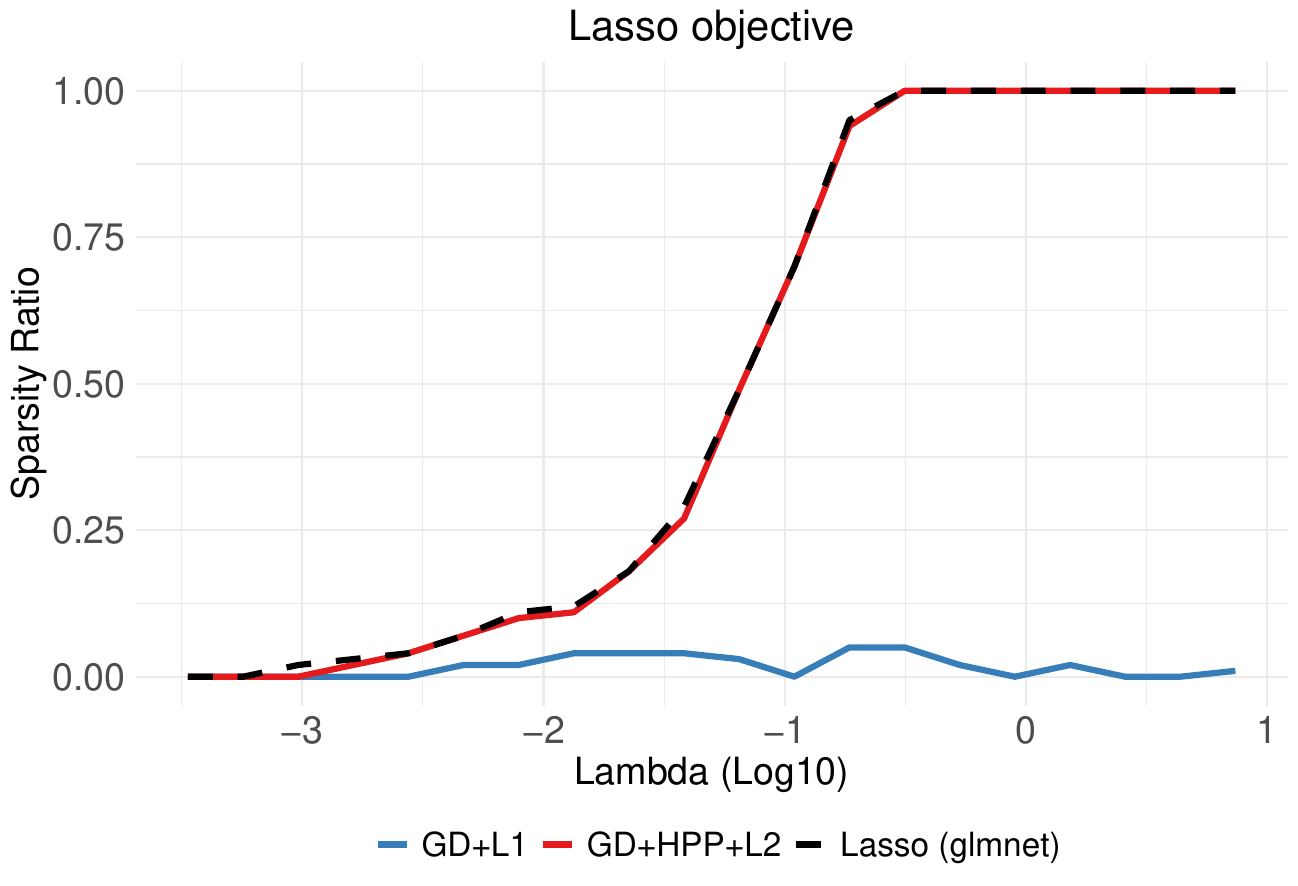}}%
        \label{fig:hpp-vs-gd-sparsity}%
        }%
    \hspace{0.01\textwidth}
    \subfloat[GHPP vs direct GD for group lasso]{%
        \raisebox{0.0cm}{\includegraphics[width=0.42\textwidth]{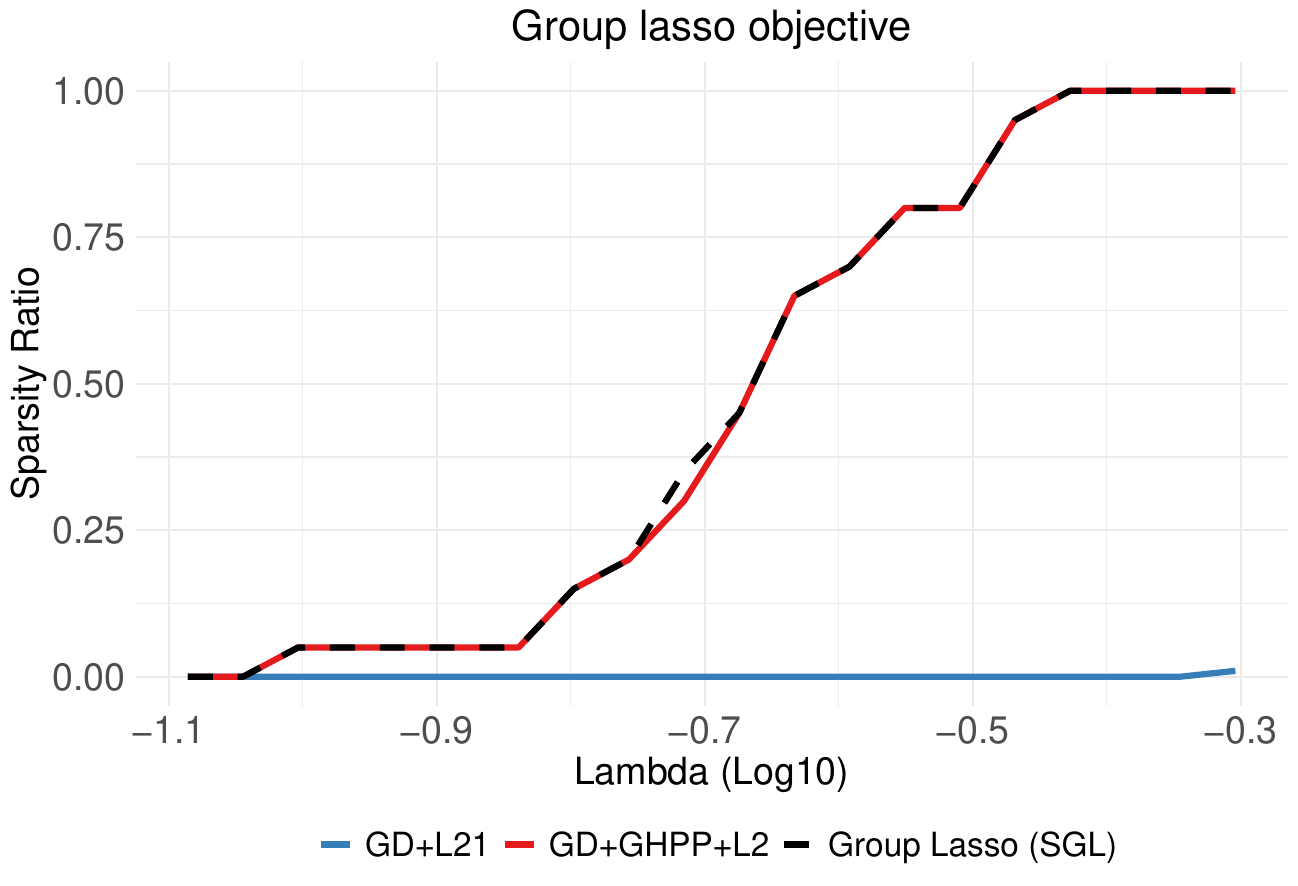}}%
        \label{fig:ghpp-vs-gd-sparsity}%
        }%
  \caption[Comparison with direct (Sub)GD optimization]{\small Comparison of regularization paths of (G)HPP-based GD and direct (Sub)GD optimization of the non-smooth $\ell_1$ regularized lasso (\textbf{a}) and $\ell_{2,1}$ regularized group lasso (\textbf{b}) objectives. Dashed lines indicate (optimal) solutions of the non-smooth optimizer. Parameters (groups) with magnitude ($\ell_2$ norm) below $1 \times 10^{-6}$ are considered $0$.%
  }
  \label{fig:comparison-direct-gd-sparsity}
  \vspace{-0.3cm}
\end{figure}
\subsection{Comparison with Convex and Non-Convex Regularizers}\label{subsec:comparison-linmod}

Next, we investigate the behavior of our smooth optimization method for $\ell_q$ regularization under SGD in a high-dimensional ($d>n$) sparse linear regression simulation setting, comparing against widely-used convex and non-convex regularizers.
The $\ell_q$ regularized sparse linear regression problem we consider is defined as $\P(\bbeta)= \frac{1}{n} \Vert \bm{Y}-\bm{X}\bbeta \Vert_2^2 + \lambda \Vert \bbeta \Vert_{2/k}^{2/k}$. %
Smooth optimization of this objective is achieved by overparametrization of $\bbeta$ using the $\text{HPP}_{k}$ for factorization depths $k \in \{2,3,4,6\}$. Combined with $\ell_2$ regularization of the surrogate parameters, equivalent smooth surrogates for SGD optimization are given by 
$\Q(\bu_1,\ldots,\bu_k) = {\textstyle \frac{1}{n} } \Vert \bm{Y} - \bm{X} \bu_{l}^{\odot k}\Vert_2^2 + {\textstyle \frac{\lambda}{k}  \sum_{l=1}^{k} \norm{\bu_l}_2^2}$. 
We compare our models against widely used implementations of convex $\ell_1$ and non-convex SCAD and MCP regularizers, as well as an oracle model that is obtained as the least squares estimator using only the true informative features. All models are evaluated with respect to their standardized estimation error $\Vert \hbbeta - \bbeta^{\ast}\Vert_2^2/\Vert\bbeta^{\ast}\Vert_2^2$, as well as their test root mean squared error {\small$\sqrt{n^{-1} \Vert \bm{Y} - \hat{\bm{Y}} \Vert_2^2}$} (RMSE).
\begin{figure}[ht]
\centering
\includegraphics[width=0.75\textwidth]{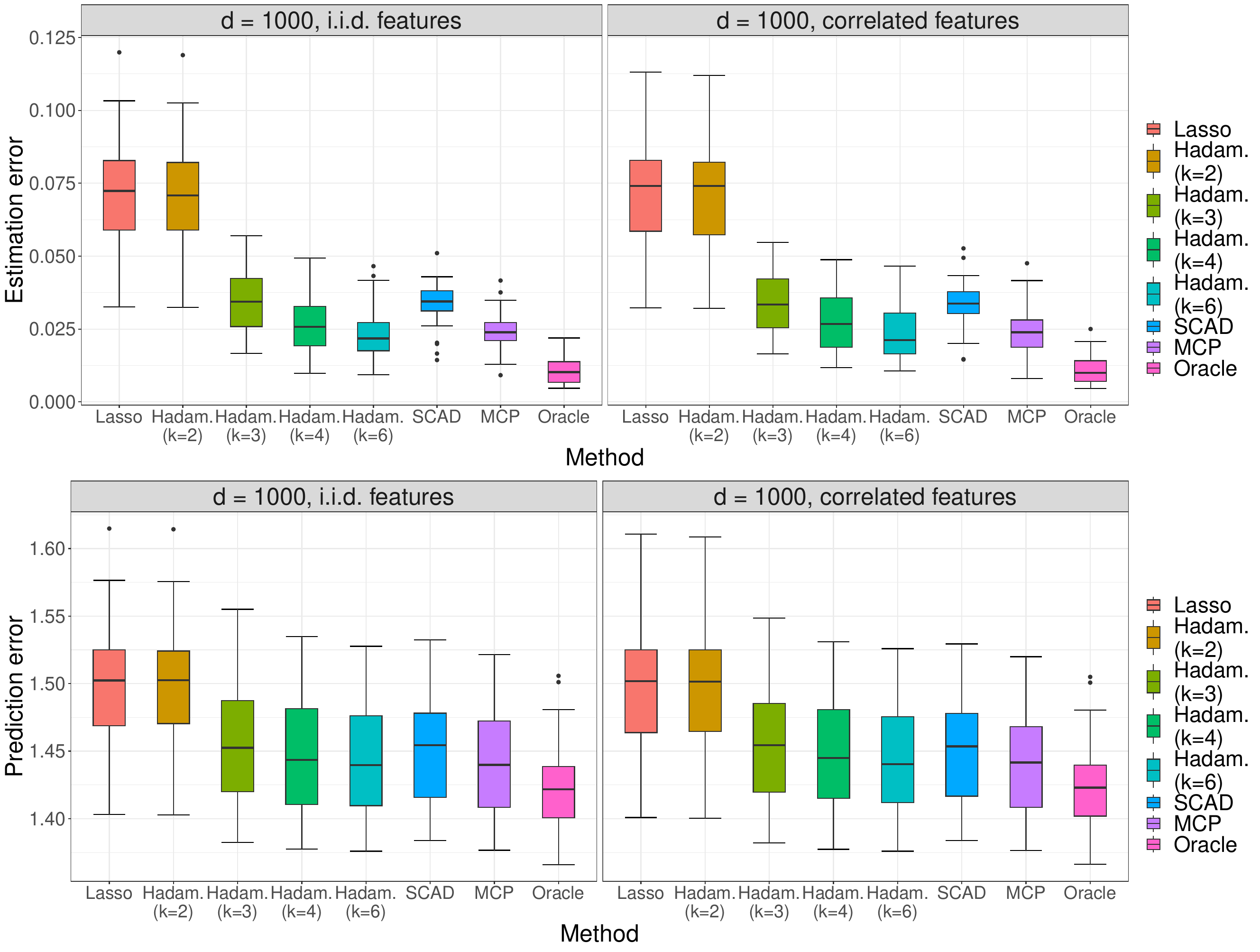}
\caption[Numerical Experiments (Estimation and Prediction)]{\small Stand. estimation error (top row) and test prediction error (bottom row) for two $\bm{\Sigma}$ settings (columns) of our approach for depths $k\in\{2,3,4,6\}$, compared with specialized optimizers for $\ell_1$ and non-convex SCAD and MCP penalties.}
\label{fig:simulation-estpred}
\vspace{-0.3cm}
\end{figure}
%

\noindent Figure~\ref{fig:simulation-estpred} shows the distribution of estimation and test prediction errors over 30 simulation runs. The results indicate that the performance of our differentiable method for $\ell_q$ regularization 
improves monotonically with the factorization depth $k$, outperforming $\ell_1$ regularization for $k>2$, and surpassing or matching both SCAD and MCP. These results are noteworthy considering the use of vanilla SGD without tuning. Comparing the performance of the Hadamard parametrized model of depth $k=2$ and the standard implementation of the lasso, we find virtually identical results, empirically validating our theoretical results.\\
Besides estimation and prediction error, the support recovery of our approach is also of interest for variable selection. In line with previous findings, we demonstrate empirically that deeper factorizations improve support recovery. Appendix~\ref{app:experiments} contains the corresponding results, 
as well as additional experiments for both a low-dimensional ($d<n$) setting and varying sparsity of the ground-truth parameter, whose findings are consistent with previous results.
\vspace{-0.1cm}

\subsection{Unstructured Sparsity: Enhanced DNN Pruning}\label{subsec:lenet-pruning}

In this application, we demonstrate how one-shot pruning of DNNs can be enhanced with differentiable (non-convex) sparse regularization using the $\text{HPP}_k$. Pruning \citep{lecun1989optimal} is the dominant sparsification technique for DNNs \citep{hoefler2021sparsity} and selectively removes components according to some saliency criterion, typically chosen to be the weight magnitude. Our method, as any sparse regularizer, can be easily combined with other sparsification schemes, e.g., by additionally applying global magnitude pruning \citep{blalock2020state} after training the overparametrized sparse network.

\begin{figure}[ht]
  \centering
  \includegraphics[width=0.83\textwidth]{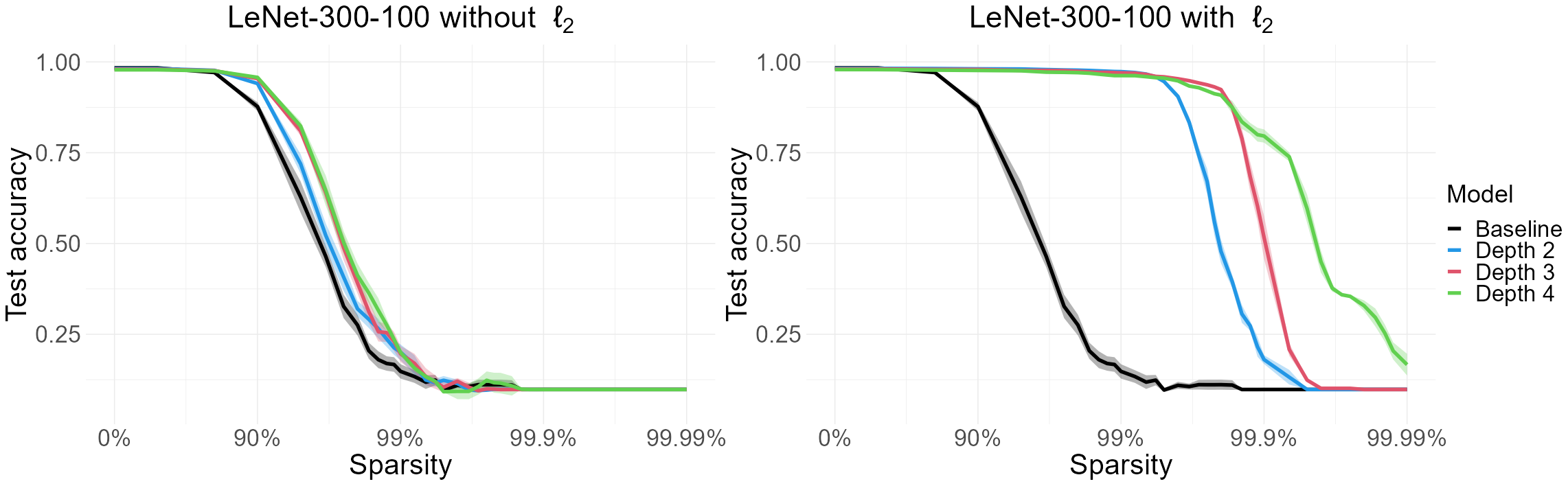}
  \caption[One-shot pruning curves for LeNet-300-100 on MNIST]{\small One-shot pruning curves obtained by overparametrizing the weights and biases of a LeNet-300-100 trained on MNIST using the $\text{HPP}_k$. \textbf{Left}: results for unregularized models. \textbf{Right}: adding smooth $\ell_2$ regularization perhaps counterintuitively produces profound sparsity-inducing effects. 
  Magnitude-based pruning constitutes the baseline and the error bars show standard errors over five random initializations.
  }
  \label{fig:lenet-hppk}
  \vspace{-0.4cm}
\end{figure}

\noindent To evaluate this approach, we train a LeNet-300-100 on the MNIST image classification task \citep{deng2012mnist} using Adam.  
The fully connected network has two hidden layers with 300 and 100 units and ReLU activation. We apply the $\text{HPP}_k$ to all $266,610$ weights and biases for depths $k \in \{2,3,4\}$. After training, the Hadamard factors are collapsed and the reconstructed model is further pruned to desired sparsity levels without finetuning. Figure~\ref{fig:lenet-hppk} (left) shows the pruning curves for $\lambda=0$ and different depths $k$. The plot reveals that factorizing the parameters without surrogate $\ell_2$ regularization already improves the pruning performance, in line with the arguments provided for the mechanism of Powerpropagation \citep{schwarz2021powerpropagation}. This is surprising since the model expressivity has not changed, highlighting important trajectory-dependent effects. The right plot is with active $\ell_2$ regularization, inducing sparse $\ell_{2/k}$ regularization according to our theory. The Pareto curves are taken as the best performance over a grid of $\lambda$ values for each sparsity level. The results show drastic improvements over both the baseline (magnitude pruning) and the unregularized overparametrization, with induced non-convex regularization ($k>2$) further outperforming induced $\ell_1$ sparsity. At a fixed accuracy of $75\%$, magnitude pruning still uses $\approx 24,000$ param., while the models for $k=2,4$ require only $\approx 1,800$ and $230$ parameters, respectively. Similarly, at a fixed sparsity of $99.9\%$, the model performance for $k=2$ almost degrades to random guessing, while the depth $4$ model retains $>80\%$ test accuracy. 
\vspace{-0.1cm} 

\subsection{Structured Sparsity: Filter-Sparse CNNs}

The next experiment applies the structured Hadamard power parametrization from Section~\ref{sec:had-group-powers} to a small VGG-style CNN to obtain filter sparsity. The network has a total of $99,178$ parameters of which $64,800$ are filter weights. Although structured sparsity in DL generally leads to poorer performance-sparsity trade-offs than unstructured sparsity, its capacity to jointly remove whole model components permits a much greater reduction in computational footprint and is thus of particular interest for practical applications. Writing the regularized CNN training objective for filter sparsity as $\P(\bm{\psi}, \bbeta)=\L(\bpsi,\bbeta)+\lambda \Vert \bbeta \Vert_{2,2/k}^{2/k}$, all biases and the weights of fully-connected layers are contained in $\bpsi$ while $\bbeta$ comprises the grouped filter weights of the convolutional layers. Applying the $\text{GHPowP}_{k}$ as defined in (\ref{eq:reparam-had-group-powers}) to $\bbeta$, the equivalent differentiable objective reads $\Q(\bpsi, \bu,\bnu) =  \L(\bpsi,\bu \odotg |\bnu|^{\circ(k-1)}) + \frac{\lambda}{k} \sum_{j=1}^{L} (\Vert \bm{u}_{{j}} \Vert_{2}^{2} + (k-1)  \nu_{j}^{2})$, where $L$ is the total number of filters. Effectively, the weights of each filter are multiplied by a shared scalar $|\nu_j|^{k-1}$, inducing the group structure. Note that by using a structured Hadamard power parametrization, only one additional parameter per filter is introduced for any factorization depth $k$, resulting in minimal overparametrization ($99,370$ parameters). Figure~\ref{fig:vgg-ghpowp} shows the regularization path for the overparametrized CNNs trained on MNIST using real-valued depths $k \in \{2,2.5,3\}$. The models are trained using SGD without any post-hoc pruning and compared to (structured) magnitude pruning of the original CNN based on the $\ell_2$ norm of the filter weights. The results show a $>90\%$ filter reduction at a negligible drop in accuracy, with deeper factorizations allowing for slightly higher sparsity. In comparison, structured magnitude pruning already starts degrading sharply at $50\%$ sparsity.
\begin{figure}[t]
  \centering
  \includegraphics[width=0.78\textwidth]{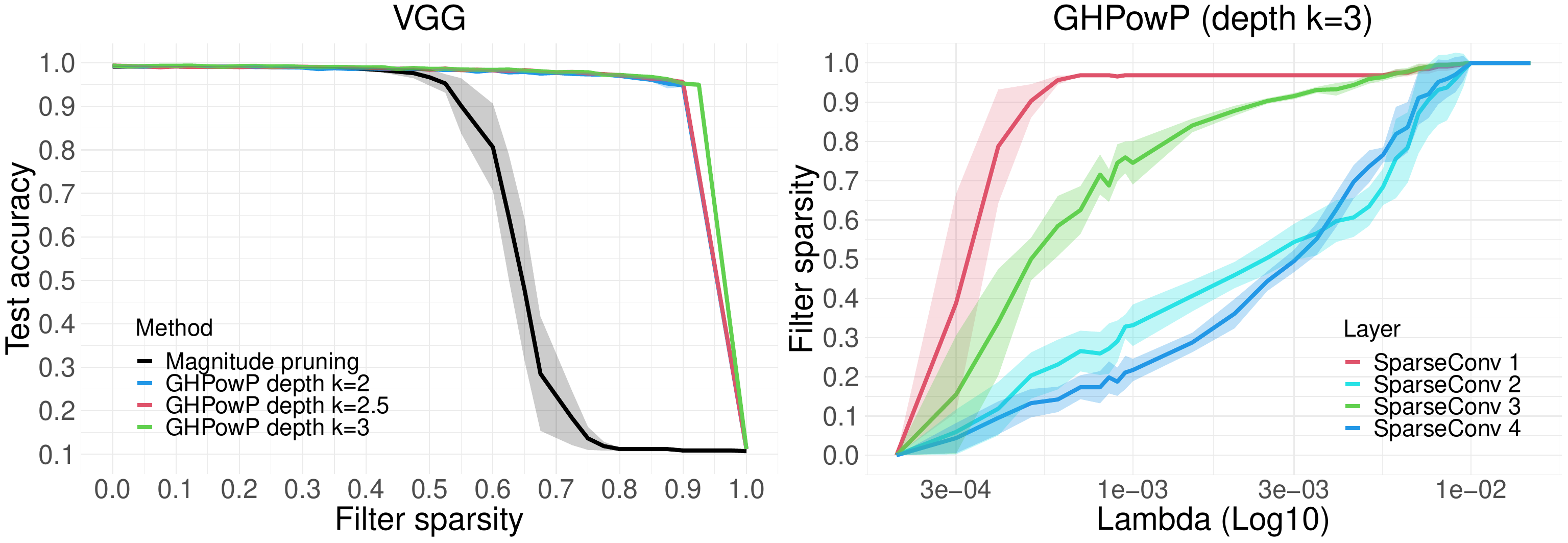}
  \caption[Structured Sparsity using GHPowP]{\small \textbf{Left}: regularization paths for (structured) filter sparsity using the $\text{GHPowP}_k$ for $k \in \{2, 2.5, 3\}$ to overparametrize the filter weights of a small VGG architecture trained on MNIST. Structured magnitude pruning based on filter norms constitutes the baseline. \textbf{Right}: layer-wise sparsity patterns for the $\text{GHPowP}_3$. Error bars show standard errors over ten random initializations.
  }
  \label{fig:vgg-ghpowp}
  \vspace{-0.3cm}
\end{figure}

\subsection{Computational Complexity}

An important question is how the overparametrization in our method affects the runtime complexity of DNN training using SGD. Since the networks are reduced to their base parametrization after training and sparse components are removed, the inference time complexity is reduced by the extent of the achieved sparsity. During training, the overparametrization increases both model size and computational complexity, which is heavily dependent on the architecture, hardware, and specific choice of $\K$.
\begin{figure}[b!]
  \centering
  \includegraphics[width=0.8\textwidth]{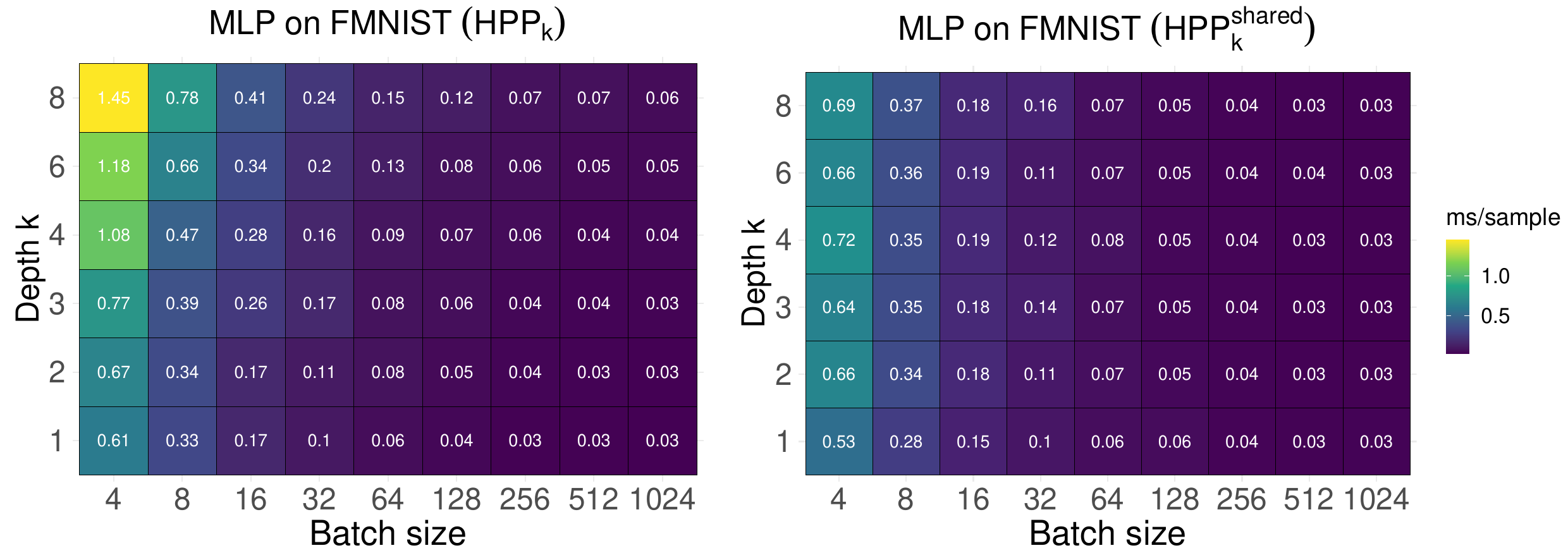}
  \caption[Time per sample for MLP on FMNIST]{\small Time per sample (training) for different factorization depths and batch sizes. \textbf{Left}: full overparametrization using the $\text{HPP}_k$. \textbf{Right}: parameter sharing significantly reduces computational overhead. Averages over four epochs are displayed.
  }
\label{fig:complexity-fmnist}
  \vspace{-0.3cm}
\end{figure}
To evaluate the impact of our approach, we train a fully-connected ReLU network with four hidden layers on the Fashion MNIST data set \citep{xiao2017fashion}. Figure~\ref{fig:complexity-fmnist} reports the mean wall-clock training time per sample for different batch sizes and factorization depths $k$ of both the $\text{HPP}_k$ (\ref{eq:hppk}) and its parameter-sharing counterpart the $\text{HPP}_{k}^{shared}$ (\ref{eq:hppk-shared}). The results show that the computational overhead increases sublinearly in $k$, but with diminishing effects for larger batches. For the $\text{HPP}_k$, training time is at worst roughly tripled for $k=8$, whereas parameter sharing affords significant improvements over the full $\text{HPP}_k$: for batch sizes $\geq 64$, there is no discernible increase in training time for the tested depth levels. Details on architecture, hardware, and additional results for a ResNet are provided in Appendix~\ref{app:subsec:computational-complexity}.

\section{Summary and Discussion}\label{sec:discussion}

In this work, we propose a general framework for smooth optimization of objectives that involve non-smooth and potentially non-convex sparse regularization of parameter subsets. Being model- and loss-agnostic, our approach is applicable to a wide range of scenarios. The key idea underlying our method is to find a smooth variational form of the non-smooth sparse regularizer. Applying a smooth parametrization map and a change of regularizers enables the construction of an equivalent smooth surrogate objective, eliminating the need for specialized optimization routines for non-smooth and non-convex problems. Moreover, our framework can be easily integrated into existing differentiable structures such as DNNs. 
Our general template is applied to the smooth optimization of a broad range of non-smooth $\ell_q$ and $\ell_{p,q}$ regularized optimization problems for (structured) sparsity. Numerical experiments demonstrate the practical feasibility and effectiveness of our method in various sparse learning problems and in comparison with other methods.

\noindent Our approach also presents certain limitations that merit discussion. One limitation pertains to the initialization of the surrogate parameters, where an optimal choice is not straightforward. 
In addition, while our approach enables efficient optimization using SGD, obtaining (numerically) exact zeros is not guaranteed for small $\lambda$. This is a characteristic of SGD and not a limitation of the optimization transfer \textit{per se}. 
For variable selection, we recommend a post-thresholding step.
It is worth emphasizing that these challenges do not inherently limit the potential of our approach; instead, they underline key areas where additional research is needed.

\noindent There are several promising avenues for future research. Notably, our approach offers the flexibility to construct reparametrized, sparse ``drop-in" replacements for network components, allowing for modular sparse regularization in differentiable network structures. This makes our method especially well-suited for exploring applications across various domains, such as input-sparse DNNs.
Although a heuristic initialization performed well in our experiments, there is further great interest in understanding how to construct initialization schemes tailored to the surrogate parameters. %
There is also an opportunity to investigate the relationship between our smooth optimization transfer approach and implicit regularization methods in the DL literature. Our approach enforces a balanced parameter norm condition through surrogate $\ell_2$ regularization, which bears similarities to the balanced weight conditions employed in implicit regularization techniques. Investigating this relationship could reveal valuable insights and potential synergies between the two approaches. Lastly, establishing theoretical conditions for the existence of differentiable parameterizations $\K$ and surrogate regularizers $\Rxi$ inducing a given $\Rbeta$ constitutes an interesting direction for future research.

\begin{appendices}

\section{Proofs of Lemmas and Theorems}\label{app:missing-proofs}

\subsection{Proof of Lemma~\ref{lemma:p_to_pk}}\label{app:p_to_pk}
\begin{proof}
    Assume $(\hat{\bpsi},\hat{\bbeta})$ is a local minimizer of $\P(\bpsi,\bbeta)$, then $\exists\, \varepsilon>0: \forall (\bpsi',\bbeta')\in\mathcal{B}((\hat{\bpsi},\hat{\bbeta}),\varepsilon):P(\hat{\bpsi},\hat{\bbeta})\leq \P(\bpsi',\bbeta')$.
    Since $\K(\bxi)$ is a continuous surjection, so is $\tilde{\K}(\bpsi,\bxi)\triangleq (\bpsi,\K(\bxi))$. Pick any $(\hat{\bpsi},\hat{\bxi})\in\tilde{\K}^{-1}(\hbpsi,\hbbeta)=\{\hbpsi\}\times\{\hbxi:\K(\hbxi)=\hbbeta\}$. By continuity of $\tilde{\K}$, there $\exists \delta>0:\tilde{\K}(\mathcal{B}((\hbpsi,\hbxi),\delta))\subseteq \mathcal{B}(\tilde{\K}(\hbpsi,\hbxi),\varepsilon) = \mathcal{B}((\hbpsi,\hbbeta),\varepsilon)$. This means $\forall\,(\bpsi',\bxi') \in \mathcal{B}((\hbpsi,\hbxi),\delta):\,(\bpsi',\K(\bxi'))=(\bpsi',\bbeta')\in\mathcal{B}((\hbpsi,\hbbeta),\varepsilon)$. Since by assumption, $P(\hat{\bpsi},\hat{\bbeta})\leq \P(\bpsi',\bbeta')$ for all $(\bpsi',\bbeta')\in\mathcal{B}((\hat{\bpsi},\hat{\bbeta}),\varepsilon)$, and by continuity all $(\bpsi',\bxi') \in \mathcal{B}((\hbpsi,\hbxi),\delta)$ map to some $(\bpsi',\bbeta')$ in $\mathcal{B}((\hbpsi,\hbbeta),\varepsilon)$ under $\tilde{\K}$, we conclude that 
    {\small
    $$\forall\, (\bpsi',\bxi')\in\mathcal{B}((\hbpsi,\hbxi),\delta):\P(\tilde{\K}(\hbpsi,\hbxi)) = \P(\hbpsi,\K(\hbxi)) = \P(\hbpsi,\hbbeta) \leq \P(\bpsi',\bbeta') = \P(\bpsi',\K(\bxi')) \,.$$
    Therefore, if $(\hat{\bpsi},\hat{\bbeta})$ is a local minimizer of $\P(\bpsi,\bbeta)$, then all $(\hat{\bpsi},\hat{\bxi})$ in the fiber $\tilde{\K}^{-1}(\hbpsi,\hbbeta)$ are local minimizers of $\P(\bpsi,\K(\bxi))$ with equivalent local minima $\P(\hat{\bpsi},\hat{\bbeta})=\P(\hat{\bpsi},\K(\hat{\bxi}))$}.
\end{proof}
\vspace{-0.0cm}
\subsection{Proof of Lemma~\ref{lemma:pk_to_p}}\label{app:pk_to_p}
\begin{proof}
    Assume $(\hat{\bpsi},\hbxi)$ is a local minimizer of $\P(\bpsi,\K(\bxi))$, then $\exists\, \varepsilon>0: \forall (\bpsi',\bxi')\in\mathcal{B}((\hbpsi,\hbxi),\varepsilon):\P(\hbpsi,\K(\hbxi))\leq \P(\bpsi',\K(\bxi'))$.
    Since $\K(\bxi)$ is locally open at $\hbxi$, so is $\tilde{\K}(\bpsi,\bxi)\triangleq (\bpsi,\K(\bxi))$ at $(\hbpsi,\hbxi)$. By local openness, we can find $\delta>0$ such that $\mathcal{B}(\tilde{\K}(\hbpsi,\hbxi),\delta) \subseteq \tilde{\K}(\mathcal{B}((\hbpsi,\hbxi),\varepsilon))$. Thus, $\forall\,(\bpsi',\bbeta')\in \mathcal{B}(\tilde{\K}(\hbpsi,\hbxi),\delta)\,\,\exists\,\,(\bpsi',\bxi')\in\mathcal{B}((\hbpsi,\hbxi),\varepsilon)$ such that $(\bpsi',\K(\bxi'))$ $=(\bpsi',\bbeta')$. But since we have by assumption that $\forall\,(\bpsi',\bxi')\in\mathcal{B}((\hbpsi,\hbxi),\varepsilon):\P(\hbpsi,\K(\hbxi))=\P(\hbpsi,\hbbeta)\leq \P(\bpsi',\K(\bxi'))$, and we established $\forall (\bpsi',\bbeta')\in \mathcal{B}((\hbpsi,\hbbeta),\delta)\,\exists\,(\bpsi',\bxi')\in\mathcal{B}((\hbpsi,\hbxi),\varepsilon):(\bpsi',\bbeta')=(\bpsi',\K(\bxi'))$, it follows 
    $$\forall\,(\bpsi',\bbeta')\in \mathcal{B}((\hbpsi,\hbbeta),\delta):\P(\hbpsi,\hbbeta)=\P(\hbpsi,\K(\hbxi)) \leq \P(\bpsi',\K(\bxi'))=\P(\bpsi',\bbeta')\,.$$ 
    Thus, $(\hbpsi,\hbbeta)=(\hbpsi,\K(\hbxi))$ is a local minimizer of $\P(\bpsi,\bbeta)$ with corresponding local minimum $\P(\hat{\bpsi},\hat{\bbeta})=\P(\hat{\bpsi},\K(\hat{\bxi}))$.
\end{proof}
\vspace{-0.0cm}

\subsection{Proof of Lemma~\ref{lemma:p_to_q}}\label{app:p_to_q}
\begin{proof}
    Assume $(\hat{\bpsi},\hat{\bbeta})$ is a local minimizer of $\P(\bpsi,\bbeta)$, then $\exists\, \varepsilon>0: \P(\hat{\bpsi},\hat{\bbeta})\leq \P(\bpsi',\bbeta')$ $ \forall \,(\bpsi',\bbeta')\in\mathcal{B}((\hat{\bpsi},\hat{\bbeta}),\varepsilon)$. Since $\K(\bxi)$ is a continuous surjection, so is $\tilde{\K}(\bpsi,\bxi)\triangleq (\bpsi,\K(\bxi))$. 
    By assumption of the variational form in Assumption~\ref{ass:min}, $\exists \hbxi \in \argmin \Rxi(\bxi)_{\bxi:\K(\bxi)=\hbbeta} \subseteq \{\bxi: \K(\bxi)=\hbbeta\}$ so that $\Rxi(\hbxi)=\Rbeta(\K(\hbxi))=\Rbeta(\hbbeta)$, and therefore also $\P(\hbpsi,\hbbeta)=\L(\hbpsi,\hbbeta)+\lambda \Rbeta(\hbbeta)=\L(\hbpsi,\K(\hbxi))+\lambda\Rbeta(\K(\hbxi))=\L(\hbpsi,\K(\hbxi))+\lambda \Rxi(\hbxi)=\Q(\hbpsi,\hbxi)$.\\
    By continuity of $\tilde{\K}$, there $\exists \delta>0:\tilde{\K}(\mathcal{B}((\hbpsi,\hbxi),\delta))\subseteq \mathcal{B}(\tilde{\K}(\hbpsi,\hbxi),\varepsilon) = \mathcal{B}((\hbpsi,\hbbeta),\varepsilon)$. This means $\forall\,(\bpsi',\bxi') \in \mathcal{B}((\hbpsi,\hbxi),\delta):\,(\bpsi',\K(\bxi'))=(\bpsi',\bbeta')\in\mathcal{B}((\hbpsi,\hbbeta),\varepsilon)$. Because $(\hbpsi,\hbbeta)$ is a local minimizer of $\P(\bpsi,\bbeta)$, $\P(\hat{\bpsi},\hat{\bbeta})\leq \P(\bpsi',\bbeta')$ for all $(\bpsi',\bbeta')\in\mathcal{B}((\hat{\bpsi},\hat{\bbeta}),\varepsilon)$, and by continuity of $\tilde{\K}$, all $(\bpsi',\bxi') \in \mathcal{B}((\hbpsi,\hbxi),\delta)$ map to some $(\bpsi',\bbeta')$ in $\mathcal{B}((\hbpsi,\hbbeta),\varepsilon)$. Then we can conclude $\P(\hbpsi,\K(\hbxi))\leq\P(\bpsi',\K(\bxi'))$ for all $(\bpsi',\bxi') \in \mathcal{B}((\hbpsi,\hbxi),\delta)$. Lastly, using the majorization property of the surrogate penalty, $\Rxi(\bxi)\geq\Rbeta(\K(\bxi))\,\forall\,\bxi$, we obtain the following chain of inequalities:
    \begin{equation*}
        \forall\, (\bpsi',\bxi') \in \mathcal{B}((\hbpsi,\hbxi),\delta):\,\, \Q(\hbpsi,\hbxi)=\P(\hbpsi,\K(\hbxi)) \leq \P(\bpsi',\K(\bxi')) \leq \Q(\bpsi',\bxi').
    \end{equation*}
    \noindent Thus, $(\hbpsi,\hbxi)$ is a local minimizer of $\Q(\bpsi,\bxi)$. Therefore, if $(\hbpsi,\hbbeta)$ is a local minimizer of $\P(\bpsi,\bbeta)$, then all $(\hbpsi,\hbxi)$ such that $\hbxi \in \argmin_{\bxi:\K(\bxi)=\hbbeta} \Rxi(\bxi)$ are local minimizers of $\Q(\bpsi,\bxi)$ with $\Q(\hbpsi,\hbxi)=\P(\hbpsi,\hbbeta)$.
%
\end{proof}
\vspace{-0.0cm}
\subsection{Proof of Lemma~\ref{lemma:q_to_p}}\label{app:q_to_p}

\begin{proof}

Assume $(\hbpsi,\hbxi)$ is a local minimizer of $\Q(\bpsi,\bxi)$, then $\exists\, \varepsilon_0>0$ such that $\forall (\bpsi',\bxi')\in \mathcal{B}((\hbpsi,\hbxi),\varepsilon_0):\Q(\hbpsi,\hbxi)\leq \Q(\bpsi',\bxi')$. First, we show that for each local minimizer $(\hbpsi,\hbxi)$ of $\Q$, letting $\K(\hbxi) = \hbbeta$, it must also hold that $\hbxi \in \hbxi(\hbbeta) =\argmin_{\bxi:\K(\bxi) = \hbbeta}\Rxi(\bxi)$, i.e., $\hbxi$ is a minimizer of the SVF given $\hbbeta$. Suppose for contradiction that $\hbxi$ is not a local minimizer of $\Rxi(\bxi)$ over the fiber $\K^{-1}(\hbbeta)$. Then $\forall\,\varepsilon>0 \, \exists \,\tbxi \in \mathcal{B}(\hbxi,\varepsilon)$ such that $\K(\tbxi)=\hbbeta$ and $\Rxi(\tbxi) < \Rxi(\hbxi)$. Let $\varepsilon = \varepsilon_0$. Because $\K(\tbxi) = \K(\hbxi)$, and the loss $\L(\bpsi,\K(\bxi))$ is constant over all $\bxi \in \K^{-1}(\hbbeta)$, we have $\L(\hbpsi,\K(\tbxi)) = \L(\hbpsi,\hbbeta)$. But then $\Q(\hbpsi,\tbxi) = \L(\hbpsi,\hbbeta) + \lambda \Rxi(\tbxi) < \L(\hbpsi,\hbbeta) + \lambda \Rxi(\hbxi) = \Q(\hbpsi,\hbxi)$,
with $(\hbpsi,\tbxi) \in \mathcal{B}((\hbpsi,\hbxi),\varepsilon_0)$, contradicting local minimality of $(\hbpsi,\hbxi)$. Thus if $(\hbpsi,\hbxi)$ is a local minimizer of $\Q(\bpsi,\bxi)$, then, as all local minima of the SVF are global, $\hbxi \in \argmin_{\bxi:\K(\bxi)=\hbbeta}\Rxi(\bxi)$, with $\Rxi(\hbxi) = \Rbeta(\hbbeta)$, and so $\Q(\hbpsi,\hbxi) = \P(\hbpsi,\hbbeta)$.

\noindent Using this result, we now proceed to prove that $(\hbpsi,\hbbeta)$ is a local minimizer of $\P(\bpsi,\bbeta)$ by contradiction. Suppose $(\hbpsi,\hbbeta)$ is not a local minimizer of $\P(\bpsi,\bbeta)$, then $\forall \delta>0\; \exists (\tbpsi,\tbbeta) \in \mathcal{B}((\hbpsi,\hbbeta),\delta):\; \P(\tbpsi,\tbbeta)<\P(\hbpsi,\hbbeta)$. 
By Assumption~\ref{ass:min}, the set-valued solution map $\hbxi(\bbeta)$ is lower hemicontinuous at $\hbbeta$, and Lemma~\ref{lemma:lhc-product-augmented} extends this property to the identity-augmented product map $g:(\bpsi, \bbeta) \mapsto \{ (\bpsi, \bxi) : \bxi \in \hbxi(\bbeta) \}$. Since we know $\hbxi \in \hbxi(\hbbeta)$, it follows from lower hemicontinuity of $g(\bpsi,\bbeta)$ at $(\hbpsi,\hbbeta)$ that for all $\varepsilon>0$ there is $\delta>0$ so that:

\begin{equation*}
\forall \,\, (\bpsi', \bbeta') \in \B((\hbpsi, \hbbeta), \delta) \, \exists \,(\bpsi', \bxi') \in g(\bpsi', \bbeta') \cap \B((\hbpsi, \hbxi), \varepsilon). 
\end{equation*}

Letting $\varepsilon = \varepsilon_0$ and $(\bpsi',\bbeta')=(\tbpsi,\tbbeta)$, this implies there is also $(\tbpsi,\tbxi) \in \B((\hbpsi, \hbxi), \varepsilon_0)$ with $\tbxi \in \hbxi(\tbbeta)$. As $\tbxi$ is a minimizer of $\Rxi(\bxi)$ over the fiber $\K^{-1}(\tbbeta)$, we have $\Rxi(\tbxi)=\Rbeta(\tbbeta)$ and thus $\P(\tbpsi,\tbbeta)=Q(\tbpsi,\tbxi)$. But then we have found $(\tbpsi,\tbxi)\in\mathcal{B}((\hbpsi,\hbxi),\varepsilon_0)$ such that $\Q(\tbpsi,\tbxi)=\P(\tbpsi,\tbbeta)<\P(\hbpsi,\hbbeta)=\Q(\hbpsi,\hbxi)\,,$
contradicting that $(\hbpsi,\hbxi)$ is a local minimizer of $\Q$. This shows that if $(\hbpsi,\hbxi)$ is a local minimizer of $\Q(\bpsi,\bxi)$, then $(\hbpsi,\hbbeta)=(\hbpsi,\K(\hbxi))$ is a local minimizer of $\P(\bpsi,\bbeta)$ with $\Q(\hat{\bpsi},\hat{\bxi})=\P(\hat{\bpsi},\hbbeta)$.

\end{proof}

%

\begin{lemma}[Lower hemicontinuity of product-augmented maps] \label{lemma:lhc-product-augmented}

Let $\hbxi : \Rd \rightrightarrows \Rdxi$ be lower hemicontinuous at $\hbbeta \in \Rd$. Augmenting $\hbxi(\bbeta)$ by the identity function $\operatorname{id}_{\bpsi}: \mathbb{R}^{d_{\bpsi}} \to \mathbb{R}^{d_{\bpsi}},\bpsi \mapsto \bpsi$, we can define the Cartesian product map $g : \mathbb{R}^{d_{\bpsi}} \times \Rd \rightrightarrows \mathbb{R}^{d_{\bpsi}} \times \Rdxi, \, (\bpsi, \bbeta) \mapsto g(\bpsi, \bbeta) \triangleq \{ (\bpsi, \bxi) : \bxi \in \hbxi(\bbeta) \}$. Then $g$ is lower hemicontinuous at $(\hbpsi, \hbbeta) \in \mathbb{R}^{d_{\bpsi}} \times \Rd$ for any $\hbpsi \in \mathbb{R}^{d_{\bpsi}}$.
\end{lemma}

\begin{proof}
Let $\varepsilon > 0$ and let $(\hbpsi, \hbxi) \in g(\hbpsi, \hbbeta)$, i.e., $\hbxi \in \hbxi(\hbbeta)$, be arbitrary. Since $\hbxi(\bbeta)$ is lower hemicontinuous at $\hbbeta$, there is $\delta_1 > 0$ such that for all $\tbbeta \in \B(\hbbeta, \delta_1)$ and all $\bxi' \in \hbxi(\hbbeta)$, there exist $\tbxi \in \hbxi(\tbbeta)$ with $\tbxi \in \B(\bxi',\varepsilon/\sqrt{2})$, 
in particular for $\hbxi \in \hbxi(\hbbeta)$. Set $\delta \triangleq \min(\delta_1, \varepsilon/\sqrt{2})$. Let $(\tbpsi, \tbbeta) \in \B((\hbpsi, \hbbeta), \delta)$. Then $\|\tbpsi - \hbpsi\|_2^2 + \|\tbbeta - \hbbeta\|_2^2 < \delta^2 \leq \frac{\varepsilon^2}{2}$, 
so in particular, $\|\tbpsi - \hbpsi\|_2 < \delta \leq \varepsilon/\sqrt{2}$, and $\tbbeta \in \B(\hbbeta, \delta_1)$. Hence there exists $\tbxi \in \hbxi(\tbbeta)$ with $\|\tbxi - \hbxi\|_2 < \varepsilon/\sqrt{2}$. Combining these, we obtain
\begin{equation*}
\|(\tbpsi, \tbxi) - (\hbpsi, \hbxi)\|_2^2 = \|\tbpsi - \hbpsi\|_2^2 + \|\tbxi - \hbxi\|_2^2 < \frac{\varepsilon^2}{2} + \frac{\varepsilon^2}{2} = \varepsilon^2,
\end{equation*}
so that $(\tbpsi, \tbxi) \in \B((\hbpsi, \hbxi), \varepsilon)$. Since $(\tbpsi, \tbxi) \in g(\tbpsi, \tbbeta)$, this shows that for all $(\tbpsi, \tbbeta) \in \B((\hbpsi, \hbbeta), \delta)$, there exists $(\tbpsi, \tbxi) \in g(\tbpsi, \tbbeta) \cap \B((\hbpsi, \hbxi), \varepsilon)$. Since $\hbpsi$ and $\hbxi \in \hbxi(\hbbeta)$ were arbitrary, this shows the lower hemicontinuity of $g$ at $(\hbpsi, \hbbeta)$.
\end{proof}

\subsection{Proof of Lemma~\ref{lemma:lhc-solution-map}}\label{app:proof-lhc}

\begin{proof}

We prove lower hemicontinuity for each $\hbxi_j(\bbeta_j),\, j \in [L]$, separately, and then extend to the entire solution map $\hbxi(\bbeta)$. Our argument proceeds by (i) deriving the minimal $\mathcal{R}_{\bxi_j}(\bxi_j)$ over the fiber $\K_j^{-1}(\bbeta_j)$ given our structural assumptions, (ii) analytically constructing solution magnitudes attaining this minimum, and (iii) verifying the solution map is lower hemicontinuous.

Fix $j \in [L]$ and consider $\bbeta_j \in \mathbb{R}^{|\Gj|}$.  
As each of the $k$ factors $\bxi_{jl}$ is either a scalar or a vector in $\mathbb{R}^{|\Gj|}$, let $S_j \subseteq [k]$ denote the subset of scalar indices, $V_j = [k] \setminus S_j$ the vector indices, and let the respective sum of exponents be $k_1 = \sum_{l \in V_j} \alpha_l$ and $k_2 = \sum_{l \in S_j} \alpha_l$. 
By the power-product assumption on $\K_j$, each entry $i \in \Gj$ can be written as 
\begin{equation*}
    \K_{ji}(\bxi_j) = \underbrace{\textstyle\prod_{l=1}^k \operatorname{sign}(\xi_{jl}^{(i)}) \, |\xi_{jl}^{(i)}|^{\alpha_l}}_{\text{Assumption } \ref{ass-parametrization-map} } = \underbrace{\textstyle\prod_{l \in V_j} \operatorname{sign}(\xi_{jli}) \, |\xi_{jli}|^{\alpha_l}}_{\triangleq \beta_{ji}^V = \beta_{ji}/\beta_j^S}  \cdot \underbrace{\textstyle\prod_{l \in S_j} \operatorname{sign}(\xi_{jl}) \, |\xi_{jl}|^{\alpha_l}}_{\triangleq \beta_j^S} =
\beta_{ji},
\end{equation*}
where $\xi_{jl}^{(i)}$ equals the $i$th entry $\xi_{jli}$ of $\bxi_{jl}$ for $l \in V_j$ and $\xi_{jl}^{(i)}=\xi_{jl}$ for scalars $l \in S_j$. The partial product $\beta_{ji}^V$ written as the fraction $\beta_{ji}/\beta_j^S$ is well-defined for $\bbeta_j \neq \bm{0}$, as the product-structure implies $\xi_{jl} \neq 0\, \forall \,l \in S_j$ and hence $\beta_j^S \neq 0$. 

\textbf{(i) Minimum of \texorpdfstring{$\mathcal{R}_{\bxi_j}(\bxi_j)$ over $\K_j^{-1}(\bbeta_j)$}{surrogate regularizer over fiber}.} For $\bbeta_j = \bm{0}$, the unique norm-minimizing solution $\hbxi_j(\bm{0})$ is the singleton $\{\bm{0}\}$, yielding $\mathcal{R}_{\bxi_j}(\bm{0})=\bm{0}$. Next, consider $\bbeta_j \neq \bm{0}$. We reorder the terms in $\mathcal{R}_{\bxi_j}(\bxi_j)$ as follows:
\begin{equation*}\resizebox{0.99\textwidth}{!}{${\textstyle
\mathcal{R}_{\bxi_j}(\bxi_j) = \sum_{l=1}^k \alpha_l \| \bxi_{jl} \|_2^2 = \sum_{l \in V_j} \alpha_l \| \bxi_{jl} \|_2^2 + \sum_{l \in S_j} \alpha_l \xi_{jl} ^2  = \sum_{i \in \Gj} \sum_{l \in V_j}
\alpha_l \, \xi_{jli}^2+\sum_{l \in S_j}\alpha_l \, \xi_{jl}^2.}$}
\end{equation*}
Using the weighted AM-GM inequality (Prop.~\ref{prop-wamgm}) on the first term and inserting the parametrization constraints $\K_{ij}(\bxi_j)=\beta_{ji}=\beta_{ji}^V \cdot \beta_{j}^S$, the constrained minimum is
\scalebox{0.74}{
  \begin{minipage}{\linewidth}
  \begin{align*}
  \sum_{i \in \Gj} \sum_{l \in V_j} \alpha_l \, \xi_{jli}^2 \ge  \sum_{i \in \Gj} k_1 \big( \prod_{l \in V_j} \left(\xi_{jli}^{2}\right)^{\alpha_l} \big)^{1/k_1} = \sum_{i \in \Gj} k_1 \big(\big| \prod_{l \in V_j} \operatorname{sign}(\xi_{jli}) |\xi_{jli}|^{\alpha_l} \big|\big)^{2/k_1} = \sum_{i \in \Gj} k_1 |\beta_{ji}/\beta_j^S |^{2/k_1} = k_1 \| \bbeta_j/\beta_j^S \|_{2/k_1}^{2/k_1}\,,
  \end{align*}
  \end{minipage}
}

Applying the same inequality to the partially minimized sum and inserting $\K_j$, it holds
\scalebox{0.93}{
  \begin{minipage}{\linewidth}
\begin{align*}
k_1  \| {\bbeta_j/\beta_j^S}\|_{2/k_1}^{2/k_1}
+
\textstyle\sum_{l \in S_j}
\alpha_l \, \xi_{jl}^2
&\ge
(k_1 + k_2)
\big(
 \| {\bbeta_j/\beta_j^S}\|_{2/k_1}^{2}
\cdot \textstyle\prod_{l \in S_j}
(\xi_{jl}^2)^{\alpha_l}
\big)^{1/(k_1 + k_2)}\\
&= (k_1 + k_2)
\big(
 \| {\bbeta_j/\beta_j^S}\|_{2/k_1}
|\beta_j^S|
\big)^{2/(k_1 + k_2)}
= (k_1 + k_2)
 \| \bbeta_j\|_{2/k_1}^{2/(k_1 + k_2)}\,.
\end{align*}
\end{minipage}}
Combined, this yields the minimum value of $\mathcal{R}_{\bxi_j}(\bxi_j)$ over $\K_j^{-1}(\bbeta_j)$: 
\begin{equation*}
\mathcal{R}_{\bxi_j}(\bxi_j) =
\sum_{l=1}^k
\alpha_l \, \|\bxi_{jl}\|_2^2 \geq 
k_1  \| {\bbeta_j/\beta_j^S}\|_{2/k_1}^{2/k_1} +\sum_{l \in S_j}
\alpha_l \, \xi_{jl}^2 \geq(k_1 + k_2)
 \| \bbeta_j\|_{2/k_1}^{2/(k_1 + k_2)} ,
\end{equation*}
with equality holding if and only if both AM-GM optimality conditions are met, i.e., $|\xi_{jli}| = |\xi_{jl'i}| = |{\beta_{ji}/\beta_j^S}|^{1/k_1}
\,\, \forall l,l' \in V_j$ and all $i \in \G_j$, as well as $|\xi_{jl}| = |\xi_{jl'}| =  \|{\bbeta_j/\beta_j^S}  \|_{2/k_1}^{1/k_1} = \left\|
\bbeta_j
\right\|_{2/k_1}^{1/(k_1 + k_2)}
\, \forall \,l,l' \in S_j$, while $\K_j(\bxi_j)=\bbeta_j$.

\textbf{(ii) Construction of solutions.} The AM–GM optimality conditions and the multiplicative structure of $\K_j$ determine the absolute values of all solution parameters. Applying correct sign configurations to $\bxi_j$ respecting $\operatorname{sign}(\bbeta_j)$ under $\K_j$ then produces full solutions.
For a direct construction for any $\bbeta_j \neq \bm{0}$, define for $k_1,k_2\geq1$ the scale at which the scalar factors must balance at optimality, 
$T(\bbeta_j) \triangleq \left\| \bbeta_j \right\|_{2/k_1}^{1/(k_1 + k_2)}$. Then the absolute values of the solution parameters are given by:
$$\hbxi_j^{abs}: \mathbb{R}^{|\Gj|} \to \mathbb{R}^{d_{\bxi_j}}\,,\, \bbeta_j \mapsto |\bxi_j| = 
\begin{cases}
    |\xi_{jl}|=T(\bbeta_j)\,\,&\text{for}\,|\xi_{jl}| \in \mathbb{R}, \,l \in S_j   \\
    |\bxi_{jl}|=T(\bbeta_j)^{-k_2/k_1} | \bbeta_{j}|^{\circ (1/k_1)}\,&\text{for}\,\,|\bxi_{jl}| \in \mathbb{R}^{|\Gj|},\, l \in V_j
\end{cases} $$

Since $\hbxi_j^{abs}$ is a composition of continuous functions on $\mathbb{R}^{|\Gj|}$ (norms, powers, multiplication), its continuity follows immediately.
Note for vector-valued $|\bxi_{jl}|$, each $i$th entry is $T(\bbeta_j)^{-k_2/k_1} | \beta_{ji}|^{1/k_1}$, $i \in \Gj$. To verify the balancedness optimality conditions, we have by construction  $|\xi_{jli}|=|\xi_{jl'i}|=T(\bbeta_j)^{-k_2/k_1}|\beta_{ji}|^{1/k_1} \, \forall \, l,l' \in V_j, i \in \Gj$, as well as $|\xi_{jl}|=T(\bbeta_j) \, \forall \, l \in S_j$. To show that the $\xi_{jli}$ balance at the optimal scale in the first AM-GM application, see
\begin{align*}
    |\beta_{ji}/\beta_j^S|^{1/k_1} = |\beta_{ji}/T(\bbeta_j)^{k_2}|^{1/k_1} =  T(\bbeta_j)^{-k_2/k_1} | \beta_{ji}|^{1/k_1} = |\xi_{jli} | \quad \forall \, l \in V_j \,\, i \in \G_j.
\end{align*}
Finally, to confirm that the lower bound of the penalty for the factors in $V_j$, $k_1 \| \bbeta_j / \beta_j^S \|_{2/k_1}^{2/k_1}$, is balanced in the second AM-GM application, we use both the definition of $T(\bbeta_j)$ and $|\beta_j^S|=T(\bbeta_j)^{k_2}$ to obtain
\begin{align*}
    \| \bbeta_j / \beta_j^S \|_{2/k_1}^{1/k_1} &=  T(\bbeta_j)^{-k_2/k_1}\| \bbeta_j \|_{2/k_1}^{1/k_1} = \| \bbeta_j \|_{2/k_1}^{-k_2/(k_1(k_1+k_2))} \| \bbeta_j \|_{2/k_1}^{1/k_1} = \| \bbeta_j \|_{2/k_1}^{k_1/(k_1(k_1+k_2))}\\
    &= \| \bbeta_j \|_{2/k_1}^{1/(k_1+k_2)} = T(\bbeta_j).
\end{align*}
To show feasibility of the constructed solutions, select any $\beta_{ji}=\K_{ji}(\bxi_j), i \in \Gj$. Then:
\scalebox{0.92}{
  \begin{minipage}{\linewidth}
\begin{align*}
\K_{ji}(\hat{\xi}_{ji}^{abs}(\bbeta_j))&=\prod_{l \in S_j} T(\bbeta_j)^{\alpha_l}
\prod_{l \in V_j} (T(\bbeta_j)^{-k_2/k_1} | \beta_{ji}|^{1/k_1})^{\alpha_l}=T(\bbeta_j)^{k_2}\left( |\beta_{ji}|^{1/k_1} T(\bbeta_j)^{-k_2/k_1} \right)^{k_1}\\
&=T(\bbeta_j)^{k_2}\large( |\beta_{ji}|^{1/k_1}  \large)^{k_1}T(\bbeta_j)^{-k_2}=|\beta_{ji}|.
\end{align*} 
  \end{minipage}
}\\ \vspace{0.1cm}

\noindent The non-uniqueness of the solution signs and the multiplicative parametrization result in sign-flip symmetries. This also implies that, in addition to balanced magnitudes, solutions also require suitable sign configurations to ensure $\K_j(\bxi_j)=\bbeta_j$ and not a sign-permuted version.

\textbf{(iii) Lower hemicontinuity of solution mapping.} To establish lower hemicontinuity of $\hbxi_j(\bbeta_j)$, our goal is to show: for all $\bar{\bxi}_j \in \hbxi_j(\bbeta_j)$ and any $\varepsilon>0$, there is $\delta>0$ so that for all $\bbeta_j' \in \B(\bbeta_j,\delta)$ there exists $\bxi_j' \in \hbxi_j(\bbeta_j')\,\cap\,\B(\bar{\bxi}_j,\varepsilon)$. The continuity of the $\hbxi_j^{abs}(\bbeta_j)$ significantly simplifies the argument.\\
Fix an arbitrary $\bbeta_j \neq \bm{0}$ and any solution $\bar{\bxi}_j \in \hbxi_j(\bbeta_j)$. By our construction of the absolute values of the solutions, it follows that $|\bar{\xi}_{jl}|=T(\bbeta_j)$ and $|\bar{\xi}_{jli}|=T(\bbeta_j)^{-k_2/k_1}|\beta_{ji}|^{1/k_1}$ for all $i \in \Gj, l \in [k]$. For any $\bbeta_j'$ near $\bbeta_j$, we can use the construction $\hbxi_j^{abs}(\bbeta_j')$ to obtain $|\bxi_j'|$.\\
For a fixed $\bar{\bxi}_j \in \hbxi_j(\bbeta_j)$, define the signed function $\hbxi_j^{sign}(\bbeta_j')\triangleq \operatorname{sign}(\bar{\bxi}_j) \odot \hbxi_j^{abs}(\bbeta_j')$ by applying the signs of $\bar{\bxi}_j$ to the absolute values. Trivially, $\hbxi_j^{sign}(\bbeta_j)=\bar{\bxi}_j$. Note that this is a slight abuse of notation: in the case of a zero entry $\beta_{ji}=0$, we have $\bar{\xi}_{jli}=0 \, \forall \, l \in V_j$, and the respective entries in $\operatorname{sign}(\bar{\bxi}_j)$ are undefined. In this case, we set those signs to any pattern that respects $\operatorname{sign}(\beta_{ji}')$ to ensure $\K_{ji}(\hat{\xi}_{ji}^{sign}(\bbeta_j'))=\beta_{ji}'$ for all $\bbeta_j'$ sufficiently close to $\bbeta_j$. This is permissible for our argument because the sign of $\xi_{jli}'$ becomes irrelevant when measuring distance to $\bar{\xi}_{jli}=0$. For simplicity, we can hence assume $\operatorname{sign}(\bar{\bxi}_j)$ to be well-defined. Because $\hbxi_j^{abs}$ is continuous at $\bbeta_j$,  $\forall \varepsilon>0 \exists\delta>0: \forall\bbeta' \in \B(\bbeta_j,\delta)$:  $\hbxi_j^{abs}(\bbeta_j') \in \B(\hbxi_j^{abs}(\bbeta_j), \varepsilon)$. But then $\bxi_j' \triangleq \hbxi_j^{sign}(\bbeta_j') \in \B(\bar{\bxi}_j, \varepsilon)$ must hold, because 
\begin{align*}
    \| \bar{\bxi}_j - \bxi_j' \|_2 = \| \operatorname{sign}(\bar{\bxi}_j) \odot |\bar{\bxi}_j| -\operatorname{sign}(\bar{\bxi}_j) \odot |\bxi_j'| \|_2 
    &= \| \operatorname{sign}(\bar{\bxi}_j) \odot (\hbxi_j^{abs}(\bbeta_j) - \hbxi_j^{abs}(\bbeta_j')) \|_2 \\
    &= \| (\hbxi_j^{abs}(\bbeta_j) - \hbxi_j^{abs} (\bbeta_j') \|_2  < \varepsilon
\end{align*}
Therefore, $\hbxi_j(\bbeta_j)$ is lower hemicontinuous at $\bbeta_j \neq 0$. Now, for $\bbeta_j=\bm{0}$ we have $\hbxi_j(\bm{0})=\{\bm{0}\}$. Because we know $\hbxi_j^{abs}(\bbeta_j)$ continuously approaches $\bm{0}$ as $\bbeta_j \to \bm{0}$ we can infer that for all $\varepsilon>0$ there is $\delta>0$ such that for any $\bbeta_j' \in \B(\bm{0}, \delta)$, \textit{all} solutions in $\hbxi_j(\bbeta_j')$ are in $\B(\bm{0},\varepsilon)$, and thus $\hbxi_j(\bbeta_j)$ is also lower hemicontinuous at $\bm{0}$.

This shows $\hbxi_j(\bbeta_j)$ is lower hemicontinuous on $\mathbb{R}^{|\Gj|}$. Since each $\hbxi_j(\bbeta_j)$ is lower hemicontinuous on $\mathbb{R}^{|\Gj|}$ and the solution map is the Cartesian function product $\hbxi(\bbeta) = (\hbxi_1(\bbeta_1), \ldots, \hbxi_L(\bbeta_L))$, it follows that $\hbxi(\bbeta)$ is lower hemicontinuous at every $\bbeta \in \Rd$. 
\end{proof}

\vspace{-0.0cm}
%
%
\subsection{Proof of Lemma~\ref{lemma:S_hdp-def}} \label{app:hdp-proof}
\begin{proof}
As in Lemma~\ref{lemma:S_hpp-def}, we can proceed by finding the minimum element-wise. Since the constraint implies that the difference of two non-negative numbers equals $\beta_j$ for $j=1,\ldots,d$, we further differentiate by the sign of $\beta_j$. For $\beta_j=0$, the constraint reduces to $\gamma_j^2=\delta_j^2$, which provides a unique minimizer $(\hat{\gamma}_j, \hat{\delta}_j) = (0,0)$, resulting in a minimum $\ell_2$ regularization term of $0=|\beta_j|$. For $\beta_j>0$, the constraint gives us $ \gamma_j^2 = \beta_j + \delta_j^2 \geq \beta_j \implies |\gamma_j| \geq \sqrt{|\beta_j|} = \sqrt{\beta_j}$. Thus, we consider $\gamma_j = \pm \sqrt{|\beta_j|}$ and $\delta_j=0$. This choice trivially satisfies the constraint, and it is easy to see that any other pair $(\gamma_j, \delta_j)$ satisfying $\beta_j = \gamma_j^2 - \delta_j^2$ needs to have a strictly larger magnitude in both $\gamma_j$ and $\delta_j$, resulting in a larger sum of the squared $2\text{-norms}$. Thus the minimizers for $\beta_j>0$ are given by $(\hat{\gamma}_j, \hat{\delta}_j) = (\pm \sqrt{|\beta_j|},0)$, resulting in a minimum regularization term of $|\beta_j|$. For $\beta_j < 0$, an analogous argument holds: By the constraint $\gamma_j^2-\delta_j^2=\beta_j$ we have $\delta_j^2 = \gamma_j^2 - \beta_j \geq - \beta_j = |\beta_j| \implies |\delta_j| \geq \sqrt{|\beta_j|}$. Considering $\delta_j = \pm \sqrt{|\beta_j|}$ and $\gamma_j = 0$, we again observe that any other pair $(\gamma_j, \delta_j)$ satisfying the constraint has strictly larger magnitude in $\gamma_j$ and $\delta_j$, resulting in a larger $\ell_2$ regularization term. Thus, the minimizers for $\beta_j<0$ are given by $(\hat{\gamma}_j, \hat{\delta}_j) = (0,\pm \sqrt{|\beta_j|})$, yielding a minimum $\ell_2$ penalty of $|\beta_j|$.\\
In all three cases, the minimum of $\gamma_j^2+\delta_j^2$ subject to $\gamma_j^2-\delta_j^2=\beta_j$ is given by $|\beta_j|$. The proof is completed by iterating over $j=1,\ldots,d$.
\end{proof}
\vspace{-0.0cm}
\subsection{Proof of Lemma~\ref{lemma:S_ghpp-def}}\label{app:proof-lemma-ghpp}

\begin{proof}
Due to the separable structure of the parametrization, we can proceed by finding the minimizer for each summand $j\in[L]$. Using the AM-GM on 
$\Vert \bm{u}_{j} \Vert_{2}^{2}$ and $\nu_{j}^{2}$,
\begin{align} \label{eq:am-gm-groupstep}
    \frac{\Vert \bm{u}_{j} \Vert_{2}^{2} + \nu_{j}^{2}}{2} &\geq \sqrt{\nu_{j}^{2} \cdot \Vert \bm{u}_{j} \Vert_{2}^{2}} = \sqrt{( \nu_{j} \cdot \Vert \bm{u}_{j} \Vert_{2})^{2}} 
    = |\nu_j| \cdot  \Vert \bm{u}_{j} \Vert_{2} =  \Vert \nu_j \bm{u}_{j} \Vert_{2} =  \Vert \bm{\beta}_{j} \Vert_{2} \,, \nonumber
\end{align}
where we used the absolute homogeneity of norms. The expression reduces to equality if and only if $\Vert \bm{u}_{j} \Vert_{2}^{2} = \nu_{j}^{2} = \Vert \bm{\beta}_{j} \Vert_{2}$. Iterating over all groups $j=1,\ldots,L$ shows that the constrained minimum in (\ref{eq:S_ghpp-def}) is indeed $2 \Vert \bm{\beta} \Vert_{2,1}$ for all $\bbeta\in\Rd$.
\end{proof}
\vspace{-0.0cm}
%
\subsection{Proof of Lemma~\ref{lemma:openness-ghpp}}\label{app:proof-openness-ghpp}
\begin{proof}
We show that $\K:\mathbb{R}^{d}\times\mathbb{R}^{L}\to\Rd,\,(\bu,\bnu) \mapsto \bu \odotg \bnu$, is locally open at $(\bu,\bnu)$, with $\bu=(\bu_1,\ldots,\bu_L)^{\top}$ and $\bnu=(\nu_1,\ldots,\nu_L)^{\top}$, if the $(\bu_j,\nu_j)$ are such that $\nu_j=0$ implies $\norm{\bu_j}_2=0$ for all $j\in[L]$. Recall that $d=|\G_1|+\ldots+|\G_L|$. We proceed in two steps. First, we find the points of openness for the group-wise parametrizations $\K_j:\mathbb{R}^{|\G_j|}\times\mathbb{R}\to \mathbb{R}^{|\G_j|},\,(\bu_j,\nu_j)\mapsto  \bu_j\nu_j$. In a second step, we then show that local openness of $\K_j$ at $(\bu_j,\nu_j)$ for $j\in[L]$  implies local openness of the GHPP $$\K(\bu,\bnu)\triangleq \bu \odotg \bnu = (\bu_1\nu_1,\ldots,\bu_L\nu_L)^{\top} = (\K_1(\bu_1,\nu_1),\ldots,\K_L(\bu_L,\nu_L))^{\top}\,$$ at $\bu=(\bu_1,\ldots,\bu_L)^{\top}$ and $\bnu=(\nu_1,\ldots,\nu_L)^{\top}$. For the first step, we show that the $\K_j$ are open at all points  $(\bu_j,\nu_j)\in\mathbb{R}^{|\G_j|}\times\mathbb{R}$ except $(\bu_j,\nu_j)\in(\mathbb{R}^{|\G_j|}\times\{0\})\setminus\{(\bm{0},0)\}$. To do this, we use the following result on the local openness of matrix multiplication:
\begin{proposition}[Prop.~1 in \citet{nouiehed2022learning}, rephrased]
\hspace{0.1cm} Let $\mathcal{M} : \mathbb{R}^{m \times z}\times \mathbb{R}^{z \times n} \to \mathbb{R}^{m \times n},\,(\bm{M}_1,\bm{M}_2)\mapsto \bm{M}_1\bm{M}_2$, denote the bilinear matrix multiplication mapping such that $z\geq\min\{m,n\}$. Then $\mathcal{M}$ is locally open at $(\bm{M}_1,\bm{M}_2)$ if and only if 
\begin{align*}
\exists \tilde{\bm{M}}_1 \in \mathbb{R}^{m \times z}:\tilde{\bm{M}}_1\bm{M}_2=\bm{0}_{m \times n} \,&\land\, \tilde{\bm{M}}_1+\bm{M}_1\, \text{is full row-rank} \,\,\, \textbf{or} \, \\
\exists\tilde{\bm{M}}_2 \in \mathbb{R}^{z \times n}:\bm{M}_1\Tilde{\bm{M}}_2=\bm{0}_{m \times n} \,&\land\, \tilde{\bm{M}}_2+\bm{M}_2 \,\text{is full column-rank}\,.
\end{align*}
\end{proposition}
Letting $m=|\G_j|>1,\,z=1$ and $n=1$, we can apply this result to the group-wise functions $\K_j$: $\K_j$ is open at $(\bm{0},0)\in\mathbb{R}^{|\G_j|}\times\mathbb{R}$ if $\exists \tilde{\nu}_j: \bm{0}\tilde{\nu}_j=\bm{0}$ and  $0+\tilde{\nu}_j$ has full column-rank, i.e., $\tilde{\nu}_j\neq 0$. This holds for all $\tilde{\nu}_j\neq 0$. Further, $\K_j$ is open at $(\bu_j,\nu_j)$, with $\norm{\bu_j}_2\geq0,\,\nu_j \neq 0$, if $\exists \tilde{\nu}_j: \bu_j \tilde{\nu}_j = \bm{0}$, and $\nu_j + \tilde{\nu}_j \neq 0$. This holds for $\tilde{\nu}_j=0$. Finally, $\K_j$ were to be open at $(\bu_j,0)$ with $\norm{\bu_j}_2>0$, if either $\exists \tilde{\nu}_j: \bu_j \tilde{\nu}_j=\bm{0}$ and $\nu_j+\tilde{\nu}_j \neq 0$, or $\exists\tilde{\bu}_j: \tilde{\bu}_j\nu_j=\bm{0}$ and $\bu_j+\tilde{\bu}_j$ has full row-rank. The first condition implies $\tilde{\nu}_j=0$, but then $0+\tilde{\nu}_j=0$, contradicting $\nu_j + \tilde{\nu}_j\neq0$. Also, there is no such $\tilde{\bu}_j$ as in the second condition, since $\bu_j+\tilde{\bu}_j \in \mathbb{R}^{|\G_j| \times 1}$ can not be full row-rank for $|\G_j|>1$. Therefore, we have shown that the $\K_j$ are locally open at all points in $\mathbb{R}^{|\G_j|}\times\mathbb{R}$ except $(\bu_j,\nu_j)\in(\mathbb{R}^{|\G_j|}\times\{0\})\setminus\{(\bm{0},0)\}$.
\vspace{0.15cm}

\noindent For the second step, let the Cartesian product of two Euclidean spaces be endowed with the norm $\norm{\norm{\cdot}_2,\norm{\cdot}_2}_2$. We now show that if $\K_j$ is open at $(\bu_j,\nu_j)$ for $j\in[L]$, then $\K$ is open at $(\bu,\bnu)$, i.e.,
$$\forall \varepsilon>0 \,\exists\, \tilde{\delta}>0:\, \mathcal{B}(\K(\bu,\bnu),\tilde{\delta}) \subseteq \K(\mathcal{B}((\bu,\bnu), \varepsilon))\,.$$
Let $\varepsilon>0$ be arbitrary. Define $\varepsilon_j \triangleq \varepsilon/\sqrt{L}$. By the local openness of the $\K_j$ at $(\bu_j,\nu_j)$, there are $\delta_j$ such that $\mathcal{B}(\K_j(\bu_j,\nu_j),\delta_j) \subseteq \K_j(\mathcal{B}((\bu_j,\nu_j),\varepsilon_j))$ for all $j\in[L]$. Let $\tilde{\delta}\triangleq \min_j\{\delta_j\}$ and let $\tilde{\bbeta}\in\mathcal{B}(\K(\bu,\bnu),\tilde{\delta})$ be arbitrary. Writing $\tilde{\bbeta}=(\tilde{\bbeta}_1,\ldots,\tilde{\bbeta}_L)^{\top}$, we then have
$$\Vert\tilde{\bbeta}-\K(\bu,\bnu)\Vert_2^2 = \textstyle\sum_{j=1}^{L} \Vert\tilde{\bbeta}_j-\K_j(\bu_j,\nu_j)\Vert_2^2 < \tilde{\delta}^2\,,$$
which implies $\Vert\tilde{\bbeta}_j-\K_j(\bu_j,\nu_j)\Vert_2<\tilde{\delta}\leq\delta_j$. By local openness of the $\K_j$, there then exist $(\tilde{\bu}_j,\tilde{\nu}_j)$ such that $\K_j(\tilde{\bu}_j,\tilde{\nu}_j)=\tilde{\bbeta}_j$, with $\Vert (\bu_j, \nu_j) - (\tilde{\bu}_j, \tilde{\nu}_j) \Vert = \Vert \Vert \bu_j -\tilde{\bu}_j \Vert_2,\, |\nu_j - \tilde{\nu}_j|\,\Vert_2 < \varepsilon_j = \varepsilon/\sqrt{L}$. Defining $\tilde{\bu}=(\tilde{\bu}_1,\ldots,\tilde{\bu}_L)^{\top}$ and $\tilde{\bnu}=(\tilde{\nu}_1,\ldots,\tilde{\nu}_L)^{\top}$, we find 
\begin{align*}
    \Vert (\bu,\bnu) - (\tilde{\bu},\tilde{\bnu}) \Vert^2 =  \Vert \, \Vert \bu-\tilde{\bu} \Vert_2,\, \Vert \bnu - \tilde{\bnu} \Vert_2\,\Vert_2^2 &= \textstyle\sum_{j=1}^{L} \Vert \bu_j - \tilde{\bu}_j \Vert_2^2 + \textstyle\sum_{j=1}^{L} |\nu_j - \tilde{\nu}_j|^2\\
    &= \textstyle\sum_{j=1}^{L} \Vert\,\Vert \bu_j-\tilde{\bu}_j \Vert_2,\,|\nu_j-\tilde{\nu}_j|\,\Vert_2^2\\
    &< \textstyle\sum_{j=1}^{L} \left(\frac{\varepsilon}{\sqrt{L}} \right)^2 = \varepsilon^2\,,
\end{align*}
and thus $ \Vert (\bu,\bnu) - (\tilde{\bu},\tilde{\bnu}) \Vert = \Vert \, \Vert \bu-\tilde{\bu} \Vert_2,\, \Vert \bnu - \tilde{\bnu} \Vert_2\,\Vert_2<\varepsilon$. By definition of $\K$, we have $$\K(\tilde{\bu},\tilde{\bnu})=(\K_1(\tilde{\bu}_1,\tilde{\nu}_1),\ldots, \K_L(\tilde{\bu}_L,\tilde{\nu}_L))^{\top} = (\tilde{\bbeta}_1,\ldots,\tilde{\bbeta}_L)^{\top}=\tilde{\bbeta} \in \Rd\,.$$ 
Taking both results together, we obtain $\tilde{\bbeta}\in \K(\mathcal{B}((\bu,\bnu),\varepsilon))$. Because $\tilde{\bbeta}$ was chosen without loss of generality, it follows that $\mathcal{B}(\K(\bu,\bnu),\tilde{\delta}) \subseteq \K(\mathcal{B}((\bu,\bnu), \varepsilon))$. As $\varepsilon>0$ was arbitrary, we have shown the second step, i.e., that local openness of $\K_j$ at $(\bu_j,\nu_j)$ for all $j\in[L]$ implies local openness of $\K$ at $(\bu,\bnu)$, with $\bu=(\bu_1,\ldots,\bu_L)^{\top}$ and $\bnu=(\nu_1,\ldots,\nu_L)^{\top}$.\\
Combining both steps completes the proof, and it is shown that $\K$ is locally open at $(\bu,\bnu)$, if for all $(\bu_j,\nu_j), j\in[L]$, it holds that $\nu_j$ is zero only if $\norm{\bu_j}_2 = 0$ as well.
\end{proof}
\vspace{-0.0cm}
%


\subsection{Derivation of group size-adjusted GHPP}
\label{app:adj-ghpp}

We can induce the group size-adjusted group lasso penalty $\mathcal{R}_{\bm{\beta}}(\bm{\beta})\triangleq\sum_{j=1}^{L} \sqrt{|\mathcal{G}_j|}\norm{\bm{\beta}_j}_2$ as a simple extension to the previous GHPP approach, by counting each entry in $\bv_j$ as its own parameter for the surrogate regularization, instead of subsuming all 
entries of the Hadamard factor 
under the scalar parameter $\nu_j$ as in~\ref{sec:group-lasso-vanilla}. 
In this setting, the surrogate $\ell_2$ regularization term counts $\nu_j$ not once, but $|\mathcal{G}_j| \triangleq p_j$ times, and is written as follows:
$\widetilde{\Rxi}(\bu,\bnu)=\sum_{j=1}^{L}\big(\Vert \bm{u}_{j} \Vert_{2}^{2} + p_j \nu_{j}^{2}\big)$. 
Applying the AM-GM inequality to $\Vert \bm{u}_{j} \Vert_{2}^{2}$ and $(\sqrt{p_j}\nu_j)^2$ for $j \in [L]$, it holds
{\small
\begin{align}
\textstyle \sum_{j=1}^{L}\left(\Vert \bm{u}_{j} \Vert_{2}^{2} + (\sqrt{p_j}\nu_j)^2\right)
& 
\geq 2 \textstyle \sum_{j=1}^{L} \sqrt{\Vert \bm{u}_{j} \Vert_{2}^{2} (\sqrt{p_j}\nu_j)^2} = 2 \textstyle \sum_{j=1}^{L} \sqrt{\left(\Vert \bm{u}_{j} \Vert_{2} (\sqrt{p_j}\nu_j)\right)^2} 
\nonumber \\
&= 
2 \textstyle\sum_{j=1}^{L} \left|\Vert \bm{u}_{j} \Vert_{2} \cdot (\sqrt{p_j}\nu_j)\right| \nonumber = 2 \textstyle\sum_{j=1}^{L} \sqrt{p_j} \cdot |\nu_j | \cdot \Vert \bm{u}_{j} \Vert_{2} 
\nonumber \\ 
&= 
2 \textstyle\sum_{j=1}^{L} \sqrt{p_j} \Vert \nu_j \bm{u}_{j} \Vert_{2} \nonumber = 2 \textstyle\sum_{j=1}^{L} \sqrt{p_j} \Vert \bm{\beta}_{j} \Vert_{2} 
\,,
\end{align}
}

\noindent with equality if and only if $\Vert \bm{u}_{j} \Vert_{2}^{2} = (\sqrt{p_j}\nu_{j})^{2} = \sqrt{p_j} \Vert \bm{\beta}_{j} \Vert_{2}$. The constrained minimizers $\hat{\bm{u}}_{j}$ and $\hat{\nu_j}$ corresponding to some $\bm{\beta}_j$ are obtained as
\begin{equation}\label{eq:minimizers-ghpp-adj}
\arg\hspace{-0.05cm}\min_{\hspace{-0.4cm}\substack{(\bu_j,\nu_j):\\ \bm{\beta}_{j} = \nu_j \bm{u}_{j}}}
   \Vert \bm{u}_{j} \Vert_{2}^{2} + (\sqrt{p_j}\nu_{j})^{2} = 
    \begin{cases}
       \pm \left(\frac{\bbeta_j}{\sqrt{\Vert \bm{\beta}_{j} \Vert_{2}/ \sqrt{p_j}}}, \sqrt{\Vert \bm{\beta}_{j} \Vert_{2}/ \sqrt{p_j}}\right) & \hspace{-0.25cm}\text{ $\norm{\bbeta_j}_2>0$} \\
       (\bm{0},0) & \hspace{-0.25cm}\text{ $\norm{\bbeta_j}_2=0$} 
    \end{cases} \nonumber
\end{equation}
for each $j \in [L]$. Using identical arguments as for the unadjusted GHPP in \ref{sec:group-lasso-vanilla}, we can construct the equivalent smooth surrogate $\Q$ in Equation (\ref{eq:q-adj-ghpp}) for the non-smooth objective $\P$ regularized with the adjusted $\ell_{2,1}$ penalty in Equation (\ref{eq:p-adj-ghpp}).
Minimizing $\Q$ over $(\bm{\psi}, \bm{u}, \bnu)$ yields (local) solutions to $\P$ in (\ref{eq:p-adj-ghpp}), which can be reconstructed using $(\hbpsi,\hat{\bm{\beta}}) = (\hbpsi,\hat{\bm{u}} \odotg \hat{\bm{\nu}})$ as defined above.
\vspace{-0.0cm}
%
%
\subsection{Proof of Lemma~\ref{lemma:S_hppk-def}}
\label{app:proof-s-hppk}
\begin{proof}
Applying the AM-GM inequality for each $j=1,\ldots,d$ to the squared parameters $u_{jl}^{2}$, $l=1,\ldots,k$, we obtain
\begin{align}
   \frac{u_{j1}^{2} + \ldots + u_{jk}^{2}}{k} &\geq \sqrt[k]{(u_{j1}^{2}) \cdot \ldots \cdot (u_{jk}^{2})} = \sqrt[k]{\left( u_{j1} \cdot \ldots \cdot u_{jk}\right)^{2}} = \sqrt[k]{|\beta_j|^2} = |\beta_j|^{2/k} \;, \nonumber
\end{align}
with equality holding if and only if $u_{j1}^{2} = \ldots = u_{jk}^{2} = |\beta_j|^{2/k}$. Summing over all $j\in [d]$ then shows the result.
\end{proof}
\vspace{-0.0cm}

\subsection{Proof of Lemma~\ref{lemma:proof-hppk}}
\label{app:proof-hppk}
\begin{proof}
To prove the global openness of the $k$-linear function $\mathcal{K}:\prod_{l=1}^{k} \mathbb{R}^{d} \to \mathbb{R}^{d}, (\bm{u}_1,\ldots,\bm{u}_{k}) \mapsto \bigodot_{l=1}^{k} \bm{u}_{l} = \bm{\beta}$, defining the $\text{HPP}_k$, we make use of an existing result for scalar-valued multilinear maps and then generalize it to the $d$-dimensional real-valued case.
\begin{proposition}[Theorem 1.2 in \citet{balcerzak2016certain}, rephrased]

Let $X_1, \ldots,$ $X_k$ be normed spaces over the scalar field $\mathbb{K} \in\{\mathbb{R}, \mathbb{C}\}$, and let $T$ from $X_1 \times \cdots \times X_k$ to $\mathbb{K}$ be a nontrivial $k$-linear functional. Then $T$ is globally open.
\end{proposition}
Using this result, the global openness of the $\text{HPP}_k$ for $d=1$ follows directly, or equivalently, for a single entry of the general $d$-dimensional $\text{HPP}_k$. We define the entry-wise parametrizations as $\mathcal{K}_j:\prod_{l=1}^{k} \mathbb{R} \to \mathbb{R}, (u_{j1},\ldots,u_{jk}) \mapsto \prod_{l=1}^{k} u_{jl} = \beta_j$ for $j\in[d]$, such that 
\begin{align*}
\K(\bu_1,\ldots,\bu_k)&=(\K_1(u_{11},\ldots,u_{1k}),\ldots,\K_{d}(u_{d1},\ldots,u_{dk}))^{\top}\\
\implies \, \bigodot_{l=1}^{k}\bu_l &= {\textstyle (\prod_{l=1}^{k} u_{1l},\ldots,\prod_{l=1}^{k} u_{dl})}^{\top}\,,
\end{align*}
where $\bu_l=(u_{1l},\ldots,u_{dl})^{\top} \in \Rd$ contains the parameters in each Hadamard factor $l\in[k]$. Let $\bu_j = (u_{j1},\ldots,u_{jk})^{\top} \triangleq (u_{jl})_{l=1}^{k} \in \mathbb{R}^{k}$ collect the parameters of the entry-wise parametrizations $\K_j$ for $j \in [d]$. For clarity, we also use $(\bu_l)_{l=1}^{k}$ to abbreviate $(\bu_1,\ldots,\bu_k)$, and further endow the $k$-times Cartesian product of Euclidean spaces with the norm $\Vert (\bu_l)_{l=1}^{k} \Vert$ $\triangleq \Vert \,\Vert\bu_1 \Vert_2,\ldots, \Vert\bu_k \Vert_2 \, \Vert_2$.
\vspace{0.15cm}

\noindent We now proceed to show that local openness of $\K_j$ at $\bu_j=(u_{jl})_{l=1}^{k}$, i.e.,
\begin{equation*} \label{eq:local-openness-kj} 
\forall \varepsilon_j>0 \exists \delta_j>0: \mathcal{B}(\K_j((u_{jl})_{l=1}^{k}),\delta_j) \subseteq \K_j(\mathcal{B}((u_{jl})_{l=1}^{k},\varepsilon_j)) \quad \forall\,j\in[d]\,,
\end{equation*}
implies local openness of $\K$ at $(\bu_l)_{l=1}^{k}$, i.e.,
\begin{equation}\label{eq:local-openness-k} \nonumber
\forall \varepsilon>0 \exists \tilde{\delta}>0:\, \mathcal{B}(\K((\bu_l)_{l=1}^{k}),\tilde{\delta}) \subseteq \K(\mathcal{B}((\bu_l)_{l=1}^{k},\varepsilon))\,,
\end{equation}
where each $\bu_l$ is constructed as $\bu_l=(u_{1l},\ldots,u_{dl})^{\top}$ from the points of openness $\bu_j=(u_{jl})_{l=1}^{k}$ of the $\K_j$. Let $\varepsilon>0$ be arbitrary and define $\varepsilon_j \triangleq \varepsilon/\sqrt{d}$. By our assumption, there are $\delta_j$ such that (\ref{eq:local-openness-kj}) holds for each $j\in[d]$ with $\varepsilon_j$. Let $\tilde{\delta}\triangleq\min_{j}\{\delta_j\}$ and pick any $\tilde{\bbeta}\in  \mathcal{B}(\K((\bu_l)_{l=1}^{k}),\tilde{\delta})$. It then holds by definition 
\vspace{-0.1cm}
\begin{equation*}\label{eq:loc-op}
\Vert \tilde{\bbeta} - \K((\bu_l)_{l=1}^{k}) \Vert_2^2 = {\textstyle \sum_{j=1}^{d} } | \tilde{\beta}_j - \K_j((u_{jl})_{l=1}^{k}) |^2 < \tilde{\delta}^2 \leq \delta_j^2,
\end{equation*}
implying $\tilde{\beta}_j \in \mathcal{B}(\K_j((u_{jl})_{l=1}^{k}),\tilde{\delta}) \subseteq \mathcal{B}(\K_j((u_{jl})_{l=1}^{k}),\delta_j)$. By local openness of the $\K_j$, it follows that $\tilde{\beta}_j \in \K_j(\mathcal{B}((u_{jl})_{l=1}^{k},\varepsilon_j))$. This means that $\forall \, j \in [d]$ we have
{\small
\begin{equation*}
    \exists (\tilde{u}_{jl})_{l=1}^{k}: \K_j((\tilde{u}_{jl})_{l=1}^{k})=\tilde{\beta}_j \,\,\text{and}\,\, \Vert (u_{jl})_{l=1}^{k} - (\tilde{u}_{jl})_{l=1}^{k} \Vert^2 = \Vert\,\vert u_{j1}-\tilde{u}_{j1}\vert,\ldots,\vert u_{jk}-\tilde{u}_{jk}\vert\,\Vert_2^2 < \varepsilon_j^2\,.
\end{equation*}
}
Collecting the $\tilde{u}_{jl}$ as $\tilde{\bu}_l = (\tilde{u}_{1l},\ldots,\tilde{u}_{dl})^{\top}$ for $l \in [k]$, and evaluating $\K$ at these arguments, we obtain 
$$\K((\tilde{\bu}_l)_{l=1}^{k})=(\K_1((\tilde{u}_{1l})_{l=1}^{k}),\ldots,\K_d((\tilde{u}_{dl})_{l=1}^{k}))^{\top}=(\tilde{\beta}_1,\ldots,\tilde{\beta}_d)^{\top}=\tilde{\bbeta} \in \Rd\,,$$
as well as 
\begin{align*}
  \Vert (\bu_l)_{l=1}^{k} - (\tilde{\bu}_l)_{l=1}^{k}\Vert^2 &= \Vert\,\Vert \bu_1 - \tilde{\bu}_1 \Vert_2,\ldots,\Vert \bu_k - \tilde{\bu}_k \Vert_2 \,\Vert_2^2 = \textstyle\sum_{l=1}^{k} \Vert \bu_l - \tilde{\bu}_l \Vert_2^2 \\
  &= \textstyle\sum_{l=1}^{k} \textstyle\sum_{j=1}^{d} \vert u_{jl}-\tilde{u}_{jl} \vert^2 = \textstyle\sum_{j=1}^{d} \Vert (u_{jl})_{l=1}^{k} - (\tilde{u}_{jl})_{l=1}^{k} \Vert_2^2\\ 
  &< \textstyle\sum_{j=1}^{d} \varepsilon_j^2 = d \left(\frac{\varepsilon}{\sqrt{d}} \right)^2 = \varepsilon^2\,,
\end{align*}
i.e., $\Vert (\bu_l)_{l=1}^{k} - (\tilde{\bu}_l)_{l=1}^{k}\Vert < \varepsilon$. Taking both findings together, it follows $\tilde{\bbeta} \in \K(\mathcal{B}((\bu_l)_{l=1}^{k},\varepsilon))$. Because $\tilde{\bbeta}$ was arbitrary, we have $\mathcal{B}(\K((\bu_l)_{l=1}^{k}),\tilde{\delta}) \subseteq \K(\mathcal{B}((\bu_l)_{l=1}^{k},\varepsilon))$. Finally, because $\varepsilon>0$ was arbitrary, local openness of $\K_j$ at $(u_{jl})_{l=1}^{k}$ for all $j=1,\ldots,L$ implies local openness of $\K$ at $(\bu_l)_{l=1}^{k}$. Since the $\K_j$ are globally open, it follows that $\K$ is also globally open, completing the proof.
\end{proof}
\vspace{-0.0cm}

\subsection{Proof of Lemma~\ref{lemma:S_ghppk-def}}\label{app:lemma:S_ghppk-def}
\begin{proof}
Using the AM-GM on the group-wise parameters $j\in [L]$, it holds
\scalebox{0.92}{
  \begin{minipage}{\linewidth}
\begin{align}
    \textstyle\sum_{j=1}^{L} \frac{\Vert \bm{u}_{{j}} \Vert_{2}^{2} + \textstyle\sum_{r=1}^{k-1} \nu_{jr}^{2}}{k} &\geq \textstyle\sum_{j=1}^{L} \normalsize( \Vert \bm{u}_{j} \Vert_{2}^{2} \cdot \nu_{j2}^{2}\cdot\ldots\cdot \nu_{jk}^{2} \normalsize)^{1/k} = \textstyle\sum_{j=1}^{L} \Big(\sqrt{\left(\Vert \bm{u}_{j} \Vert_{2} \cdot \nu_{j2}\cdot\ldots\cdot \nu_{jk} \right)^{2}} \Big)^{2/k} \nonumber \\
    &= \textstyle\sum_{j=1}^{L} \left| \Vert \bm{u}_{j} \Vert_{2} \cdot \nu_{j2}\cdot\ldots\cdot \nu_{jk} \right|^{2/k} \nonumber = \textstyle\sum_{j=1}^{L} \left(|\nu_{j2}\cdot\ldots\cdot \nu_{jk}| \cdot \Vert \bm{u}_{j} \Vert_{2} \right)^{2/k} \nonumber \\
    &= \textstyle\sum_{j=1}^{L} \Vert  \bm{u}_{j} \cdot \nu_{j2}\cdot\ldots\cdot \nu_{jk}  \Vert_{2}^{2/k} = \textstyle\sum_{j=1}^{L} \Vert \bm{\beta}_{{j}} \Vert_{2}^{2/k} = \Vert \bm{\beta} \Vert_{2,2/k}^{2/k} \nonumber
\end{align}
\end{minipage} }
with equality if and only if $\Vert \bm{u}_{{j}} \Vert_{2}^{2} = \nu_{j2}^{2} = \ldots = \nu_{jk}^{2} = \Vert \bm{\beta}_{{j}} \Vert_{2}^{2/k}$. The result follows.
\end{proof}
\vspace{-0.0cm}

\subsection{Proof of Corollary~\ref{cor:openness-ghppk}}\label{app:cor-openness-ghppk}

\begin{proof}
First, we note that local openness is preserved under composition. Given two maps $\K_1:\mathcal{M}\to\mathcal{N}$ and $\K_2:\mathcal{N}\to\mathcal{O}$ between (Cartesian products of) Euclidean spaces, if $\K_1$ is open at $m\in\mathcal{M}$, and $\K_2$ is open at $n=\K_1(m)\in\mathcal{N}$, then $\K_2 \circ \K_1$ is open at $m$.
\vspace{0.1cm}

\noindent To obtain points of local openness of $\K(\bm{u}, \bm{\nu}_1,\ldots,\bm{\nu}_{k-1})=\bu\odotg \bnu_{r}^{\odot (k-1)}$, we utilize the preservation of local openness under composition by reducing $\K$ to a composition involving two parametrizations of which we already know the points of openness, the $\text{HPP}_k$, and the $\text{GHPP}$. Specifically, we express $\K$ as the composition $\K=\text{GHPP}\circ \K_{\bu,\bnu}$, where $\K_{\bu,\bnu}$ is an auxiliary \textit{globally} open map constructed in the following.\\
We have that $\K_{\bnu}:\prod_{r=1}^{k-1}\mathbb{R}^{L}\to\mathbb{R}^{L},\, (\bnu_1,\ldots,\bnu_{k-1})\mapsto \bnu_{r}^{\odot (k-1)}$, is globally open due to Lemma~\ref{lemma:proof-hppk}, as it can be recognized to be the $\text{HPP}_{k-1}$ for vectors in $\mathbb{R}^L$. Then, we can define the identity-augmented map
\begin{equation*}
  \K_{\bu,\bnu}:\Rd \times \prod_{r=1}^{k-1}\mathbb{R}^{L}\to \Rd \times \mathbb{R}^{L}, (\bm{u}, \bm{\nu}_2,\ldots,\bm{\nu}_{k}) \mapsto(\bu,\K_{\bnu}(\bnu_{1},\ldots,\bnu_{k-1}))  
\end{equation*}
that is simply the Cartesian product function of $\K_{\bnu}$ and the identity function $\text{id}_{\bu}: \Rd \to \Rd,\bu \mapsto \bu$, and maps inputs for the $\text{GHPP}_k$ to the input domain of the $\text{GHPP}$ by adding an independent extra entry $\bu$ and multiplying the remaining $k-1$ inputs $\bnu_{r} \in \mathbb{R}^{L}$ element-wise to obtain a single $\bnu \in \mathbb{R}^L$, so the image of $(\bm{u}, \bm{\nu}_2,\ldots,\bm{\nu}_{k})$ under $\K_{\bu,\bnu}$ is $(\bu,\bnu) \in \Rd \times \mathbb{R}^L$. Since local openness is trivially preserved for an identity-augmented Cartesian product map, and $\K_{\bnu}$ is globally open, $\K_{\bu,\bnu}$ is also globally open. Due to the composability property, $\K$ is thus locally open at $(\bm{u}, \bm{\nu}_2,\ldots,\bm{\nu}_{k})$ if the GHPP is locally open at $(\bu,\K_{\bnu}(\bnu_{1},\ldots,\bnu_{k-1}))$ by Lemma~\ref{lemma:openness-ghpp}.
\end{proof}
\vspace{-0.0cm}

\begin{remark}[Points of Openness for the $\text{GHPP}_{k_1, k_1+k_2}$]
To establish preservation of local minima under smooth parametrization of $\bbeta$ using $\text{GHPP}_{k_1,k_1+k_2}$ (\ref{eq:ghppk1k-def}) in general objectives $\P(\bpsi,\bbeta)$ using Lemma~\ref{lemma:pk_to_p}, we can make essentially the same line of arguments as in the previous result regarding the points of openness for the $\text{GHPP}_k$.
First, we define nested parametrizations $\K_{\bu}$ and $\K_{\bnu}$, both of which are globally open maps since they correspond to a $\text{HPP}_k$ mapping 
with depths $k_1$ and $k_2$, respectively. These are combined in the globally open pre-composition $\K_{\bu,\bnu}\triangleq(\K_{\bu}(\bm{\mu}_1,\ldots,\bm{\mu}_{k_1}),\K_{\bnu}(\bnu_1,\ldots,\bnu_{k_2}))$. This allows us to express $\K$ as the composition $\K=\text{GHPP}\circ\K_{\bu,\bnu}$. By the preservation of local openness under composition, if points in the domain of $\K$ are such that the conditions for local openness of the GHPP in Lemma~\ref{lemma:openness-ghpp} apply for $\K_{\bu}$ and $\K_{\bnu}$,
then $\K$ is also locally open at that point. %
\end{remark}

\subsection{Proof of Lemma~\ref{lemma:S_ghppk1k-def}}\label{app:lemma:S_ghppk1k-def}
\begin{proof}
This proof requires a simple weighted generalization of the AM-GM inequality:
\begin{proposition}[Weighted AM-GM inequality] \label{prop-wamgm}
Let $n\in\mathbb{N}$, $x_1,\ldots,x_n$ non-negative real values, $w_1,\ldots,w_n$ non-negative real weights, and $w\triangleq\sum_{i=1}^{n} w_i$. Then\\
\begin{equation}
\frac{w_{1} x_{1}+w_{2} x_{2}+\cdots+w_{n} x_{n}}{w} \geq \sqrt[w]{x_{1}^{w_{1}} x_{2}^{w_{2}} \cdots x_{n}^{w_{n}}} \,, \nonumber
\end{equation}
with equality holding if and only if $x_1=\ldots=x_n$.
\end{proposition}
We proceed in two steps. First, the AM-GM inequality is applied to the $k_1$ squared parameters $\mu_{jti}$ present in the parametrization of a single scalar entry $\beta_{ji}$ of $\bbeta$, $i \in \G_j$, for each $j=1,\ldots,L$. From this, the minimum of $\sum_{t=1}^{k_1} \Vert \bm{\mu}_{jt} \Vert_{2}^{2}$ as a function of the auxiliary parameter $\bu_j$ can be inferred. In the second step, the weighted AM-GM inequality is used to obtain the minimum of the overall regularization term:
\scalebox{0.92}{
  \begin{minipage}{\linewidth}
\begin{align}
    \hspace{-0.2cm}\sum_{j=1}^{L} \Big(\sum_{t=1}^{k_1} \Vert \bm{\mu}_{jt} \Vert_{2}^{2} + \sum_{r=1}^{k_2}  \nu_{jr}^{2} \Big) &= \sum_{j=1}^{L} \big( \sum_{i \in \mathcal{G}_{j}} \sum_{t=1}^{k_1} \mu_{jti}^{2} + \sum_{r=1}^{k_2}  \nu_{jr}^{2} \big) \explainup{\geq}{\tiny{(i)}} \sum_{j=1}^{L} \Big( \sum_{i \in \mathcal{G}_{j}} k_1 \big( \prod_{t=1}^{k_1} \mu_{jti}^{2} \big)^{1/k_1} + \sum_{r=1}^{k_2}  \nu_{jr}^{2} \Big)  \nonumber\\
    &\hspace{-1.2cm}= \sum_{j=1}^{L} \Bigg(  k_1 \sum_{i \in \mathcal{G}_{j}} \big| \underbrace{\prod_{t=1}^{k_1} \mu_{jti}}_{u_{ji}} \big|^{2/k_1} + \sum_{r=1}^{k_2}  \nu_{jr}^{2} \Bigg)  \nonumber = \sum_{j=1}^{L} \Bigg(  k_1 \Vert \bm{u}_{{j}} \Vert_{2/k_1}^{2/k_1} + \sum_{r=1}^{k_2}  \nu_{jr}^{2} \Bigg)  \nonumber \\
    &\hspace{-1.2cm}\explainup{\geq}{\tiny{(ii)}} \sum_{j=1}^{L} (k_1+k_2) \left[ \left(\Vert \bm{u}_{{j}} \Vert_{2/k_1}^{2/k_1}\right)^{k_1} \cdot  \prod_{r=1}^{k_2}  \nu_{jr}^{2} \right]^{{\tiny 1/(k_1+k_2)}}  \nonumber \hspace{-0.35cm}= \sum_{j=1}^{L} k \bigl| \Vert \bm{u}_{{j}} \Vert_{2/k_1} \cdot  \prod_{r=1}^{k_2}  \nu_{jr} \bigr|^{2/k}  \nonumber \\
    &\hspace{-1.2cm}= \sum_{j=1}^{L} k  \big\Vert \bm{u}_{{j}} \cdot  \prod_{r=1}^{k_2}  \nu_{jr} \big\Vert_{2/k_1}^{2/k} = k \sum_{j=1}^{L} \Vert \bm{\beta}_{{j}} \Vert_{2/k_1}^{2/k} = k \Vert \bm{\beta} \Vert_{2/k_1,2/k}^{2/k}  \nonumber
\end{align}
\end{minipage}}

\noindent The first inequality $(i)$ using the AM-GM inequality holds with equality if and only if $\mu_{jti}^{2} = |u_{ji}|^{2/k_1}\; \forall \; t=1,\ldots,k_1,\; i \in \mathcal{G}_{j}$ and $j=1,\ldots,L$. The second inequality $(ii)$ applies Proposition~\ref{prop-wamgm} and reduces to equality if and only if $\Vert \bm{u}_{j} \Vert_{2/k_1}^{2/k_1} = \nu_{j1}^{2} = \ldots = \nu_{j k_2}^{2} = \Vert \bm{\beta}_{{j}} \Vert_{2/k_1}^{2/k} \;\forall\; j=1,\ldots,L$.   
\end{proof}
\vspace{-0.0cm}
\subsection{Proof of Lemma~\ref{lemma:S_hppk-shared-def}}\label{app:proof-hppk-shared}
\begin{proof} We apply the AM-GM inequality to each summand $j=1,\ldots,d$ of the surrogate penalty $\Rxi$:
{\small
\begin{align}\label{eq:had-shared-modified-l2}
    \frac{\norm{\bu}_2^2+(k-1)\norm{\bv}_2^2}{k} =\sum_{j=1}^{d} \frac{u_j^2 + (k-1) v_j^2}{k} &\geq \sum_{j=1}^{d} {\textstyle \sqrt[k]{u_j^2 \prod_{l=1}^{k-1} v_j^2}} = \sum_{j=1}^{d} \sqrt[k]{\left(u_j \cdot v_j^{k-1} \right)^2} \nonumber\\
    &= \sum_{j=1}^{d} |\beta_j|^{2/k} = \norm{\bbeta}_{2/k}^{2/k} \,, \nonumber
\end{align}
}
with equality holding if and only if $u_j^2 = v_j^2 = |\beta_j|^{2/k}$ for all $j=1,\ldots,d$.
\end{proof}
\vspace{-0.0cm}

\subsection{Parameter sharing and identical initialization}\label{app:ident-init}
In (S)GD, dynamics with shared Hadamard factors can be related to their fully overparametrized counterparts through identical initialization of the to-be-shared parameters. 
We define a differentiable surrogate $\Q$ based on the $\text{HPP}_k$, i.e., $\bbeta=\bu_{l}^{\odot k}$, with surrogate $\ell_2$ regularization for $\Rxi$ and no additional unregularized parameters $\bpsi$:
{\small \begin{align}\Q(\bu_1,\ldots,\bu_k)=\L(\K(\bu_1,\ldots,\bu_k))+\lambda\Rxi(\bu_1,\ldots,\bu_k)\,. \nonumber
\end{align}
}

Consider an updating scheme for the $\bu_l$, given by $\bu_{l}^{t+1}=\bu_{l}^{t}-\alpha \nabla_{\bu_{l}} \Q(\bu_1^{t},\ldots,\bu_k^{t})$, where $\alpha$ denotes the learning rate. Assume identical initialization for $k-1$ factors, i.e., $\bu_1^0=\tilde{\bu}$ and $\bu_2^0=\ldots=\bu_k^0=\tilde{\bv}$. Then we have for $l=1,\ldots,k$,
{\small \begin{equation}
\nabla_{\bu_{l}} \Q(\bu_1,\ldots,\bu_k) = \big(\partial \bbeta/ \partial \bu_l\big)^{\top} \nabla_{\bbeta} \L(\bbeta) + \lambda \nabla_{\bu_l} \Rxi(\cdot) = \textnormal{diag} \big( {\textstyle \bigodot_{l'\in[k]\setminus\{l\}}} \bu_{l'}\big) \nabla_{\bbeta} \L(\bbeta) + 2\lambda \bu_l\,, \nonumber
\end{equation}}

\noindent where $\big(\partial \bbeta/ \partial \bu_l\big)$ is a $d \times d$ matrix containing partial derivatives $(\partial \beta_i / \partial u_{jl})_{ij}$, $i,j \in [d]$. At initialization, the gradients of $\Q$ with respect to the $\bu_l$ are given by
%
$\nabla_{\bu_{1}} \Q(\bu_1^0,\ldots,\bu_k^0) =  \textnormal{diag}\big(\tilde{\bv}^{k-1}\big) \nabla_{\bbeta} \L(\bbeta) + 2\lambda \tilde{\bu}$ and 
$\nabla_{\bu_{l}} \Q(\bu_1^0,\ldots,\bu_k^0) = \textnormal{diag} \big(\tilde{\bu}\odot \tilde{\bv}^{k-2}\big) \nabla_{\bbeta} \L(\bbeta) + 2\lambda \tilde{\bv}$, 
where $l=2,\ldots,k$. Note that the gradient is constant over the identically initialized factors. It thus follows from the updating rule that 
$\bu_2^{t}=\ldots=\bu_k^{t}\,\forall t\in\mathbb{N}$.\\ Compare this to the gradient of an alternative surrogate $\tilde{\Q}$ based on the shared parametrization $\tilde{\K}(\bu,\bv)=\bu\odot\bv^{k-1}$, with initialization $(\bu^0,\bv^0)=(\tilde{\bu},\tilde{\bv})$ and penalty $\tilde{\Rxi}=\norm{\bu}_2^2+(k-1)\norm{\bv}_2^2$. It is easy to see that $\nabla_{\bu}\tilde{\Q}(\bu,\bv)=\nabla_{\bu_1}\Q(\bu_1,\ldots,\bu_k)$, and under identical initialization for $l=2,\ldots,k$, we have $\nabla_{\bv}\tilde{\Q}(\bu,\bv)=(k-1)\nabla_{\bu_l}\Q(\bu_1,\ldots,\bu_k)$. Therefore, updating $\bu^{t+1}=\bu^{t}-\alpha \nabla_{\bu}\tilde{\Q}(\bu^{t},\bv^{t})$ and $\bv^{t+1}=\bv^{t}-\frac{\alpha}{k-1}\nabla_{\bv}\tilde{\Q}(\bu^{t},\bv^{t})$, using a scaled learning rate $\frac{\alpha}{k-1}$ for $\bv$, results in identical updates compared to running gradient descent on $\Q$ with identical initialization for the $k-1$ (shared) factors.
\vspace{-0.25cm}
\subsection{Proof of Lemma~\ref{lemma:S_hpowp}}\label{app:proof-hpowp}
\begin{proof} We apply the weighted AM-GM inequality to each summand $j\in[d]$ of $\Rxi$:
{\small
\begin{align}
    \frac{\norm{\bu}_2^2+(k-1)\norm{|\bv|}_2^2}{k} &=\sum_{j=1}^{d} \frac{u_j^2 + (k-1) |v_j|^2}{k} \geq \sum_{j=1}^{d} \sqrt[k]{u_j^2 \left( |v_j|^{2} \right)^{k-1}} = \sum_{j=1}^{d} \sqrt[k]{\left(u_j \cdot |v_j|^{k-1} \right)^2} \nonumber \\
    &= \sum_{j=1}^{d} \big( |\smash[b]{\underbrace{u_j \cdot |v_j|^{k-1}}_{=\beta_j}}|\big)^{2/k} =  \sum_{j=1}^{d} |\beta_j|^{2/k} = \norm{\bbeta}_{2/k}^{2/k} \,, \nonumber
\end{align}
}
with equality holding if and only if $u_j^2 = |v_j|^2 = |\beta_j|^{2/k}$ for all $j=1,\ldots,d$.
\end{proof}
\vspace{-0.0cm}

\subsection{Proof of Lemma~\ref{lemma:S_ghpowp}}\label{app:lemma:S_ghpowp}
\begin{proof} We again apply the weighted AM-GM inequality on the group level for each $j\in[L]$ of the surrogate penalty $\Rxi$ and find
\scalebox{0.96}{
  \begin{minipage}{\linewidth}
\begin{align} \nonumber
    \frac{\Vert \bm{u} \Vert_{2}^{2} + (k-1) \Vert \bm{\nu} \Vert_{2}^{2}}{k} &= \textstyle\sum_{j=1}^{L} \frac{\Vert \bm{u}_{{j}} \Vert_{2}^{2} + (k-1)|\nu_j|^{2}}{k} \geq \textstyle\sum_{j=1}^{L} \left( \Vert \bm{u}_{j} \Vert_{2}^{2} \cdot \left( |\nu_j|^{2} \right)^{k-1} \right)^{1/k} \\
    &= \textstyle\sum_{j=1}^{L} \left(\Vert \bm{u}_{j} \Vert_{2} \cdot |\nu_j|^{k-1} \right)^{2/k} = \textstyle\sum_{j=1}^{L} \Vert |\nu_j|^{k-1} \cdot \bm{u}_{j} \Vert_{2}^{2/k} = \textstyle\sum_{j=1}^{L} \Vert \bm{\beta}_{{j}} \Vert_{2}^{2/k} 
    \,, \nonumber
\end{align}
\end{minipage}}

with equality holding if and only if $\Vert \bm{u}_{{j}} \Vert_{2}^{2} = |\nu_j|^{2} = \Vert \bm{\beta}_{{j}} \Vert_{2}^{2/k}\; \forall j = 1,\ldots,L$.
\end{proof}
\vspace{-0.05cm}

\subsection{Proof of Lemma~\ref{lemma:S_ghpowpk1}}\label{app:lemma:S_ghpowpk1}
\begin{proof} Applying the weighted AM-GM inequality to the surrogate regularizer $\Rxi(\bm{\mu},\bnu)$ on the group-level for each $j \in [L]$, we find
{\small
\begin{align} \nonumber
    \frac{k_1 \Vert \bm{\mu} \Vert_{2}^{2} + k_2 \Vert \bm{\nu} \Vert_{2}^{2}}{k} &= \sum_{j=1}^{L} \frac{k_1 \Vert \bm{\mu}_{j} \Vert_{2}^{2} + k_2|\nu_j|^{2}}{k} = \sum_{j=1}^{L} \frac{k_1 \Vert \bm{u}_{j} \Vert_{2/k_1}^{2/k_1} + k_2 |\nu_j|^{2}}{k}  \\ 
    &\geq \sum_{j=1}^{L} \left( \left( \Vert \bm{u}_{j} \Vert_{2/k_1}^{2/k_1} \right)^{k_1} \cdot \left( |\nu_j|^{2} \right)^{k_2} \right)^{1/(k_1+k_2)} = \sum_{j=1}^{L} \left( \Big| \Vert \bm{u}_{j} \Vert_{2/k_1} \cdot |\nu_{j}|^{k_2} \Big| \right)^{2/k} \nonumber \\ 
    &= \sum_{j=1}^{L} \Vert \bm{u}_{j} \cdot |\nu_{j}|^{k_2} \Vert_{2/k_1}^{2/k} = \sum_{j=1}^{L} \Vert \bm{\beta}_{{j}} \Vert_{2/k_1}^{2/k} = \Vert \bm{\beta} \Vert_{2/k_1,2/k}^{2/k} \nonumber \,,
\end{align}
}
with equality holding if and only if $\Vert \bm{\mu}_{j} \Vert_{2}^{2} = |\nu_j|^{2} = \Vert \bm{\beta}_{{j}} \Vert_{2/k_1}^{2/k}\; \forall j = 1,\ldots,L$.
\end{proof}
\vspace{-0.0cm}
%
%

\subsection{Fibers and structure of product  parametrizations}\label{app:connected-fibers}

The fibers or level sets $\K^{-1}(\bbeta)$ of the parametrizations considered in our work (Assumption~\ref{ass-parametrization-map}) are well-behaved and exhibit regularity properties worth discussing. Let $\bbeta \in \Rd$ be a parameter that can be partitioned into $(\bbeta_1,\ldots,\bbeta_L)$ for $j \in [L]$, $L \leq d$, so that $\bbeta_j \in \mathbb{R}^{|\mathcal{G}_j|}$ and $|\mathcal{G}_1|+\ldots+ |\mathcal{G}_L| = d$. Further, let $\K: \Rdxi \to \Rd$ be a $\mathcal{C}^1$-smooth surjective parametrization of $\bbeta$. For product and power-type structures such as the $\text{HPP}_k$, $\text{GHPP}$, or the $\text{GHPowP}_k$, the following holds: All $\bbeta_j \in \mathbb{R}^{|\mathcal{G}_j|} \setminus \{\bm{0}\}$ are \textit{regular} values of $\K_j(\bm{\xi}_j)$ for each $j \in [L]$, i.e., the Jacobian $\mathcal{J}_{\K_j}(\bm{\xi}_j)$ has full row rank $|\mathcal{G}_j|$ for all $\bm{\xi}_j \in \K_j^{-1}(\bbeta_j)$. In the following, we derive this for three exemplary parametrizations:

\begin{example}[Regularity of $\operatorname{HPP}_k$]
    As the canonical multiplicative parametrization, consider the $\text{HPP}_k$. The Jacobian $\mathcal{J}_{\K_j}(\bm{\xi}_j)$ has full row rank $|\mathcal{G}_j|=1$ at all $\bm{\xi}_j \in \K_j^{-1}(\beta_j)$ with $\beta_j \ne 0$. This follows from $\beta_{j} = \prod_{l=1}^k \xi_{jl}$: each entry $\beta_{j}$ depends only on the $k$ factors $(\xi_{jl})_{l \in [k]}$, and its gradient is nonzero whenever $\beta_{j} \ne 0$, implying $\xi_{jl}\neq0 \, \forall \, l \in [k]$. The partial derivatives $\partial \beta_j/ \partial \xi_{jl}$ are themselves non-zero products of the remaining factors, $\J_{\K_j}(\bxi_j)=[\prod_{l \neq 1} \xi_{jl}, \ldots, \prod_{l \neq k} \xi_{jl}] \in \mathbb{R}^{1 \times k}$. Hence, the Jacobian of each $\K_j$ for the $\text{HPP}_k$ has full row-rank for all $\bxi_j \notin \K_j^{-1}(0)$. 
\end{example}
\begin{example}[Regularity of $\text{GHPP}$]
    The $\operatorname{GHPP}$ parametrizes $\bbeta$ as $\K(\bxi) = \bu \odotg \bnu$, or group-wise $\K_j(\bxi_j) = \bu_j \cdot \nu_j$ with $\bu_j \in \mathbb{R}^{|\Gj|}$ and $\nu_j \in \mathbb{R}$. The Jacobian $\mathcal{J}_{\K_j}(\bxi_j) \in \mathbb{R}^{|\Gj| \times (|\Gj|+1)}$ with respect to $\bxi_j = (\bu_j, \nu_j)$ is given by
    \[
    \mathcal{J}_{\K_j}(\bxi_j) =
    \begin{bmatrix}
    \nu_j & 0 & \cdots & 0 & u_{j1} \\
    0 & \nu_j & \cdots & 0 & u_{j2} \\
    \vdots & \vdots & \ddots & \vdots & \vdots \\
    0 & 0 & \cdots & \nu_j & u_{j|\Gj|}
    \end{bmatrix}.
    \]
    $\J_{\K_j}(\bxi_j)$ has full row rank $|\Gj|$ whenever $\nu_j \neq 0$, since the diagonal entries $\nu_j$ are nonzero and span the row space. Thus, all $\bbeta_j \neq \bm{0}$ are regular values of $\K_j$, as $\bbeta_j = \bu_j \cdot \nu_j$ implies $\nu_j \neq 0$ whenever $\bbeta_j \neq \bm{0}$. Note that the diagonal part is always of full rank as long as $\nu_j \neq 0$, even if some entries $u_{ji} = 0$.
\end{example}

\begin{example}[Regularity of $\text{GHPP}_k$]
    The $\operatorname{GHPowP}_k$ parametrizes $\bbeta$ as $\K(\bxi) = \bu \odotg |\bnu|^{\circ(k-1)}$, or group-wise $\K_j(\bxi_j)= \bu_j \cdot |\nu_j|^{k-1}\,,\,k>2, \bu_j \in \mathbb{R}^{|\Gj|}, \nu_j \in \mathbb{R}$. The Jacobian $\mathcal{J}_{\K_j}(\bxi_j) \in \mathbb{R}^{|\Gj| \times (|\Gj|+1)}$ with respect to $\bxi_j = (\bu_j, \nu_j)$ is given by
    \[
    \mathcal{J}_{\K_j}(\bxi_j) =
    \begin{bmatrix}
    |\nu_j|^{k-1} & & & & u_{j1}(k-1)\operatorname{sign}(\nu_j)|\nu_j|^{k-2} \\
    & |\nu_j|^{k-1} & & & u_{j2}(k-1)\operatorname{sign}(\nu_j)|\nu_j|^{k-2} \\
    & & \ddots & & \vdots \\
    & & & |\nu_j|^{k-1} & u_{j|\Gj|}(k-1)\operatorname{sign}(\nu_j)|\nu_j|^{k-2}
    \end{bmatrix}.
    \]
    Hence $\J_{\K_j}(\bxi_j)$ has row rank $|\Gj|$ as long as $\nu_j \neq 0$, similar to the $\operatorname{GHPP}$, and all $\bbeta_j \neq \bm{0}$ are regular values of $\K_j$.
\end{example}

Further, the $\K_j(\bxi_j)$ considered in our work are positively homogeneous of degree $k\geq 2$, i.e., $\K_j(c\, \bxi_j)=c^k\, \K_j(\bxi_j)$ for all $\bxi_j$ and $c>0$. In the context of the fibers of $\K_j$, the positive homogeneity implies scale-invariance of the shape of each individual fiber at the regular values $\bbeta_j$. That is, a fiber associated with a regular value $\bbeta_j$ has the same ``shape'' as the fiber of any scaled version of $\bbeta_j$. Moreover, continuity of $\K_j$ ensures that the fibers $\K_j^{-1}(\bbeta_j)$ are closed sets. However, they are not compact in the case of overparametrization, since the product structure of the $\K_j$ implies unbounded fibers that extend to infinity. The following well-known result in differential geometry can be used to characterize the fibers of $\K_j$ at regular values, which is a consequence of the inverse function theorem and is also known as the preimage theorem \citep[e.g., Theorem 3.2 in][]{hirsch2012differential}:\\

\begin{proposition}[Regular value Theorem]\label{prop-preimage-theorem}
Let $\K: \Rdxi \to \Rd,\,\bxi \mapsto \bbeta$, be a $\mathcal{C}^r, r \geq 1,$ function and let $\bbeta \in \Rd$ be a regular value of $\K$. Then the fiber $\mathcal{K}^{-1}(\bbeta)$ is a $(\dxi - d)$-dimensional $\mathcal{C}^r$ submanifold of $\Rdxi$.
\end{proposition}

Thus, the fiber of $\K_j$ at a regular value $\bbeta_j$ forms a $(d_{\bxi_j} - |\mathcal{G}_{j}|)$-dimensional smooth manifold. For product-like parametrizations $\K$ with sign-flip symmetries, this manifold consists of disjoint connected components, each embedded within an orthant of suitable sign configuration. For any non-zero $\bbeta_j$,  each connected component of the fiber $\K_j^{-1}(\bbeta_j)$ contains a unique minimal-norm point. 
By the AM-GM inequality, this point is attained when the Euclidean norms of the factors are balanced. Figure~\ref{fig:tikz-contours-hpp} visualizes the disjoint components (branches of the hyperbola) and their minimal-norm points for the scalar HPP.\\
As we traverse the fiber away from the minimal-norm points, i.e., the more unbalanced the factorization of $\bbeta_j$ becomes, the fiber exhibits increasing curvature. Figure~\ref{fig:hpc-tangent-point} illustrates such a fiber for a scalar-valued $\text{HPP}_k$ parametrization of depth $k=3$ (\ref{eq:hppk}).\\
In addition, we can prove the local connectedness of the fibers of $\K$ even at the non-regular value $\bm{0}$, where the Jacobian degenerates to the null matrix for product-type parametrizations $\K$:
\begin{lemma}[Local Connectedness of fibers of $\K$ for product-type parametrizations]\label{lemma:app-connectedness-fibers}
Under Assumption~\ref{ass-parametrization-map} and assuming all $\bbeta_j \neq \bm{0}$ are regular values of the $\K_j, j\in [L]$, the fibers of $\K: \Rdxi \to \Rd$ are (locally) connected sets at every $\bbeta \in \Rd $. 
\end{lemma}
\begin{proof}
By Assumption~\ref{ass-parametrization-map}, $\K(\bxi)$ is block-separable into parametrizations $\K_j(\bxi_j)$ and further assume the $|\mathcal{G}_j| \times \dxi_j$-dimensional Jacobian of each $\K_j$ has full row rank for all $\bxi_j \not\in \K_j^{-1}(\bm{0})$.
Therefore, the fibers of $\K_j$ at non-zero $\bbeta_j \in \mathbb{R}^{|\mathcal{G}_j|}$ are $\mathcal{C}^r$ manifolds by Proposition~\ref{prop-preimage-theorem}, and thus locally connected sets whenever $\K_j$ is regular. Further, by the product-power structure assumed for $\K_j$, the parametrization $\K_j(\bxi_j)$ maps to the non-regular value $\bbeta_j = \bm{0}$ if at least one of its factors is zero. Thus, the fiber $\K_j^{-1}(\bm{0})$ contains all $\bxi_{j1},\ldots,\bxi_{jk}$ such that at least one $\bxi_{jl} = \bm{0}$. This fiber, while not a manifold, is thus a connected set, since all of the contained hyperplanes intersect at $\bxi_{j1}=\ldots=\bxi_{jk}=\bm{0}$.

\noindent Considering the separable Cartesian product structure of $\K$, the fiber of $\K$ at $\bbeta$ is a Cartesian product of the fibers of $\K_j$ at their respective values $\bbeta_j$. Each of these fibers is either locally connected (for regular values of $\bbeta_j$) or connected (for $\bbeta_j=\bm{0}$), and thus the fiber of $\K$ at $\bbeta$ is also locally connected.
\end{proof}
\vspace{-0.0cm}


\subsection{Differentiable SCAD, MCP, and TL1 via the HPP}\label{app:hpp-tl1-scad-mcp}

For the popular non-convex regularizers SCAD, MCP, and the transformed $\ell_1$ (TL1) penalty, we can derive smooth surrogates based on the HPP and a differentiable surrogate penalty $\Rxi$, replacing absolute value terms in $\Rbeta$ by their variational quadratic formulation. All three regularizers are defined as separable functions of the entry-wise absolute values, to each of which we can apply the surrogate $(u_j^2+v_j^2)/2 \geq |\beta_j|$. In the following, we derive only smooth variational forms of those regularizers; the equivalence of the resulting overparametrized optimization problems follows from the fact that the solution map $\hbxi$ coincides with that for the HPP with induced $\ell_1$ regularization. Thus, Theorem~\ref{theorem-general} can be applied directly.

\textbf{TL1 penalty}\, The transformed $\ell_1$ penalty is a non-convex regularizer defined as
\begin{equation*}\label{eq:def-tl1}
    \Rbeta^{TL1}(\bbeta) \triangleq \sum_{j=1}^d \frac{(a+1)|\beta_j|}{a+|\beta_j|}\,,\quad a>0\,,
\end{equation*}
where the hyperparameter $a$ steers the degree of non-convexity and interpolates between the $\ell_0$ penalty (for $a \to 0$) and the $\ell_1$ penalty (for $a \to \infty$). The following result provides an SVF of $\Rbeta^{TL1}$ using the HPP, thereby enabling the construction of differentiable equivalent surrogate objectives.

\begin{lemma}[SVF for the TL1 penalty]\label{lemma:svf-tl1}
Given the HPP $\K(\bu,\bv)=\bu \odot \bv = \bbeta$, the minimum of the surrogate penalty
\begin{equation*}
    \Rxi^{TL1}(\bu,\bv) = \sum_{j=1}^d \frac{(a+1)(u_j^2+v_j^2)}{2a+(u_j^2+v_j^2)}\,,\quad a>0,
\end{equation*}
subject to $\K(\bu,\bv)=\bbeta$ constitutes the following SVF for the TL1 penalty $\Rbeta^{TL1}(\bbeta)$:
\begin{equation*}
    \underset{\bu,\bv \in \Rd: \bu \odot \bv = \bbeta}{\min} \, \Rxi^{TL1}(\bu,\bv) = \Rbeta^{TL1}(\bbeta)
\end{equation*}
\end{lemma}

\begin{proof}
    Since both $\Rbeta^{TL1}$ and $\Rxi^{TL1}$ are coordinate-wise separable, it suffices to show the equality for a generic summand $j \in [d]$, and the result immediately follows for all summands. First, note that the scalar function $h(t)=\frac{(a+1)t}{2a+t}$ is strictly increasing on $[0,\infty)$ for all $a>0$. Therefore, minimizing the corresponding summand of $\Rxi^{TL1}$ subject to $u_j v_j = \beta_j$ is equivalent to minimizing $(u_j^2+v_j^2)$ subject to the same constraint. Using the AM-GM inequality, we obtain a constrained minimum of $2|\beta_j|$ attained if and only if $|u_j|=|v_j|=\sqrt{|\beta_j|}$ and $\operatorname{sign}(u_jv_j)=\operatorname{sign}(\beta_j)$. Thus,
    \begin{align*}
        \underset{u_j, v_j: u_j v_j = \beta_j}{\min} \frac{(a+1)(u_j^2+v_j^2)}{2a+(u_j^2+v_j^2)} = \frac{(a+1)2|\beta_j|}{2a+2|\beta_j|} = \frac{(a+1)|\beta_j|}{a+|\beta_j|} \,.
    \end{align*}
    Summing over $j \in [d]$, we obtain the SVF
    \begin{equation*}
    \underset{\bu,\bv \in \Rd: \bu \odot \bv = \bbeta}{\min} \, \Rxi^{TL1}(\bu,\bv) = \Rbeta^{TL1}(\bbeta).
\end{equation*}
\end{proof}

It follows that TL1 regularized objectives of the form $\P(\bpsi,\bbeta)=\L(\bpsi,\bbeta)+\lambda \Rbeta^{TL1}(\bbeta)$ are equivalent to the smooth surrogate $\Q(\bpsi,\bu,\bv)=\L(\bpsi,\bu \odot \bv) + \lambda \Rxi^{TL1}(\bu,\bv)$.

\textbf{Minimax Concave Penalty}\, The MCP is a non-convex separable regularizer with a piecewise definition, relaxing the penalization rate to 0 away from the origin. In its non-smooth formulation, it is given by
\begin{equation*}
\Rbeta^{MCP}(\bbeta) \triangleq \sum_{j=1}^{d} \rho^{MCP}_{\lambda, \gamma}(\beta_j)\,, \quad \rho^{MCP}_{\lambda, \gamma}(\beta_j)= \begin{cases}\lambda|\beta_j|-\frac{\beta_j^2}{2 \gamma}, & |\beta_j| \leq \gamma \lambda, \\ \frac{1}{2} \gamma \lambda^2, & |\beta_j|>\gamma \lambda ,\end{cases}
\end{equation*}
where $\gamma>1$ controls the degree of non-convexity. A fully differentiable surrogate of $\Rbeta^{MCP}$ can be obtained by replacing $|\beta_j|$ by $(u_j^2+v_j^2)/2$ and $\beta_j^2$ by $(u_j^2+v_j^2)^2/4$. This yields the following result:

\begin{lemma}[SVF for the MCP]\label{lemma:svf-mcp}
Given $\K(\bu,\bv)=\bu \odot \bv = \bbeta$ (HPP), the minimum of
\begin{equation*}
\resizebox{0.95\textwidth}{!}{$\displaystyle
\Rxi(\bu,\bv) \triangleq \sum_{j=1}^d \tilde{\rho}_{\lambda,\gamma}^{MCP}(u_j,v_j), \quad
\tilde{\rho}_{\lambda,\gamma}^{MCP}(u_j,v_j) \triangleq
\begin{cases}
\lambda (u_j^2+v_j^2)/2 - \dfrac{((u_j^2+v_j^2)/2)^2}{2\gamma},
& (u_j^2+v_j^2)/2 \le \gamma\lambda, \\[0.5em]
\dfrac{1}{2}\gamma\lambda^2,
& (u_j^2+v_j^2)/2 > \gamma\lambda .
\end{cases}
$}
\end{equation*}

subject to $\K(\bu,\bv)=\bbeta$ forms an SVF for the MCP regularizer:
\begin{equation*}
    \underset{\bu,\bv \in \Rd: \bu \odot \bv = \bbeta}{\min} \, \Rxi(\bu,\bv) = \Rbeta^{MCP}(\bbeta)
\end{equation*}
\end{lemma}

%

\begin{proof}
Fix $j\in[d]$ and set $s_j\triangleq (u_j^2+v_j^2)/2$. On the first, non-constant, branch of $\tilde{\rho}_{\lambda,\gamma}^{MCP}$,
\begin{equation*}
\tilde{\rho}_{\lambda,\gamma}^{MCP}(u_j,v_j)=\lambda s_j-\frac{s_j^2}{2\gamma}=h(s_j),
\qquad
h(t)\triangleq \lambda t-\frac{t^2}{2\gamma},
\end{equation*}
and $h'(t)=\lambda-t/\gamma>0$ for $t\in[0,\gamma\lambda]$, so $h$ is strictly increasing there. Thus, whenever the first branch is feasible under $u_jv_j=\beta_j$, minimizing $\tilde{\rho}_{\lambda,\gamma}^{MCP}$ reduces to minimizing $s_j$ on the fiber.

By AM-GM, for all $(u_j,v_j)$ with $u_jv_j=\beta_j$, $s_j=(u_j^2+v_j^2)/2\ge |u_jv_j|=|\beta_j|$, and equality is attained for $|u_j|=|v_j|$ and correct signs for $\beta_j$.\\
If $|\beta_j|>\gamma\lambda$, then $s_j\ge|\beta_j|>\gamma\lambda$ for all feasible $(u_j,v_j)$, so only the second branch of $\tilde{\rho}_{\lambda,\gamma}^{MCP}$ applies and $\min_{u_jv_j=\beta_j}\tilde{\rho}_{\lambda,\gamma}^{MCP}(u_j,v_j)=\frac12\gamma\lambda^2$. \\
If $|\beta_j|\le\gamma\lambda$, the constrained minimizer of $s_j=(u_j^2+v_j^2)/2$ over $u_j v_j = \beta_j$ satisfies $s_j=|\beta_j|\leq \gamma\lambda$, so the first branch is feasible and, by monotonicity of $h$ on $[0,\gamma\lambda]$,
\begin{equation*}
\min_{u_jv_j=\beta_j}\tilde{\rho}_{\lambda,\gamma}^{MCP}(u_j,v_j)
=
h\!\big(\min_{u_jv_j=\beta_j}s_j\big)
=
\lambda|\beta_j|-\frac{\beta_j^2}{2\gamma}.
\end{equation*}
Moreover, $\frac12\gamma\lambda^2-\bigl(\lambda|\beta_j|-\frac{\beta_j^2}{2\gamma}\bigr)\ge0$, so the constant branch cannot improve upon this value. Therefore, the entry-wise minima coincide with $\rho^{MCP}_{\lambda,\gamma}(\beta_j)$, and summing over $j \in [d]$  gives the desired SVF
\begin{equation*}
\min_{\bu,\bv \in \Rd:\,\bu\odot\bv=\bbeta}\Rxi(\bu,\bv)=\Rbeta^{MCP}(\bbeta).
\end{equation*}

\end{proof}

Continuity at the boundary $s_j=\gamma\lambda$ follows from $\lambda(\gamma\lambda)-\frac{(\gamma\lambda)^2}{2\gamma}=\frac12\gamma\lambda^2$, which coincides with the second branch. Differentiability of $\tilde{\rho}_{\lambda,\gamma}^{MCP}$ with respect to $(u_j,v_j)$ at the boundary $s_j=\gamma\lambda$ follows by observing that the gradient of the first branch is $\nabla_{(u_j,v_j)}\bigl(\lambda s_j-\frac{s_j^2}{2\gamma}\bigr)
    =\bigl(\lambda-\frac{s_j}{\gamma}\bigr)\nabla_{(u_j,v_j)} s_j
    =\bigl(\lambda-\frac{s_j}{\gamma}\bigr)(u_j,v_j)$, which vanishes on $s_j=\gamma\lambda$, matching the zero gradient of the constant second branch.

%

\textbf{Smoothly Clipped Absolute Deviations}\, 
The SCAD penalty is a non-convex, separable regularizer that transitions from $\ell_1$ penalization near the origin to an unpenalized regime away from $0$. In its original non-smooth form, it is defined as
\begin{equation*}
\Rbeta^{SCAD}(\bbeta) \triangleq \sum_{j=1}^{d} \rho^{SCAD}_{\lambda,a}(\beta_j),
\end{equation*}
where $a>2$ is a hyperparameter and
\begin{equation*}
\rho^{SCAD}_{\lambda,a}(\beta_j)
=
\begin{cases}
\lambda|\beta_j|, & |\beta_j|\le \lambda,\\[0.3em]
\dfrac{-\beta_j^2+2a\lambda|\beta_j|-\lambda^2}{2(a-1)}, 
& \lambda<|\beta_j|\le a\lambda,\\[0.5em]
\dfrac{(a+1)\lambda^2}{2}, & |\beta_j|>a\lambda.
\end{cases}
\end{equation*}
As the construction of a smooth surrogate is completely analogous to the previous MCP derivation, only a brief description is provided here. The variational equivalence and boundary reduction follow the same arguments as in Lemma~\ref{lemma:svf-mcp}. The fully differentiable surrogate for $\Rbeta^{SCAD}$ is constructed via replacing $|\beta_j|$ by $(u_j^2+v_j^2)/2$ and $\beta_j^2$ by $(u_j^2+v_j^2)^2/4$, yielding a smooth, piecewise surrogate penalty. As in the MCP construction, we abbreviate $s_j\triangleq(u_j^2+v_j^2)/2$ and obtain the following smooth surrogate for each $\beta_j$:
\begin{equation*}
\resizebox{0.95\textwidth}{!}{$\displaystyle
\Rxi^{SCAD}(\bu,\bv)\triangleq \sum_{j=1}^d \tilde{\rho}^{SCAD}_{\lambda,a}(u_j,v_j)\,, \quad
\tilde{\rho}^{SCAD}_{\lambda,a}(u_j,v_j)\triangleq
\begin{cases}
\lambda s_j, & s_j\le \lambda,\\[0.4em]
\dfrac{-s_j^2+2a\lambda s_j-\lambda^2}{2(a-1)}, & \lambda<s_j\le a\lambda,\\[0.6em]
\dfrac{(a+1)\lambda^2}{2}, & s_j>a\lambda.
\end{cases}
$}
\end{equation*}
This yields an SVF of the SCAD penalty under the HPP:
\begin{equation*}
\underset{\bu,\bv \in \Rd: \bu \odot \bv = \bbeta}{\min}\Rxi^{SCAD}(\bu,\bv) = \Rbeta^{SCAD}(\bbeta) \,\,\, \forall\,\bbeta \in \Rd\,, 
\end{equation*}
which is continuously differentiable in $(\bu,\bv)$ and separable across coordinates. Using the same AM-GM arguments as for the MCP, the feasible values of $s_j$ under the constraint $u_jv_j=\beta_j$ satisfy $s_j \in [|\beta_j|,\infty)$, with the lower bound attained for balanced factorizations. Consequently, the branch conditions of $\tilde{\rho}_{\lambda,a}^{SCAD}(u_j,v_j)$ in terms of $s_j$ reduce to thresholds in $|\beta_j|$: if $|\beta_j|>a \lambda$, all feasible points $s_j$ lie on the constant branch. If $\lambda < |\beta_j| \leq a \lambda$, the minimal feasible $s_j$ lies in the middle branch. Finally, if $|\beta_j| \leq \lambda$, the minimal feasible $s_j$ lies in the first branch. Since $\tilde{\rho}^{SCAD}_{\lambda,a}$ is non-decreasing in $s_j$ on each piecewise branch, the constrained minimum over $u_jv_j=\beta_j$ is attained at the smallest feasible value $s_j=|\beta_j|$, recovering exactly the scalar SCAD penalty $\rho^{SCAD}_{\lambda,a}(\beta_j)$.

%
%
%
%

\section{Details on Numerical Experiments} \label{app:experiments}


Table~\ref{tab:overview-hyperparam} provides a detailed overview of the optimization hyperparameters, simulation settings, and data/task-specific information for the numerical experiments in Section~\ref{sec:experiments}.

\begin{table}[ht]
\resizebox{0.95\textwidth}{!}{
{\renewcommand{\arraystretch}{1.3}%
\begin{tabular}{lccccc}
\hline \hline
\textbf{Config. / Experiment} & \textbf{Optim. $\ell_1$} & \textbf{Optim. $\ell_{2,1}$} & \textbf{Sparse Lin. Reg.} & \textbf{LeNet-300-100 Pruning} & \textbf{Filter-sparse  CNN} \\
\hline 
Optimizer & SGD & SGD & SGD & Adam & SGD \\
Learning rate & $0.18$ & $0.1$ & $0.005$ & $0.001$ & $0.01$ \\
LR scheduler & cosine & cosine & decay ({\footnotesize $10^{-6}$}) & cosine & decay ({\footnotesize $10^{-5}$}) \\
Momentum & $0.9$ & $0.9$ & 0 & \xmark & $0.9$ \\
Epochs & $3000$ & $2000$ & 2000 & $75$ & $100$ \\
Early stopping & \xmark & \xmark & \cmark\,(200) & \cmark\,(10) & \cmark\,(6) \\
Batch size & full batch & full batch & $32$ & $128$ & $32$ \\
Loss & MSE & MSE & MSE & cross-entropy & cross-entropy \\
Init. $1$st factor & He Normal & He Normal & He Normal & He Normal (adj.) & Glorot Unif. \\
Init. rem. factors & $\mathds{1}$ (ones) & $\mathds{1}$ (ones) & $\mathds{1}$ (ones) & He Normal (adj.) & $\mathds{1}$ (ones)\\
Threshold (0) & $10^{-6}$ & $10^{-6}$ & val. optimal & \texttt{float32.eps} & \texttt{float32.eps} \\ 
Repetitions & \xmark & \xmark & $30$ & $5$ & $10$ \\
\cline{1-6}
\multicolumn{5}{l}{\textbf{Data and tasks}} \\
Task type & regression & regression & regression & classif. & classif. \\ 
Sparsity type & unstruct. ($\ell_1$) & struct. ($\ell_{2,1}$) & unstruct. ($\ell_{2/k}$) & unstruct. ($\ell_{2/k}$) & struct. ($\ell_{2,2/k}$) \\
Train samples & $1000$ & $1000$ & $500$ & $60,000$ & $60,000$ \\
Test samples & \xmark & \xmark & $500$ & $10,000$ & $10,000$ \\
Input dim. & $100$ & $100$ & $\{100,1000\}$ & $784$ & $784$ \\
Output dim. & $1$ & $1$ & $1$ & $10$ & $10$ \\
\hline \hline
\end{tabular}%
}
}
\caption{\small Hyperparameters and details for experiments in Section~\ref{sec:experiments}. The parentheses for early stopping indicate the patience in epochs. $\texttt{float32.eps}\approx 1.19 \times 10^{-7}$. %
}
\label{tab:overview-hyperparam}
\vspace{-0.2cm}
\end{table}


\subsection{Comparison of (G)HPP vs SubGD Optimization}

Figure~\ref{fig:comparison-direct-gd-norms} shows the norm-based regularization paths for the first experiment. The plot confirms that the (group) lasso objective can be effectively optimized using our smooth surrogate method, matching the optimal trajectory and inducing numerically exact zeros, while direct GD struggles to even shrink parameters near zero.
\begin{figure}[ht!]
  \centering
    \subfloat[HPP vs direct GD for lasso objective]{%
        \raisebox{0.0cm}{\includegraphics[width=0.42\textwidth]{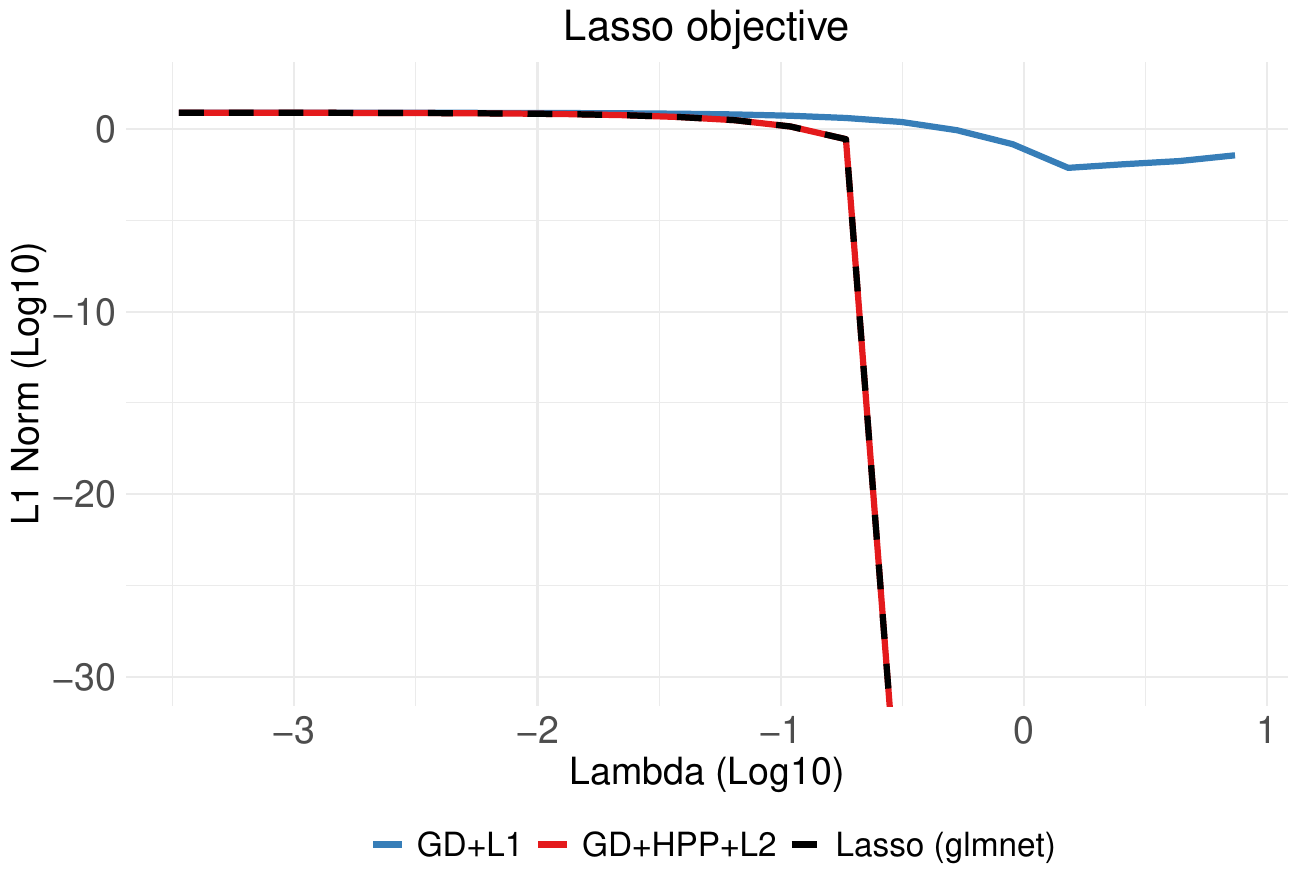}}%
        \label{fig:hpp-vs-gd-norm}%
        }%
    \hspace{0.1\textwidth}
    \subfloat[GHPP vs direct GD for group lasso]{%
        \raisebox{0.0cm}{\includegraphics[width=0.42\textwidth]{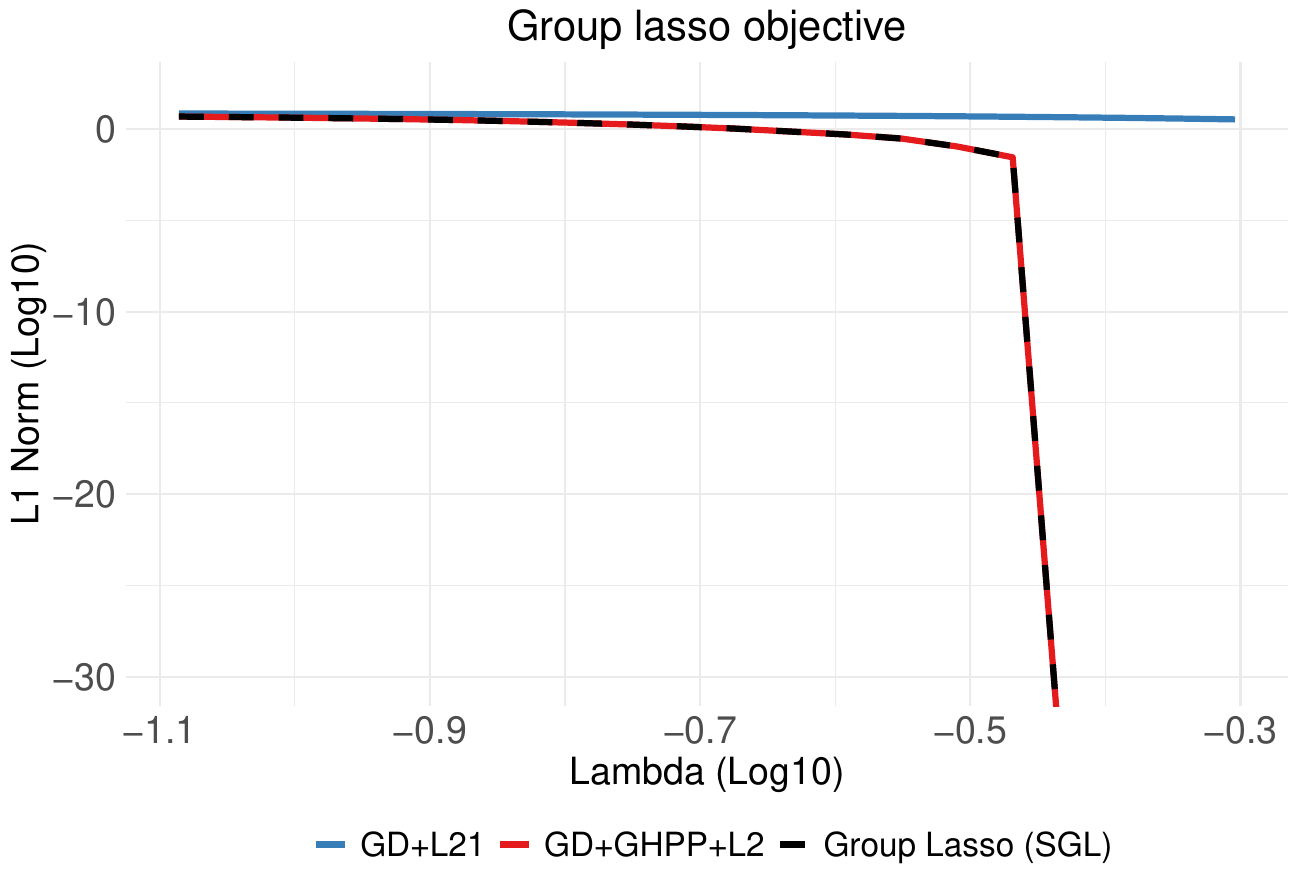}}%
        \label{fig:ghpp-vs-gd-norm}%
        }%
  \caption[Comparison of parameter norms with direct (Sub)GD optimization]{\small Comparison of parameter norms of (G)HPP-based GD and direct (Sub)GD optimization of the non-smooth $\ell_1$ regularized lasso (\textbf{a}) and $\ell_{2,1}$ regularized group lasso (\textbf{b}) objectives. Dashed lines indicate optimal solutions. %
  }
  \label{fig:comparison-direct-gd-norms}
  \vspace{-0.2cm}
\end{figure}


\subsection{Sparse Linear Regression}\label{app:subsec:sparselinear}

\noindent \textbf{Detailed Set-Up}\, In our experiments, we simulate 30 data sets $\mathcal{D}=\{(\bm{x}_i, y_i)\}_{i=1}^{n}$, where $\bm{x}_i \in \mathbb{R}^{d}$ is the feature vector and $y_i \in \mathbb{R}$ the scalar outcome. Each dataset contains $n=500$ samples each for training, validation, and testing. We focus on the 
high-dimensional $d>n$ setting, in which the number features, $d=1000$, exceeds the number of samples.
We set the number of informative features in the true parameter vector, $\bbeta^{\ast} \in \Rd$, to $s=\Vert \bbeta^{\ast} \Vert_0=10$. The magnitudes of the true signals range from the smallest signal value $|\beta^{\min}| \triangleq \frac{1}{2}\sigma\sqrt{(2/n)\log(d)}$ to a ``large'' signal $|\beta^{large}| \triangleq 2\log(d)\sigma\sqrt{(2/n)\log(d)}$, where $\sigma>0$ is the standard deviation of the additive noise in (\ref{eq:lm-dgp}). This range is based on the information-theoretic lower bound for recoverable signals \citep{wainwright2009information, zhang2010nearly}. %
Our data are simulated according to the following data-generating process:
\begin{align}\label{eq:lm-dgp}
   (\bm{X}, \bm{\varepsilon}) \sim P_{\bm{X}} \times P_{\bm{\varepsilon}},\quad  &\bm{X} \in \mathbb{R}^{n \times d},\, \bm{\varepsilon} \in \mathbb{R}^{n}\,,\quad\,\,
    \bm{Y} = \bm{X}\bbeta^{\ast} + \bm{\varepsilon},\quad \bm{Y} \in \mathbb{R}^{d} \,, 
\end{align}
\noindent where $P_{\bm{\varepsilon}}$ is a spherical Gaussian distribution $\mathcal{N}(\bm{0},\bm{I}_n)$, i.e., $\sigma=1$, and $P_{\bm{X}}$ corresponds to $\mathcal{N}(\bm{0},\bm{\Sigma})$. We consider two settings for the design matrix $\bm{X}$: in the first setting, independent features with $\bm{\Sigma} = \bm{I}_d$ are used,  whereas the second setting investigates correlated features drawn from a multivariate Gaussian with Toeplitz power covariance structure, i.e., $\bm{\Sigma}_{i,j} \triangleq \rho^{|i-j|}$ with correlation $\rho=0.5$. The $s$ informative signals' indices are randomly assigned in each simulation. To optimize the $\ell_1$, SCAD, and MCP regularized problems, we choose a specialized routine based on the Convex-Concave Procedure and the Modified Local Quadratic Approximation implemented in the \texttt{ncpen} \textsf{R} package \citep{kim2021unified}.\footnote{We also performed experiments with more widespread coordinate descent algorithms implemented in the \texttt{glmnet} \citep{friedman2010regularization} and \texttt{ncvreg} \citep{breheny2011coordinate} packages with qualitatively identical results.} 
\hphantom{0.1cm}%
\begin{figure}[ht]
  \centering
  \includegraphics[width=0.9\textwidth
  ]{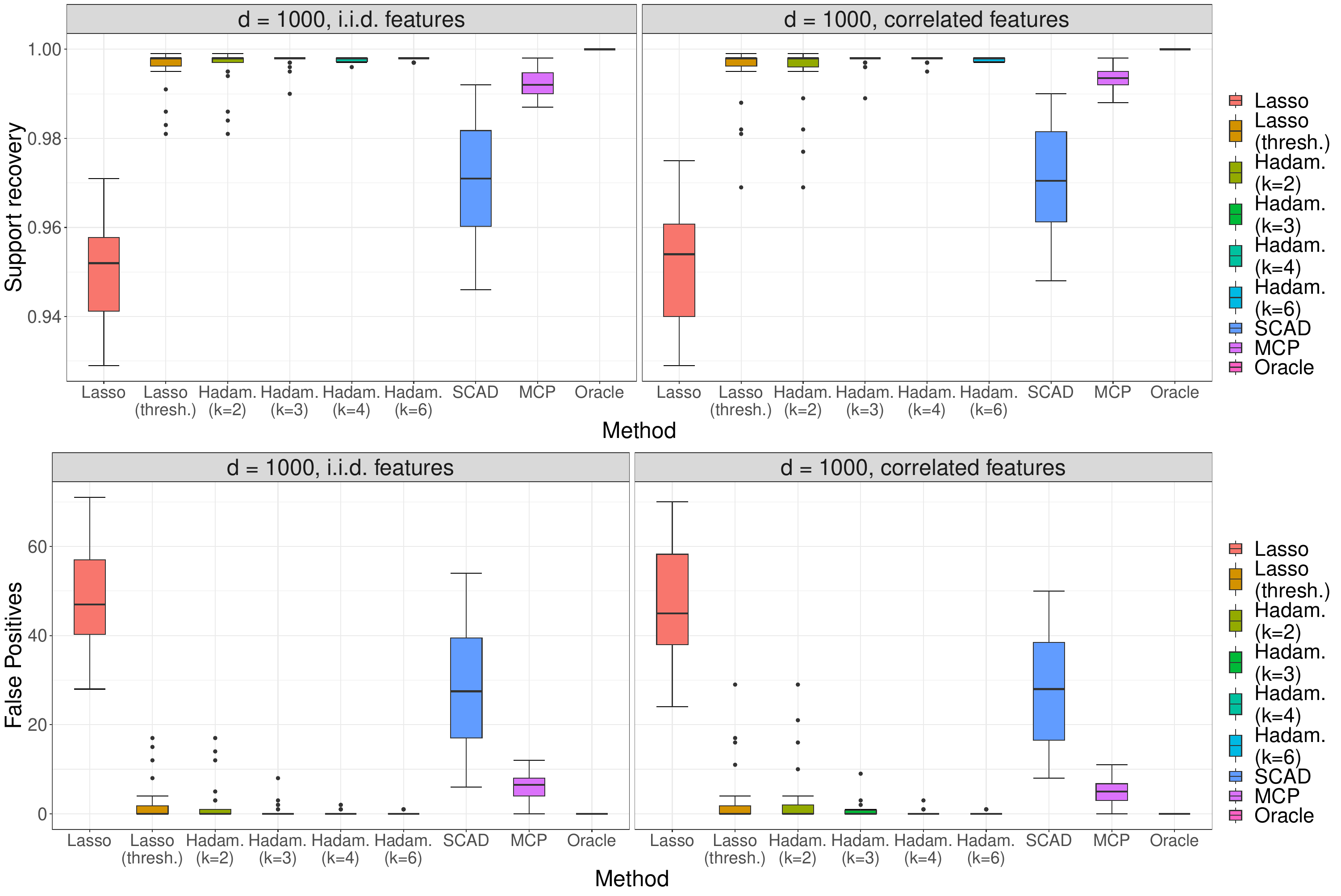}
\caption{\small Support recovery accuracy (top row) and false positives (bottom row) for different $\bm{\Sigma}$ settings (columns) and increasing factorization depths, compared against standard implementations of convex $\ell_1$ and non-convex SCAD and MCP penalties.}
\label{fig:simulation-varsel-d1000}
\vspace{-0.5cm}
\end{figure}

\noindent \textbf{Variable selection}\,  SGD does not have an in-built mechanism to produce (theoretically) exact zeros, although given a sufficiently large number of iterations, a zero floating-point representation can be obtained. To circumvent inefficient training times to obtain numerically zero parameters in our SGD-optimized models, we use early stopping combined with a post-thresholding step \citep[as suggested in][]{zhao2022high}, whose optimal cut-off for the reconstructed parameters is determined on the validation loss.
\noindent Figure~\ref{fig:simulation-varsel-d1000} shows the support recovery, defined as the classification accuracy w.r.t.~informative signals, as well as the number of false positives (FP) for the models in Section~\ref{sec:experiments}, after applying thresholding to our overparametrized models. To disentangle the effects of the optimization transfer and the thresholding step, we also apply the same operation to the conventional lasso. We can make two basic observations: first, support recovery, and FP both improve monotonically with increasing factorization depth $k$. Second, we observe that the thresholding step also significantly improves variable selection performance. 
\vspace{0.15cm}

\noindent \textbf{Low-dimensional simulation setting}\, Besides the $d>n$ setting we previously analyzed, we repeat the experiment for a low-dimensional $d<n$ setting while keeping the number of true signals constant. The magnitudes of the non-zero parameters are adjusted to the new setting using the provided definitions. Figure~\ref{fig:simulation-d100-full} shows the results for an identical simulation set-up as before, but in an alternative lower-dimensional setting with $d=100$ features, $s=10$ non-zero parameters, and $n=500$ samples in the training data. Qualitatively, the results are consistent with those for the high-dimensional setting, with minor instabilities at large factorization depths $k$, suggesting a trade-off between depth and stability.

\begin{figure}[ht]
  \centering
  \includegraphics[width=1.01\textwidth
  ]{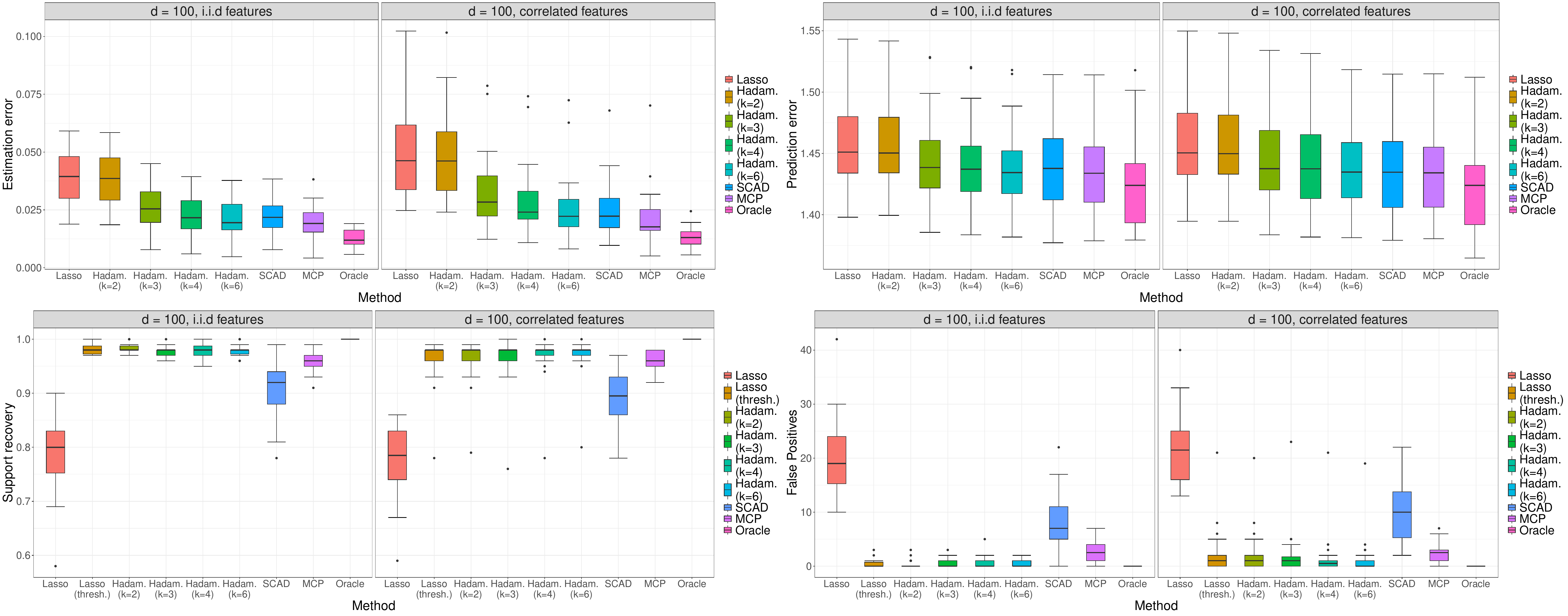}
  \caption[Numerical experiments for low-dimensional setting]{\small 
  \textbf{Left}:\, standardized estimation error (top row) and support recovery accuracy (bottom row) for different $\bm{\Sigma}$ settings (columns) of our approach for increasing factorization depths, compared against standard implementations of convex $\ell_1$ and non-convex SCAD and MCP penalties. \textbf{Right}:\, test prediction error (top row) and false positives (bottom row) for different settings of $\bm{\Sigma}$ (columns).}
  \label{fig:simulation-d100-full}
  \vspace{-0.4cm}
\end{figure}
%


\noindent \textbf{Further simulations with varying ground-truth parameter}\, Complementary to the previously analyzed low- and high-dimensional settings, we further repeat all simulations by varying the structure and sparsity of the ground-truth vector $\bbeta^\ast \in \Rd$. 

In our first additional setting, we keep the number of true signals at $s=\|\bbeta^\ast\|_0=10$, but modify their structure and scale by fixing their values to $(-0.75, -0.25, -2, -2, -2, 2, 2, 2, 0.25, 0.75)^{\top}$. Figure~\ref{fig:simulation-full-setup1} contains the full results for both low- and high-dimensional settings and independent vs correlated features, with 30 simulation repetitions for each model, feature dimension, and correlation structure. Qualitatively, the results are consistent with previous findings in Figures~\ref{fig:simulation-estpred}, \ref{fig:simulation-varsel-d1000}, \ref{fig:simulation-d100-full}.

\begin{figure}[ht]
  \centering
  \subfloat[Low-dimensional setting ($d = 100$)]{
    \includegraphics[width=1.0\textwidth]{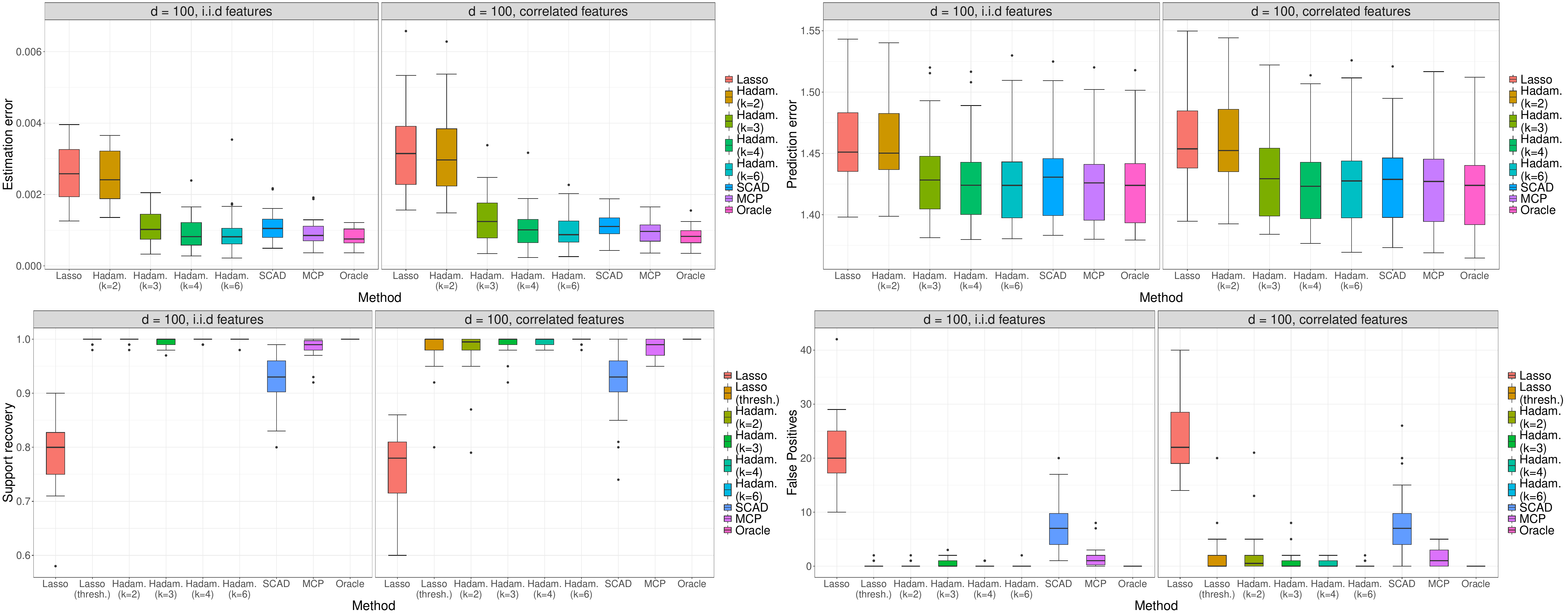}
    \label{fig:simulation-d100-full-setup1}
  }
  \vspace{0.2cm}
  \subfloat[High-dimensional setting ($d = 1000$)]{
    \includegraphics[width=1.0\textwidth]{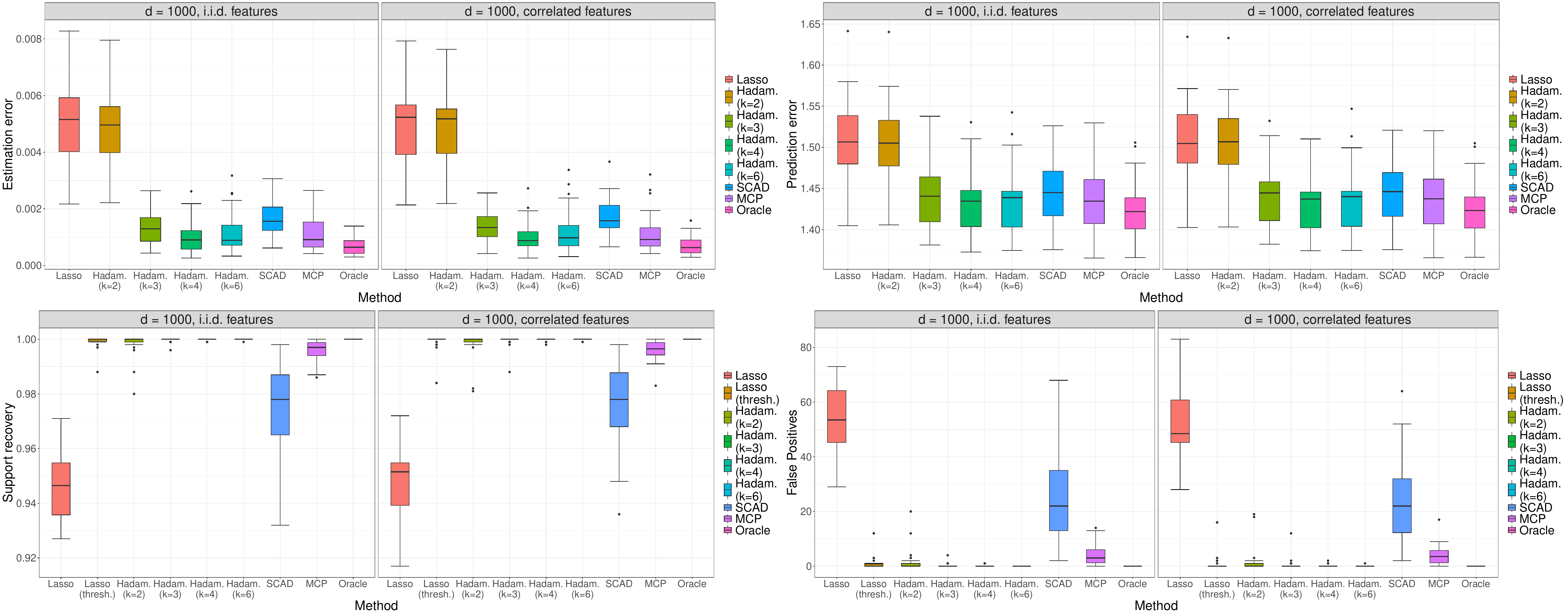}
    \label{fig:simulation-d1000-full-setup1}
  }
 \caption[Numerical experiments for signal variation setting]{\small \textbf{Simulation results for different true signal choice}. Each subfigure contains the following plots: \textbf{Left}:\, standardized estimation error (top row) and support recovery accuracy (bottom row) for different $\bm{\Sigma}$ settings (columns) of our approach for increasing factorization depths, compared against standard implementations of convex $\ell_1$ and non-convex SCAD and MCP penalties. \textbf{Right}:\, test prediction error (top row) and false positives (bottom row) for different settings of $\bm{\Sigma}$ (columns).}
  \label{fig:simulation-full-setup1}
  \vspace{-0.4cm}
\end{figure}

\vspace{0.1cm}

In our second additional setting, we vary the sparsity of the ground-truth vector and increase the number of non-zero signals to $s=\|\bbeta^\ast\|_0=40$. We specify half the coefficients to be ``small'' and logarithmically spaced between $\beta^{\min}/2$ and $\beta^{large}/2$ as defined in Appendix~\ref{app:subsec:sparselinear}. Note that this includes true signals below the recoverable threshold. The remaining 20 ``large'' signals are selected to have magnitudes between 2 and 5. Half of the signals have positive and half have negative signs. Figure~\ref{fig:simulation-full-setup3} contains the full results, again with 30 repetitions for each model, feature dimension, and correlation structure. Qualitatively, the results are largely consistent with previous findings and further highlight the trade-off between depth and stability, as the performance for deep factorizations with $k=6$ becomes brittle and degrades on average.

\begin{figure}[ht]
  \centering
  \subfloat[Low-dimensional setting ($d = 100$)]{
    \includegraphics[width=1.0\textwidth]{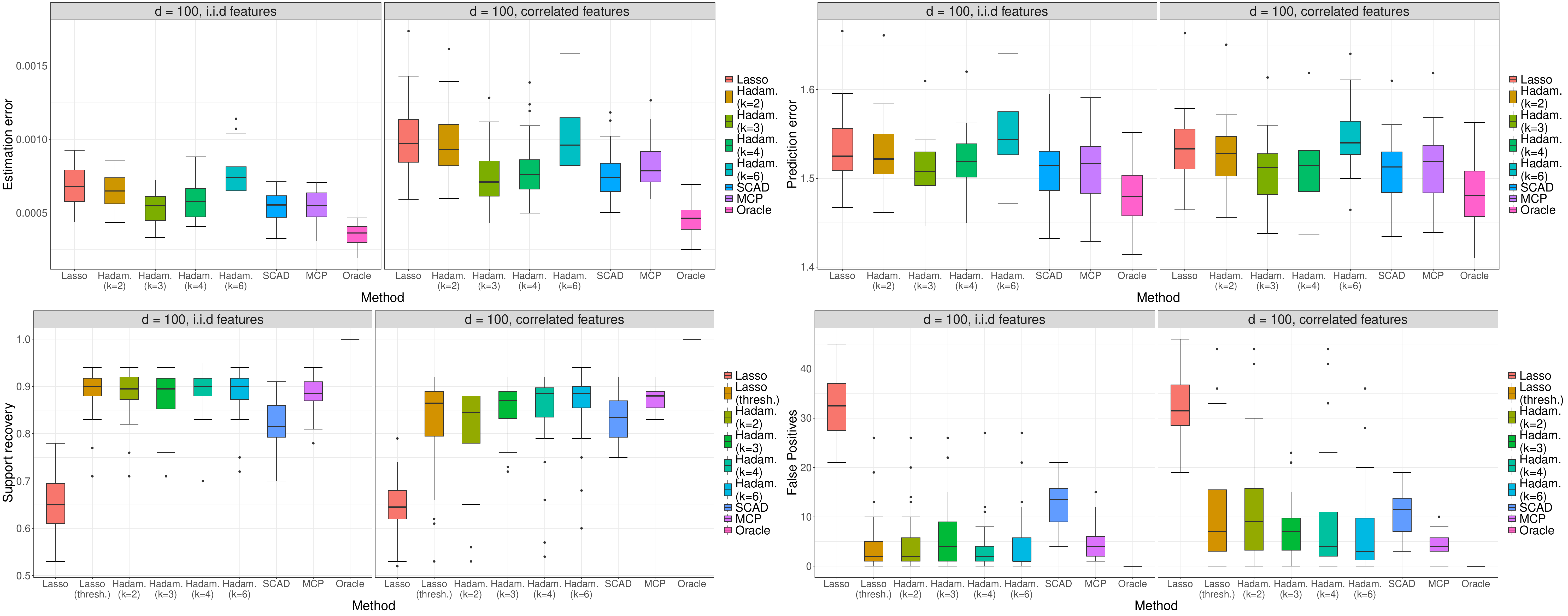}
    \label{fig:simulation-d100-full-setup3}
  }
  \vspace{0.2cm}
  \subfloat[High-dimensional setting ($d = 1000$)]{
    \includegraphics[width=1.0\textwidth]{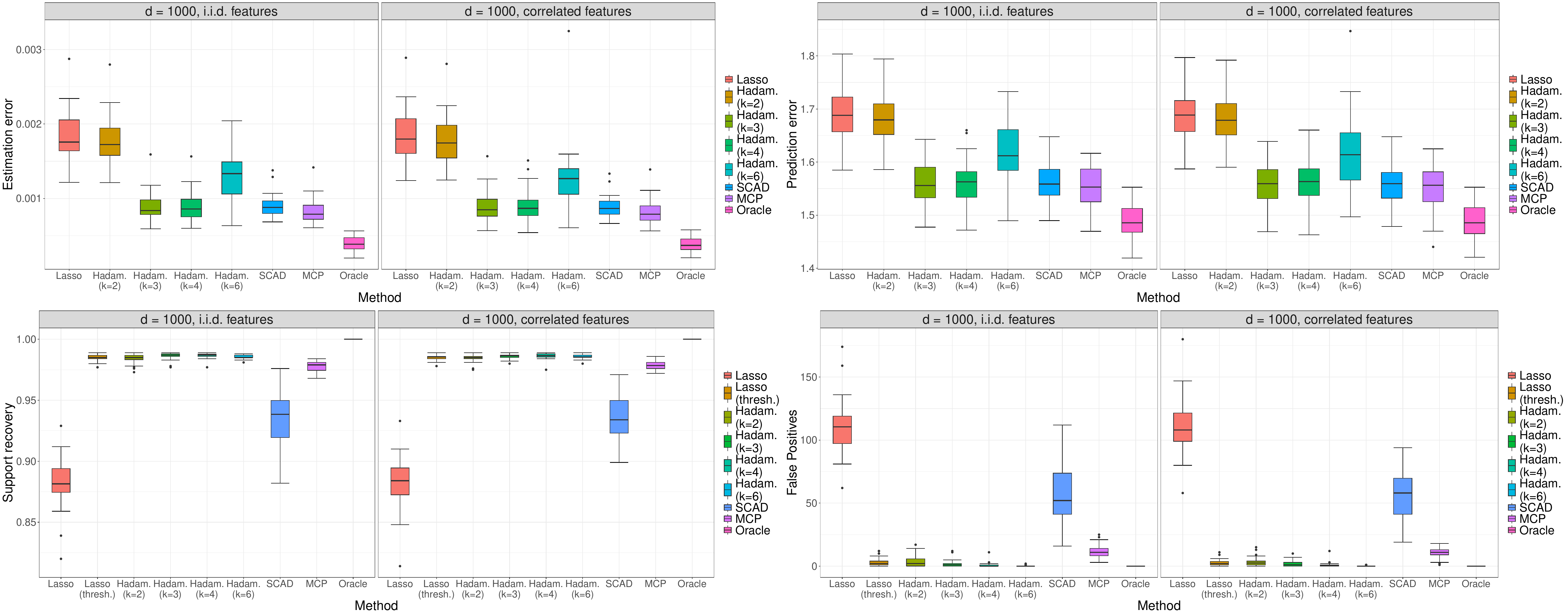}
    \label{fig:simulation-d1000-full-setup3}
  }
 \caption[Numerical experiments for different ground-truth sparsity]{\small \textbf{Simulation results for different ground-truth sparsity}. Each subfigure contains the following plots: \textbf{Left}:\, standardized estimation error (top row) and support recovery accuracy (bottom row) for different $\bm{\Sigma}$ settings (columns) of our approach for increasing factorization depths, compared against standard implementations of convex $\ell_1$ and non-convex SCAD and MCP penalties. \textbf{Right}:\, test prediction error (top row) and false positives (bottom row) for different settings of $\bm{\Sigma}$ (columns).}
  \label{fig:simulation-full-setup3}
  \vspace{-0.4cm}
\end{figure}


\subsection{Details on CNN Architecture}\label{app:details-architectures}

Our small VGG-style CNN implementation consists of two blocks of two convolutional layers after each of which max pooling is applied. The convolutional layers have a kernel size of $3$ and stride $1$. ReLU activation is used for all hidden layers. The classification head consists of two hidden layers followed by the softmax output. Dropout \citep{srivastava2014dropout} is applied after each convolutional block and dense layer. The full architecture is:
\vspace{0.1cm}
\noindent $\left[[\texttt{Input((28,28))},\texttt{Conv2D(32)},\texttt{Conv2D(32)},\texttt{MaxPool2D(2)},\texttt{Dropout(0.25)}\right]$,\\
$[\texttt{Conv2D(64)},\texttt{Conv2D(64)},\texttt{MaxPool2D(2)},\texttt{Dropout(0.25)}]$,\\
$\left[\texttt{Dense(32)},\texttt{Dropout(0.25)},\texttt{Dense(32)},\texttt{Dropout(0.25)},\texttt{Dense(10)}]\right]$
\vspace{-0.1cm}

%

\subsection{Additional Results on Computational Complexity}\label{app:subsec:computational-complexity}

The experiments on computational overhead are performed on a single 16GB RTX A4000 GPU using \texttt{TensorFlow 2.9}. The time per sample is the average wall-clock time for a single epoch normalized by sample size. The fully-connected network (MLP) has four hidden ReLU layers with 128 units each, containing $\approx 0.15$m parameters. The input is a flattened $(28,28)$ image and the softmax output has $10$ units. All trainable weights and biases are overparametrized. Figure~\ref{fig:complexity-resnet} shows additional experiments for the $\text{HPP}_k$ applied to a ResNet-20 ($\approx 0.27$m param.)~trained on CIFAR10 \citep{he2016deep}. As for the MLP, the computational overhead increases with depth and decreases with batch size. The recommended batch size of $256$ results in a $<5\%$ increase for $k=8$.
\begin{figure}[ht]
  \centering
  \includegraphics[width=0.9\textwidth]{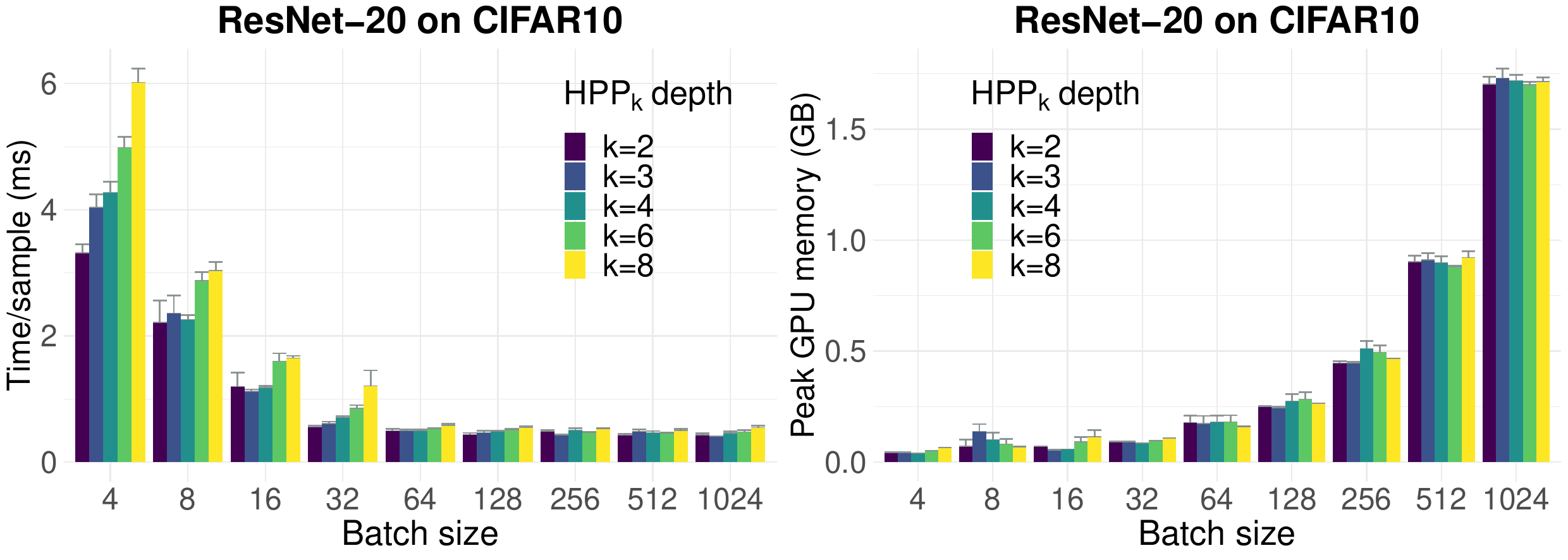}
  \caption[Time per sample (training) for ResNet20 on CIFAR10]{\small\textbf{Left}: training time per sample for different factorization depths $k$. \textbf{Right}: peak GPU memory utilization. Means and standard errors over four runs are displayed.
  }
\label{fig:complexity-resnet}
  \vspace{-0.3cm}
\end{figure}

\vspace{-0.0cm}
%

\section{Details on Geometric Intuition} \label{app:details-geometric-intuition}

%
%
%

\subsection[Difference between HPP and HDP]{Difference between HPP and HDP: $2\Vert\bm{\beta}\Vert_{1}$ vs $\Vert\bm{\beta}\Vert_{1}$ }\label{app:hpp-vs-hdp}

\begin{figure}[ht!]
  \centering
  \includegraphics[width=0.7\textwidth]{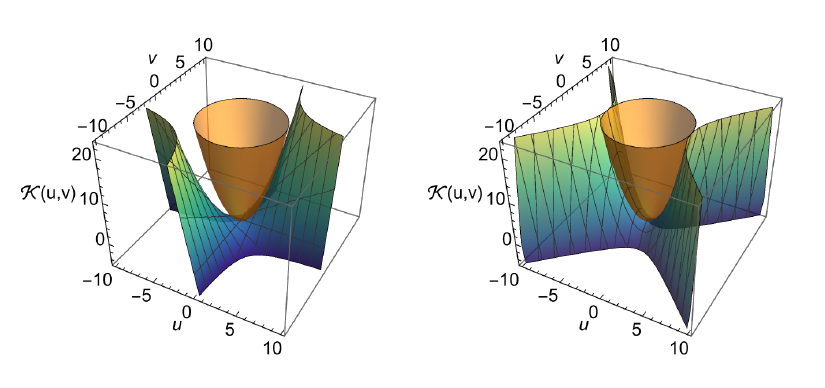}
  \caption[Hyperbola]{\small \textbf{Left}: HPP (blue/green) and the surrogate $\ell_2$ regularization $u_j^2+v_j^2$ (orange). \textbf{Right}: HDP (blue/green) and surrogate $\ell_2$ regularization.}
  \label{fig:hpp-hdp-compar}
  \vspace{-0.4cm}
\end{figure}
In the literature, both the HPP and HDP are used almost interchangeably to induce $\ell_1$ regularization in some way. However, there are subtle differences between the two, resulting in different regularization strengths. Both parametrizations define hyperbolic paraboloids for each $j=1,\ldots,d$ where the HPP can be transformed into the HDP via rotation and scaling operations. To see this, consider the scalar case $\beta\in \mathbb{R}$. Defining the $-45^{\circ}$ rotation of any point $(u,v)$ on the Cartesian plane as  $(u,v) \mapsto (\frac{u+v}{\sqrt{2}}, \frac{v-u}{\sqrt{2}})\triangleq \operatorname{rot(u,v)}$, and recalling that the coordinate change from HPP to HDP is $(\gamma,\delta)=(u+v, v-u)$, it follows that $\text{HDP}(\gamma,\delta)=\text{HPP}(\sqrt{2} \operatorname{rot}(u,v))=2 \cdot \text{HPP}(\operatorname{rot}(u,v))$. This means the HDP is the HPP but with all points rotated by $-45^{\circ}$ and its output scaled by $2$ (or equivalently, its arguments scaled by $\sqrt{2}$.) Since the surrogate $\ell_2$ regularizer is the same for both parametrizations, this scaling results in a decreased gap between $\K$ and the surrogate regularizer (cf.~Figures~\ref{fig:hpp-hdp-compar} and~\ref{fig:had-geom-intu-viewpoint}). Geometrically speaking, the fibers $\K^{-1}(\beta)$ of the HDP extend closer to the origin than for the HPP, allowing minimum-norm points with smaller norm. As derived in Sections~\ref{sec:hpp-vanilla} and~\ref{sec:hdp-subsec}, for $\beta>0$, the minimum-norm points for the HPP are $\pm (\sqrt{\beta},\sqrt{\beta})$, and $(\pm \sqrt{\beta},0)$ for the HDP. Evaluating the surrogate $\ell_2$ penalty at those points, we find an induced regularizer of $\sqrt{\beta}^2+ \sqrt{\beta}^2 = 2 |\beta|$ for the HPP, but only $0^2+\sqrt{\beta}^2=|\beta|$ for the HDP.

%

\subsection[Geometric Intuition for HPP in Three Dimensions]{Geometric Intuition for $\text{HPP}_k$ in Three Dimensions} \label{app:geom-intuition-3d}
With $\ell_2$ regularization, factorizing a scalar parameter $\beta\in\mathbb{R}$ using $\K(\bu)=\mathcal{K}(u_1, u_2, u_3)=u_1 u_2 u_3$ yields a minimal constrained $\ell_2$ penalty of $\mathcal{R}_{\bm{\beta}}(u_1, u_2, u_3) = 3 | u_1 u_2 u_3 |^{2/3}$ over $\K^{-1}(\beta)$, or $\Rbeta(\beta) = 3 |\beta|^{2/3}$ in terms of $\beta$, inducing differentiable sparse $\ell_{2/3}$ regularization (cf.~Fig.~\ref{fig:hpp3}). 
\begin{figure}[htp]
    \centering
    \subfloat[{\tiny Contours of $\K(\bu)$}]{%
        \includegraphics[width=0.18\linewidth]{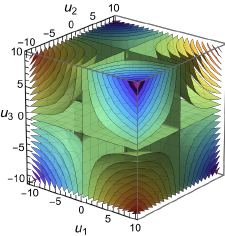}%
        \label{fig:hpc-k-contours}%
        }%
    \hspace{0.4cm}
    \subfloat[{\tiny Contours of $\Rbeta(\bu)$}]{%
        \includegraphics[width=0.18\linewidth]{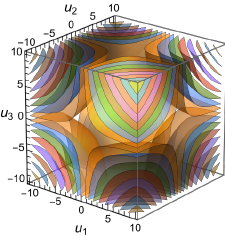}%
        \label{fig:hpc-penb-contours}%
        }%
    \hspace{0.4cm}
    \subfloat[{\tiny Fiber $\K^{-1}(5)$ and largest enclosed $\ell_2$ ball}]{%
        \includegraphics[width=0.18\linewidth]{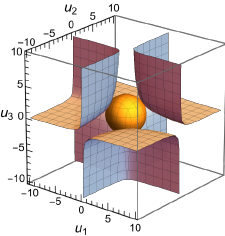}%
        \label{fig:hpc-tangent-point}%
        }%
    \hspace{0.4cm}
    \subfloat[{\tiny Majorization of $\Rbeta(\bu)$ by $\Rxi(\bu)$}]{%
        \includegraphics[width=0.18\linewidth]{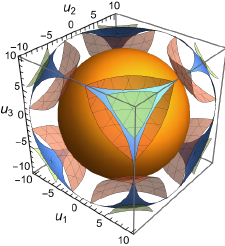}%
        \label{fig:hpc-majorization-3d}%
        }%
    \caption{\small \textbf{a)}\, Visualization of $\text{HPP}_k$, and \textbf{b)}\, minimum constrained $\ell_2$ penalty $\mathcal{R}_{\bm{\beta}}(u_1, u_2, u_3)$. \textbf{c)}\, $\text{HPP}_k$ with $\mathcal{K}(u_1, u_2, u_3)=u_1u_2u_3$ and minimum constrained $\ell_2$ regularization term $\mathcal{R}_{\bm{\beta}}(u_1, u_2, u_3)=3\cdot |u_1 u_2 u_3|^{2/3}$. \textbf{d)} fiber of $\K$ at $\beta =5$, illustrating the location of the $4$ points with minimal distance to the origin at the 
    vertices $(\hat{u}_1,\hat{u}_2,\hat{u}_3)$ of the hyperbolic smooth manifold. The vertices lie tangential to the largest enclosed $\ell_2$-ball having a radius of $\sqrt{3\cdot5^{2/3}}$. \textbf{Right}: two contours each of surrogate $\ell_2$ penalty $\Rxi$ (solid shapes) and $\mathcal{R}_{\bm{\beta}}$ at the same levels (opaque). 
    For each value $\beta=u_1 u_2 u_3\neq 0$, there are 8 points where the surrogate attains minimum $\ell_2$ distance; however, only half are solutions of the SVF by restriction to orthants that respect the sign of $\beta$ under $\K$.
    }
    \label{fig:hpp3}
\end{figure}
%
%
\subsection[Effects on Optimization Landscape]{Curvature-inducing Effects on Optimization Landscape}\label{app:effects-optim-landscape}

%
The parametrizations considered by us (cf.~Table~\ref{tab:overview}) have a significant impact on the loss landscape caused by a change in curvature induced by the multiplicative nature of the parametrization. 
Powerpropagation (\ref{eq:powerprop-param}), $\K(\bv)=\bv \odot |\bv|^{\circ(k-1)}$, as a bijective map, allows disentangling the curvature effect from overparametrization, i.e., the curvature is modified in the same base parameter space. The left panel of Figure~\ref{fig:landscape-powerprop} shows that for increasing factorization depths $k \in \{2,4,6\}$, increasingly sharp transitions at $v \in \{-1,1\}$ are induced. Given an unregularized base objective, 
the right panel of Figure~\ref{fig:landscape-powerprop} shows corresponding equivalent surrogate objectives applying Powerpropagation without surrogate regularization. The left panel of Figure~\ref{fig:landscape-powerprop-regu} displays the same base objective with non-smooth, and in parts non-convex, $\ell_q$ regularization where $q = 2/k$. The right plot depicts the corresponding equivalent surrogates obtained from our optimization transfer. 
\begin{figure}[h!]
\centering
\subfloat[ \label{fig:landscape-powerprop} \textbf{Left}: Powerpropagation $\K(v)=\beta$ (\ref{eq:powerprop-param}) for $k \in \{2,4,6\}$.\, \textbf{Right}: unregularized  objective $\P(\beta)=(1-\frac{1}{2}\beta)^2$ (black) and equivalent smooth surrogates $\Q(v)=(1-\frac{1}{2}(v|v|^{k-1}))^2$. Note the additional saddle at $v=0$.]{\includegraphics[width=0.99\textwidth]{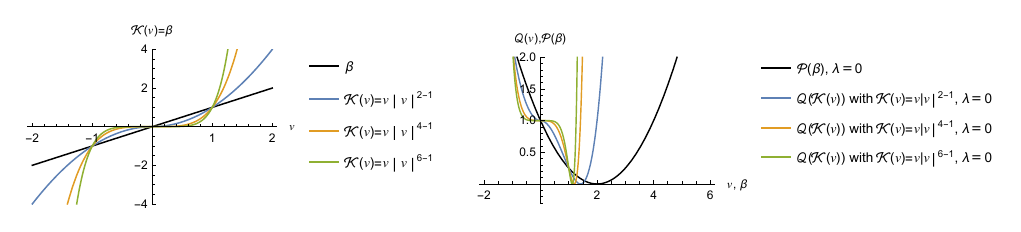}} \vspace{-0.0cm }\\
\subfloat[ \label{fig:landscape-powerprop-regu} \textbf{Left}: non-smooth $\ell_{2/k}$ regularized base objectives $\P(\beta)$ with $\lambda=\frac{1}{2}$.\, \textbf{Right}: smooth surrogates $\Q(v)=(1-\frac{1}{2}(v|v|^{k-1}))^2 + \lambda v^2$ equivalent to $\P(\beta)$ on left.]{\includegraphics[width=0.99\textwidth]{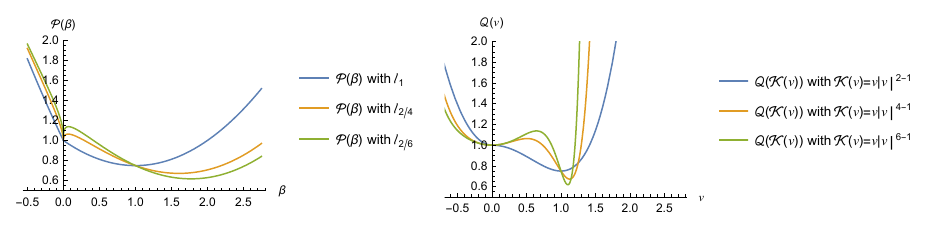}} \\ 
\caption[Loss Surfaces under Power Product Parametrizations]{\small Visualization of how smooth optimization transfer transforms the loss landscape using Powerprop.~(\ref{eq:powerprop-param}) on base objectives $\P(\beta) \triangleq (1-\frac{1}{2} \beta)^2+\lambda |\beta|^{2/k}$,  $k \in \{2,4,6\}$. 
} 
\label{fig:optim-landscapes-params}
\vspace{-0.35cm}
\end{figure}
%
%
%

%
%

%
\section[Gradient and Hessian of Surrogate Q]{Derivation of Gradient and Hessian of Smooth Surrogate $\Q$}\label{app:derivations-optim}
For simplicity, we assume no unregularized parameters $\bpsi$ so that the objective function $\Q: \Rdxi \rightarrow \mathbb{R}_0^{+}$ is defined as $\Q(\bm{\xi}) = \mathcal{L}(\mathcal{K}(\bm{\xi})) + \lambda \mathcal{R}_{\bm{\xi}}(\bxi)$,
where $\mathcal{L}: \mathbb{R}^d \rightarrow \mathbb{R}_0^{+}$ and $\mathcal{K}: \Rdxi \rightarrow \mathbb{R}^d$ are $\mathcal{C}^2$-smooth functions, $\mathcal{R}_{\bm{\xi}}: \Rdxi \rightarrow \mathbb{R}_{0}^{+}$ is a strongly convex $\ell_2$ regularization term, and $\lambda \geq 0$ a scalar. The gradient of $\Q$ with respect to $\bm{\xi}$ is given by
\begin{equation} \label{eq:gradient} \nonumber
\nabla_{\bxi} \Q(\bm{\xi}) = \mathcal{J}_{\mathcal{K}(\bm{\xi})}^{\top}(\bxi) \nabla_{\K} \mathcal{L}(\mathcal{K}(\bm{\xi})) + \lambda \nabla_{\bxi} \mathcal{R}_{\bm{\xi}}(\bxi)\,,
\end{equation}
where $\mathcal{J}_{\mathcal{K}(\bm{\xi})}(\bxi)$ is the $d \times \dxi$-dimensional Jacobian of $\mathcal{K}$ at $\bm{\xi}$, and the gradients $\nabla_{\K} \mathcal{L}(\mathcal{K}(\bm{\xi}))$ and $\nabla_{\bxi} \mathcal{R}_{\bm{\xi}}(\bxi)$ are vectors with $d$ and $\dxi$ entries. The Hessian of $\Q$ at $\bxi$ is then obtained as
%
\begin{equation} \label{eq:hessian} \nonumber
\mathcal{H}_{Q(\bm{\xi})}(\bxi) = \mathcal{H}_{\mathcal{K}(\bm{\xi})}(\bxi) \nabla_{\K} \mathcal{L}(\mathcal{K}(\bm{\xi})) + \mathcal{J}_{\mathcal{K}(\bm{\xi})}^{\top}(\bxi) \mathcal{H}_{\mathcal{L}(\mathcal{K}(\bm{\xi}))}(\bxi) \mathcal{J}_{\mathcal{K}(\bm{\xi})}(\bxi) + \lambda \mathcal{H}_{\mathcal{R}_{\bm{\xi}}}(\bxi)\,,
\end{equation}
where $\mathcal{H}_{\mathcal{K}(\bm{\xi})}(\bxi)$ is a third-order $\dxi \times \dxi \times d$-dimensional tensor, and 
$\mathcal{H}_{\mathcal{L}(\mathcal{K}(\bm{\xi}))}(\bxi)$ and $\mathcal{H}_{\mathcal{R}_{\bm{\xi}}}(\bxi)$ are Hessians of dimensions $d \times d$ and $\dxi \times \dxi$. 
From this representation, we can see that $\mathcal{J}_{\mathcal{K}(\bxi)}(\bxi)=\bm{0}$ and $\mathcal{H}_{\mathcal{K}(\bxi)}(\bxi)=\bm{0}$ imply $\mathcal{H}_{\mathcal{Q}(\bm{\xi})}(\bm{\xi})=\bm{0}$ if $\lambda=0$. For $\lambda>0$, $\mathcal{J}_{\mathcal{K}(\bxi)}(\bxi)=\bm{0}$ and $\mathcal{H}_{\mathcal{K}(\bxi)}(\bxi)=\bm{0}$ imply that $\bm{0}$ is a local minimizer of $\Q(\bxi)$ due to the positive definiteness of the $\mathcal{H}_{\mathcal{R}_{\bxi}}(\bxi)$ implied by the strong convexity of $\Rxi(\bxi)$.

%



\end{appendices}


\bibliography{sn-bibliography}

\end{document}